\documentclass{article}
\usepackage{amsmath,amsfonts,bm,mathrsfs,amsthm}
\usepackage{algorithm}
\usepackage{algorithmic}
\usepackage{amssymb}
\usepackage{wasysym}
\usepackage{wrapfig,lipsum,booktabs}

\def\st{{\textit{s.t.}}}
\def\ie{{\textit{i.e.}}}
\def\eg{{\textit{e.g.}}}

\def\Figref#1{Figure~\ref{#1}}

\def\Secref#1{Section~\ref{#1}}

\def\eqref#1{equation~\ref{#1}}
\def\Eqref#1{Equation~\ref{#1}}

\def\1{\bm{1}}

\newcommand{\approach}{\textsc{FeDGen}}
\newcommand{\mix}{\textsc{FedMix}}

\newcommand{\fedaux}{\textsc{FedAUX}}

\def\vtheta{{\bm{\theta}}}

\def\vgT{{\bm{\gT}}}

\def\cz{{\tilde{c}^*}}
\def\Dz{{\tilde{\gD}}}

\def\Mdata{{\textsc{Mnist}}}
\def\Edata{{\textsc{EMnist}}}
\def\Cdata{{\textsc{CelebA}}}
\def\Avg{{\textsc{FedAvg}}}
\def\Prox{{\textsc{FedProx}}}
\def\Ensemble{{\textsc{FedEnsemble}}}
\def\Fusion{{\textsc{FedDFusion}}}
\def\FD{{\textsc{FedDistill}}}
\def\FDFL{{\textsc{FedDistill$^+$}}}

\def\Le{{\hat{\gL}}}

\def\ye{{\hat{y}}}

\def\ze{{\hat{z}}}
\def\pe{{\hat{p}}}

\def\De{{\hat{\gD}}}
\def\dH{{d_{\gH \triangle \gH}}} 
\def\HH{{{\gH \triangle \gH}}} 

\def\rA{{\textnormal{A}}}
\def\rP{{\textnormal{P}}}

\def\rp{{\textnormal{Pr}}}

\def\ermI{{\textnormal{I}}}

\def\vtheta{{\bm{\theta}}}

\def\vw{{\bm{w}}}

\def\mLambda{{\bm{\Lambda}}}

\DeclareMathAlphabet{\mathsfit}{\encodingdefault}{\sfdefault}{m}{sl}
\SetMathAlphabet{\mathsfit}{bold}{\encodingdefault}{\sfdefault}{bx}{n}

\def\gA{{\mathcal{A}}}
\def\gB{{\mathcal{B}}}

\def\gD{{\mathcal{D}}}

\def\gH{{\mathcal{H}}}

\def\gL{{\mathcal{L}}}

\def\gN{{\mathcal{N}}}

\def\gR{{\mathcal{R}}}

\def\gT{{\mathcal{T}}}

\def\gX{{\mathcal{X}}}
\def\gY{{\mathcal{Y}}}
\def\gZ{{\mathcal{Z}}}

\def\sR{{\mathbb{R}}}

\newcommand{\E}{\mathbb{E}}

\newcommand{\R}{\mathbb{R}}

\newcommand{\KL}{D_{\mathrm{KL}}}

\DeclareMathOperator*{\argmax}{arg\,max}

\usepackage{microtype}
\usepackage{graphicx}
\usepackage{caption}
\usepackage{subcaption}
\usepackage{booktabs} 
\usepackage{multirow}

\newtheorem{theorem}{Theorem}

\newtheorem{corollary}{Corollary}
\newtheorem{remark}{Remark}
\newtheorem{lemma}{Lemma}
\newtheorem*{lemma*}{Lemma}
\newtheorem*{theorem*}{Theorem}
\newtheorem*{corollary*}{Corollary}
\newtheorem*{remark*}{Remark}

\newcommand{\judycom}[1]{{\color{blue}{#1}}}

\usepackage{hyperref}

\usepackage[accepted]{icml2021}

\icmltitlerunning{Data-Free Knowledge Distillation for Heterogeneous Federated Learning}

\begin{document}

\twocolumn[
\icmltitle{Data-Free Knowledge Distillation for Heterogeneous Federated Learning}



\icmlsetsymbol{equal}{*}

\begin{icmlauthorlist}
\icmlauthor{Zhuangdi Zhu}{to}
\icmlauthor{Junyuan Hong}{to}
\icmlauthor{Jiayu Zhou}{to} 
\end{icmlauthorlist}

\icmlaffiliation{to}{Department of Computer Science and Engineering, Michigan State University, Michigan, USA}

\icmlcorrespondingauthor{Zhuangdi Zhu}{zhuzhuan@msu.edu}
\icmlcorrespondingauthor{Jiayu Zhou}{jiayuz@msu.edu}

\icmlkeywords{Federated Learning, Machine Learning, Knowledge Distillation, ICML}

\vskip 0.3in
]



\printAffiliationsAndNotice{}  

\begin{abstract}
%
Federated Learning (FL) is a decentralized machine-learning paradigm in which a global server iteratively aggregates the model parameters of local users without accessing their data. 
User \textit{heterogeneity} has imposed significant challenges to FL, which can incur drifted global models that are slow to converge. 
\textit{Knowledge Distillation} has recently emerged to tackle this issue, by refining the server model using aggregated knowledge from heterogeneous users, 
other than directly aggregating their model parameters.
This approach, however, depends on a  proxy dataset, making it impractical unless such prerequisite is satisfied.
Moreover, the ensemble knowledge is not fully utilized to guide local model learning, which may in turn affect the quality of the aggregated model. 
In this work, we propose a \textit{data-free knowledge distillation} approach to address heterogeneous FL, where the server learns a lightweight generator to ensemble user information in a data-free manner, which is then broadcasted to users, regulating local training using the learned knowledge as an inductive bias. 
Empirical studies powered by theoretical implications show that, our approach facilitates FL with better generalization performance using fewer communication rounds, compared with the state-of-the-art\footnote{Code is available at \href{https://github.com/zhuangdizhu/FedGen}{https://github.com/zhuangdizhu/FedGen}}.
\end{abstract}


\section{Introduction}
Federated Learning (FL) is an effective machine learning approach that enables the decentralization of computing and data resources.
Classical FL, represented by \textsc{FedAvg}~\cite{mcmahan2017communication}, obtains an aggregated model by iteratively averaging the parameters of distributed local user models, therefore omits the need of accessing their data.
Serving as a communication-efficient and privacy-preserving learning scheme, 
FL has shown its potential to facilitate real-world applications, including healthcare~\cite{sheller2020federated}, biometrics~\cite{aggarwal2021fedface}, and natural language processing~\cite{hard2018federated,ammad2019federated}, to name just a few.

Along with its promising prospect, FL faces practical challenges from data \textit{\textbf{heterogeneity}}~\cite{li2020federated},
in that user data from real-world is usually \textit{non-iid} distributed, which inherently induces deflected local optimum~\cite{karimireddy2020scaffold}. 
Moreover, the permutation-invariant property of deep neural networks has further increased the heterogeneity among user models~\cite{yurochkin2019bayesian,wang2020federated}. 
Thus, performing element-wise averaging of local models, as adopted by most existing FL approaches, 
may not induce an ideal global model~\cite{li2019convergence,li2020federated}.

A variety of efforts have been made to tackle user heterogeneity, mainly from two complementary perspectives:
%
one  focuses on stabilizing local training, by regulating the deviation of local models from a global model over the \textit{parameter space}~\cite{li2020federated,dinh2020personalized,karimireddy2020scaffold}.
This approach may not fully leverage the underlying knowledge across user models, whose diversity suggests informative structural differences of their local data and thus deserves more investigation. 
Another aims to improve the efficacy of model aggregation~\cite{yurochkin2019bayesian,chenfedbe}, among which \textit{knowledge distillation} has emerged as an effective solution~\cite{lin2020ensemble,li2019fedmd}.
Provided with an unlabeled dataset as the proxy, knowledge distillation alleviates the model drift issue induced by heterogeneity, by enriching the global model with the ensemble knowledge from local models, 
which is shown to be more effective than simple parameter-averaging.
However, the prerequisite of a \emph{proxy data} can leave such an approach infeasible for many applications, where a carefully engineered dataset may not always be available on the server.
Moreover, by only refining the global model, the inherent heterogeneity among user models is not fully addressed, 
%
which may in turn affect the quality of the knowledge ensemble, especially if they are biased due to limited local data~\cite{khoussainov2005ensembles}, which is a typical case for FL.
 
Observing the challenge in the presence of user heterogeneity and the limitations of prior art, in this work, we propose a \textbf{\textit{data-free}} knowledge distillation approach for FL, dubbed as \approach~ ({\textit{Fe}derated \textit{D}istillation via \textit{Gen}erative Learning}).
Specifically, \approach\ learns a generative model derived solely from the prediction rules of user models,  which, given a target label, can yield feature representations that are consistent with the ensemble of user predictions.
%
This generator is later broadcasted to users, escorting their model training with augmented samples over the latent space, which embodies the distilled knowledge from other peer users.
Given a latent space with a dimension much smaller than the input space, the generator learned by \approach~can be lightweight, introducing minimal overhead to the current FL framework.

The proposed \approach\ enjoys multifold benefits:
i) It extracts the knowledge out of users which was otherwise mitigated after model averaging, without depending on any external data. 
ii) Contrary to certain prior work that only refines the global model, our approach directly regulates local model updating using the extracted knowledge.
We show that such knowledge imposes an inductive bias to local models, leading to better generalization performance under \textit{non-iid} data distributions.
iii) Furthermore, the proposed approach is ready to address more challenging FL scenarios, where sharing \textit{entire} model parameters is impractical due to privacy or communication constraints, since the proposed approach only requires the prediction layer of local models for knowledge extraction.

Extensive empirical studies echoed by theoretical elaborations show that, our proposed approach yields a global model with better generalization performance using fewer communication rounds, compared with the state-of-the-art.

\section{Notations and Preliminaries} \label{sec:preliminary}
Without ambiguity, in this work, we discuss a typical FL  setting for \textit{supervised learning}, i.e., the general problem of multi-class classification. 
Let $\gX \subset \R^p$ be an instance space, $\gZ \subset \sR^d$ be a {\textit{latent}} feature space with $d < p$, and $\gY \subset \sR$ be an output space.
$\gT$ denotes a {\textit{domain}} which consists of a data distribution $\gD$ over $\gX$ and a ground-truth {\textit{labeling}} function $c^*: \gX \to \gY $,  \ie~${\gT} := \langle  {\gD }, c^* \rangle $. 
Note that we will use the term \textit{domain} and \textit{task} equivalently.  
A model parameterized by $\vtheta := [ \vtheta^f; \vtheta^p ]$ consists of two components: 
a feature extractor $f:\gX \to \gZ$ parametrized by ${\vtheta^f}$,
and a predictor $h:\gZ \to \triangle^{\gY}$ parameterized by ${\vtheta^p}$, where $\triangle^{\gY}$ is the simplex over $\gY$.
Given a non-negative, convex loss function $l: \triangle^{\gY} \times \gY \to \R$, the \textbf{\textit{risk}} of a model parameterized by $\vtheta$ on domain $\gT$ is defined as $\gL_{\gT}(\vtheta):= \E_{x \sim \gD} \left[l \left( h(f(x;\vtheta^f); \vtheta^p) ,  c^*(x) \right) \right]$.

\textbf{Federated Learning} aims to learn a global model parameterized by $\vtheta$ that minimizes its risk on each of the user tasks $\gT_k$~\cite{mcmahan2017communication}:
\begin{align}\label{obj:fedavg}
    \min\nolimits_\vtheta~\E_{\gT_k \in \vgT} \left[ \gL_k(\vtheta)\right],
\end{align}
where $\vgT=\{\gT_k\}_{k=1}^K$ is the collection of user tasks.
We consider all tasks sharing the same labeling rules $c^*$ and loss function $l$, \ie, $\gT_k = \langle \gD_k, c^* \rangle$.
In practice, \Eqref{obj:fedavg} is empirically optimized by $\min_\vtheta~ \frac{1}{K} \sum_{k =1}^K  \Le_k(\vtheta)$, where
\begin{small}
$ \Le_k(\vtheta) := \frac{1}{|\De_k|} \sum_{x_i \in \De_k}\left[l(h(f(x_i;\vtheta^f);\vtheta^p),c^*(x_i)) \right]$
\end{small}
is the \textit{empirical} risk over an observable dataset $\De_k$. 
An implied assumption for FL is that the \textbf{\textit{global}} data $\De$ is distributed to each of the local domains, with ~$\De= \cup \{\De_k\}_{k=1}^K$.

\textbf{Knowledge Distillation} (KD) is also referred as a \textit{teacher-student} paradigm, with the goal of learning a lightweight student model using knowledge distilled from one or more powerful teachers~\cite{bucilua2006model,ba2014deep}.  
Typical KD leverages a \textbf{\textit{proxy}} dataset $\De_\rP$ to minimize the discrepancy between the logits outputs from the teacher model  $\vtheta_T$ and  the student model $\vtheta_S$, respectively. 
A representative choice is to use Kullback-Leibler divergence to measure such discrepancy~\cite{hinton2015distilling}:
\vspace{-0.05in}
\begin{small}
\begin{align*}
    \hspace{-0.1in}
    \min\nolimits_{\vtheta_S}~\E_{x \sim \De_\rP} \left[ \KL \left[ \sigma ( g (f(x;\vtheta^f_T); \vtheta^p_T ) \Vert \sigma ( g (f(x;\vtheta^f_S); \vtheta^p_S ) \right] \right],
\end{align*}
\end{small}
where $g(\cdot)$ is the logits output of an predictor $h$, and $\sigma(\cdot)$ is the non-linear activation applied to such logits, \ie\ $h(z; \vtheta^p) = \sigma( g(z; \vtheta^p) )$.

The idea of KD has been extended to FL to tackle user heterogeneity ~\cite{lin2020ensemble,chenfedbe}, by treating each user model $\vtheta_k$ as the \textit{teacher}, whose information is aggregated into the \textit{student} (global) model $\vtheta$ to improve its generalization performance:
\begin{small}
\begin{align*}
    \hspace{-0.15in}
     \min_{\vtheta}\underset{x \sim \De_\rP}{\E}  \left[ \KL [\sigma (\frac{1}{K} \sum_{k=1}^K g(f(x;\vtheta^f_k);\vtheta_k^p)  ) \Vert \sigma ( g(f(x;\vtheta^f);\vtheta^p) ] \right ].
\end{align*}
\end{small}
%
One primary limitation of the above approach resides in its dependence on a proxy dataset $\De_\rP$, the choice of which needs delicate consideration and plays a key role in the distillation performance~\cite{lin2020ensemble}.
Next, we show how we make KD feasible for FL in a \textit{data-free} manner.
 

\begin{figure}[b!]
    \vspace{-0.25in}
    \begin{center}
        \hspace{-0.1in}
        \begin{subfigure}[b]{0.5\textwidth}
            \centerline{\includegraphics[width=\columnwidth]{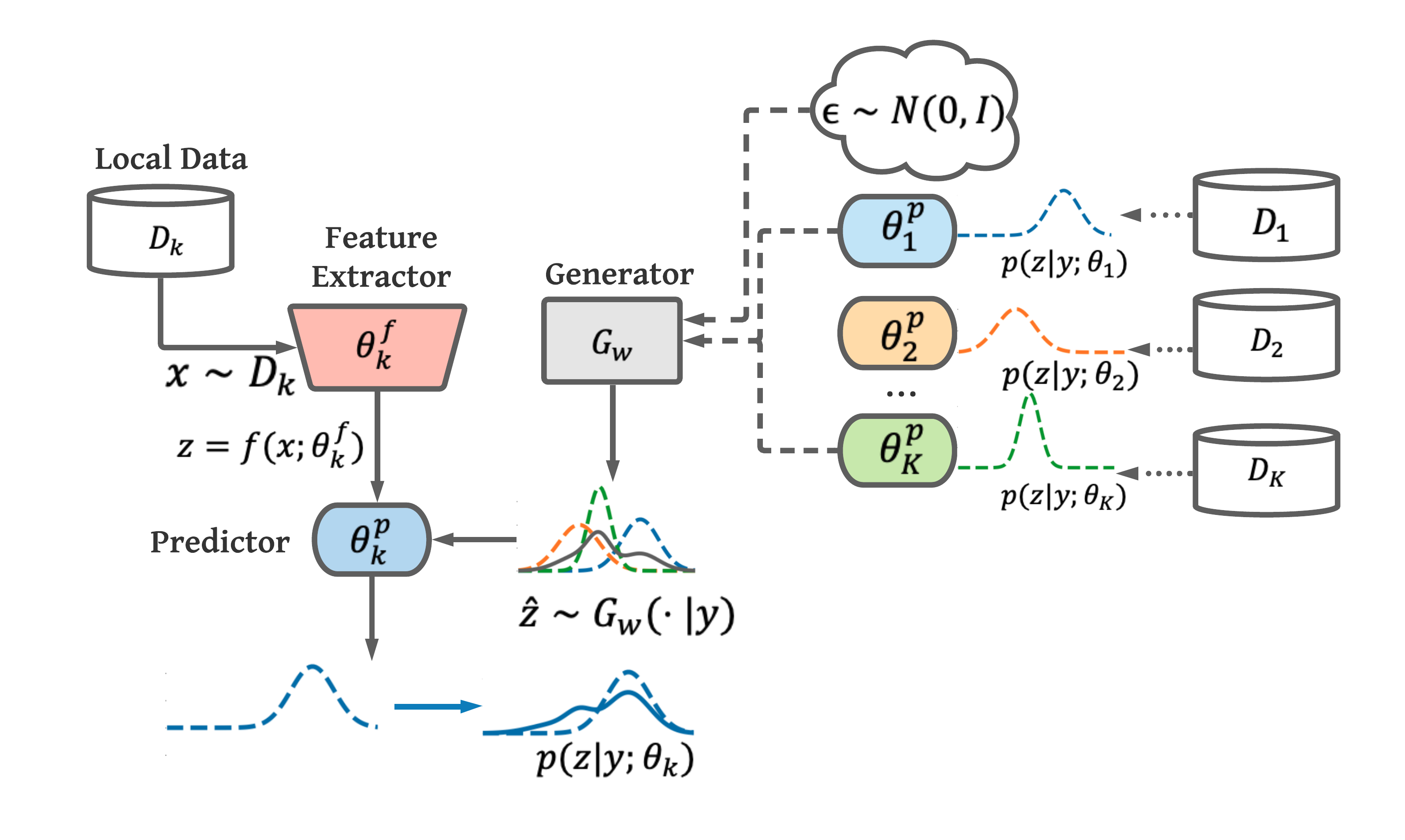}}
        \end{subfigure} 
    \vspace{-0.45in}
    \caption{Overview of \approach: a generator $G_\vw(\cdot|y)$ is learned by the server to aggregate information from different local clients without observing their data. The generator is then sent to local users, whose knowledge is distilled to user models to adjust their interpretations of a good feature distribution.}  \label{fig:fedGen-overview}
    \end{center}
\end{figure}

\begin{small}
\begin{algorithm}[tb]
   \caption{ \approach}
   \label{alg:fedGen}
\begin{algorithmic}[1]
   \STATE {\bfseries Require:} Tasks \{$\gT_k\}_{k=1}^K$;\\
            \hspace{0.1in} Global parameters $\vtheta$, local parameters $\{\vtheta_k\}_{k=1}^K$;\\
            \hspace{0.1in} Generator parameter $\vw$; $\pe(y)$ uniformly initialized;\\ 
            \hspace{0.1in} Learning rate $\alpha$, $\beta$, local steps $T$, batch size $B$, local label counter $c_k$. 
   \REPEAT 
    \STATE{Server selects active users $\gA$ uniformly at random, then broadcast $\vw, \vtheta$, $\pe(y)$ to  $\gA$. }
    
   \FOR{all user $k \in \gA$ in parallel}
       \STATE{$\vtheta_k \leftarrow \vtheta$,}
       \FOR{$t=1$, $\dots, T$}  
            \STATE{\begin{small} $\{ x_i, y_i \}_{i=1}^B \sim \gT_k$, $\{ \ze_i \sim G_\vw(\cdot|\ye_i),\ye_i \sim \pe(y) \}_{i=1}^B $.\end{small}}
           \STATE{Update label counter $c_k$.}%
           \STATE{\hspace{-0.2in} $\vtheta_k \leftarrow \vtheta_k - \beta \nabla_{\vtheta_k} J(\vtheta_k).$ \hspace{0.1in} $\triangleright$ Optimize \Eqref{obj:user} }
       \ENDFOR
        
       \STATE{User sends $\vtheta_k$, $c_k$ back to server.}
    \ENDFOR
     \STATE{Server updates $\vtheta \leftarrow \frac{1}{|\gA|}\sum_{k \in \gA} \vtheta_k$, and $\pe(y)$ based on $\{c_k\}_{k\in\gA}$.} 
     \STATE{$\vw \leftarrow \vw - \alpha \nabla_{\vw} J(\vw).$ \hspace{0.3in}$\triangleright$ Optimize \Eqref{obj:generator}}  
   \UNTIL{training~stop}
\end{algorithmic}
\end{algorithm}
\end{small}

\vspace{-0.1in}
\section{\approach: Data-Free Federated Distillation via Generative Learning}
\vspace{-0.1in}
In this section, we elaborate our proposed approach with a summary shown in Algorithm~\ref{alg:fedGen}.
An overview of its learning procedure in illustrated in Figure~\ref{fig:fedGen-overview}.

\subsection{Knowledge Extraction}
\vspace{-0.1in}
Our core idea is to extract knowledge about the global view of data distribution, which is otherwise non-observable by conventional FL, and distill such knowledge to local models to guide their learning.
We first consider learning a conditional distribution $Q^*:\gY \to \gX $ to characterize such knowledge, which is consistent with the ground-truth data distributions:  
\begin{align}\label{obj:knowledge-raw}
    {Q}^* =  \argmax_{Q_: \gY \to \gX} ~ \E_{y \sim p(y)} \E_{x \sim Q(x|y)}[\log p(y|x)],  
\end{align}
where $p(y)$ and $p(y|x)$ are the ground-truth \textit{prior} and \textit{posterior} distributions of the target labels, respectively, both of which are unknown.
To make \Eqref{obj:knowledge-raw} optimizable w.r.t $Q$, we replace $p(y)$ and $p(x|y)$ with their empirical approximations.
First, we estimate $p(y)$ as: 
\begin{align*}
    \pe(y) \propto \sum\nolimits_{k} \E_{x\sim \De_k}[\ermI(c^*(x)=y)], 
\end{align*}
where $\ermI(\cdot)$ is an indicator function, and $\De_k$ is the observable data for domain $\gT_k$. In practice, $\pe(y)$ can be obtained by requiring the training label counts from users during the model uploading phase.
Next, we approximate $p(y|x)$ using the ensemble wisdom from user models:
\begin{align*}
   \log \pe(y|x) \propto \frac{1}{K}\sum\nolimits_{k=1}^K \log p(y|x;\vtheta_k).
\end{align*}

Equipped with the above approximations, directly optimizing \Eqref{obj:knowledge-raw} over the input space $\gX$ can still be prohibitive: it brings computation overloads when $\gX$ is of high dimension, and may also leak information about the user data profile.
A more approachable idea is hence to recover an \textit{induced} distribution $G^*: \gY \to \gZ$ over a latent space, which is more compact than the raw data space and can alleviate certain privacy-related concerns:
\begin{small}
\begin{align} \label{obj:knowledge-feature}
    {G}^* = \argmax_{G: \gY \to \gZ} ~\E_{y \sim \pe(y)} \E_{z \sim G(z|y)}\left[\sum_{k=1}^K \log p(y|z;\vtheta_k^p) \right]. 
\end{align}
\end{small}

\vspace{-0.1in}
Following the above reasoning, we aim to perform knowledge extraction by learning a conditional generator $G$, parameterized by $\vw$ to optimize the following objective: 

\vspace{-0.2in}
\begin{small}
\begin{align}\label{obj:generator}
\min_{\vw} J(\vw):= \E_{y \sim \pe(y)} \E_{z \sim G_\vw(z|y)}\left[ l(\sigma (\frac{1}{K} \sum_{k=1}^{K}  g(z;\vtheta^p_k ) ), y) \right], 
\end{align}
\end{small}
\vspace{-0.2in}

where $g$ and $\sigma$ are the logit-output and the activation function as defined in Section~\ref{sec:preliminary}. 
Given an arbitrary sample $y$, optimizing \Eqref{obj:generator} only requires access to the predictor modules $\vtheta_k^p$ of user models.
Specifically, to enable diversified outputs from $G(\cdot|y)$, we introduce a noise vector $\epsilon \sim \gN(0, I)$ to the generator, which is resemblant to the re-parameterization technique proposed by prior art~\cite{kingma2013auto}, so that $z \sim G_{\vw}(\cdot|y) \equiv  G_{\vw}(y, \epsilon | \epsilon \sim \gN(0, I))$. 
We discuss more implementation details in the supplementary.

Given arbitrary target labels $y$, the proposed generator can yield feature representations $z \sim G_{\vw}(\cdot|y)$ that induce ideal predictions from the ensemble of user models.
In other words, the generator approximates an {induced} image of a \textit{consensual} distribution, which is consistent with the user data from a global view.

\subsection{Knowledge Distillation}
The learned generator $G_\vw$ is then broadcasted to local users, so that each user model can sample from $G_w$ to obtain augmented representations $z \sim G_w(\cdot|y)$ over the feature space.
As a result, the objective of a local model $\vtheta_k$ is altered to maximize the  probability that it yields ideal predictions for the augmented samples:
\vspace{-0.05in}
\begin{small}
\begin{align} \label{obj:user} 
    \hspace{-0.2in}
    \min_{\vtheta_k}~J(\vtheta_k) :=  \Le_k(\vtheta_k) + \hat{\E}_{y\sim \pe(y), z \sim G_\vw(z|y) }\left[ l(h(z;\vtheta_k^p); y)\right],
\end{align}
\end{small}
where 
\begin{small} 
    $\Le_k(\vtheta_k):=\frac{1}{|\De_k|} \sum_{x_i \in \De_k}\left[l(h(f(x_i;\vtheta_k^f);\vtheta_k^p),c^*(x_i)) \right]$
\end{small}
is the empirical risk given local data $\De_k$.
We show later that the augmented samples can introduce inductive bias to local users, reinforcing their model learning with a better generalization performance.

Up to this end, our proposed approach has realized data-free knowledge distillation, by interactively learning a lightweight generator that primarily depends on the prediction rule of local models, and leveraging the generator to convey consensual knowledge to local users.
We justify in \Secref{sec:exp-overview} that our approach can effectively handle user heterogeneity in FL, which also enjoys theoretical advantages as analyzed in Section~\ref{sec:theoretical-analysis}.

\subsection{Extensions for Flexible Parameter Sharing} 
In addition to tackling data heterogeneity, \approach~can also handle a challenging FL scenario where sharing the entire model is against communication or privacy prerequisites.
On one hand, advanced networks with deep feature extraction layers typically contain millions of parameters~\cite{he2016deep,brown2020language}, which bring significant burdens to communication.
On the other hand, it has been shown feasible to backdoor regular FL approaches~\cite{wang2020attack}.
For practical FL applications such as healthcare or finance, sharing entire model parameters may be associated with considerable privacy risks, as discussed in prior work~\cite{he2020group}.

\approach\ is ready to alleviate those problems, by sharing only the prediction layer $\vtheta_k^p$ of local models, which is the primary information needed to optimizing \Eqref{obj:generator}, while keeping the feature extractor $\vtheta_k^f$ localized.
This partial sharing paradigm is more efficient,
and at the same time less vulnerable to data leakage, as compared with a strategy that shares the entire model.
Empirical study in \Secref{sec:exp-partial-sharing} shows that, \approach~significantly benefits local users, even without sharing feature extraction modules.
We defer the algorithmic summary of this variant approach to the supplementary.

    
    \begin{figure*}[hbt!]  
        \begin{center}
            \hspace{0.1in} 
            \begin{subfigure}[b]{0.18\textwidth}
                \centerline{\includegraphics[width=\columnwidth]{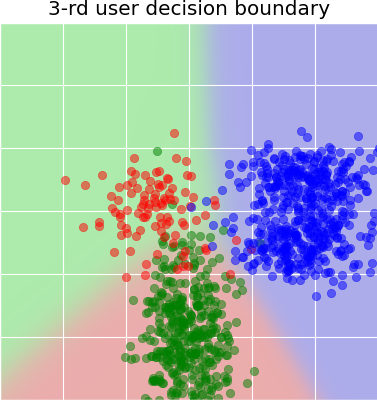}}
            \subcaption{User decision boundaries \textbf{\textit{before}} KD.} \label{fig:user-decision-before}
            \end{subfigure} 
            \hspace{0.15in}
            \begin{subfigure}[b]{0.2\textwidth}
                \centerline{\includegraphics[width=\columnwidth]{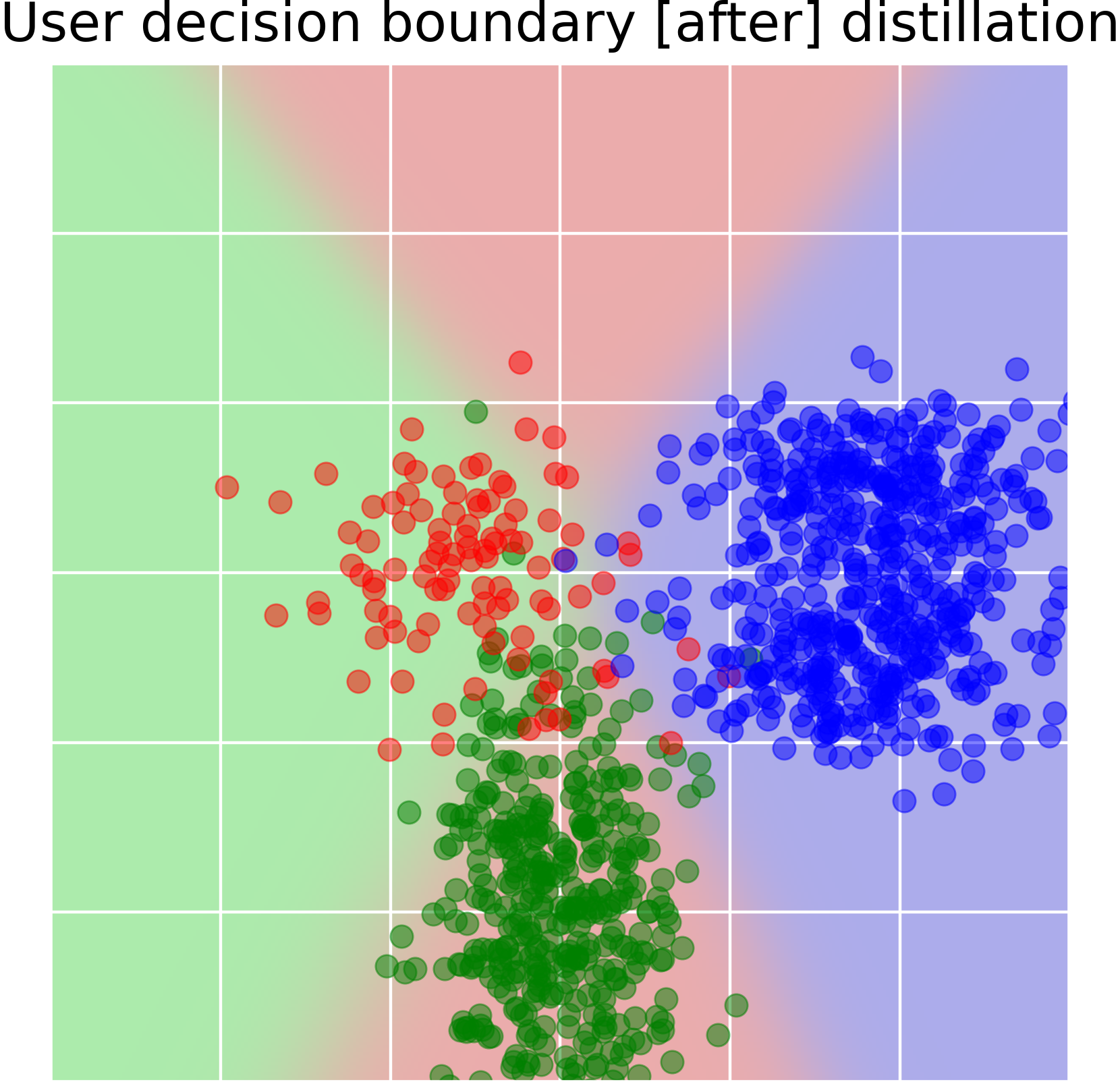}}
            \subcaption{User decision boundaries \textbf{\textit{after}} KD.} \label{fig:user-decision-after}
            \end{subfigure} 
            \hspace{0.15in}
            \begin{subfigure}[b]{0.18\textwidth}
                \centerline{\includegraphics[width=\columnwidth]{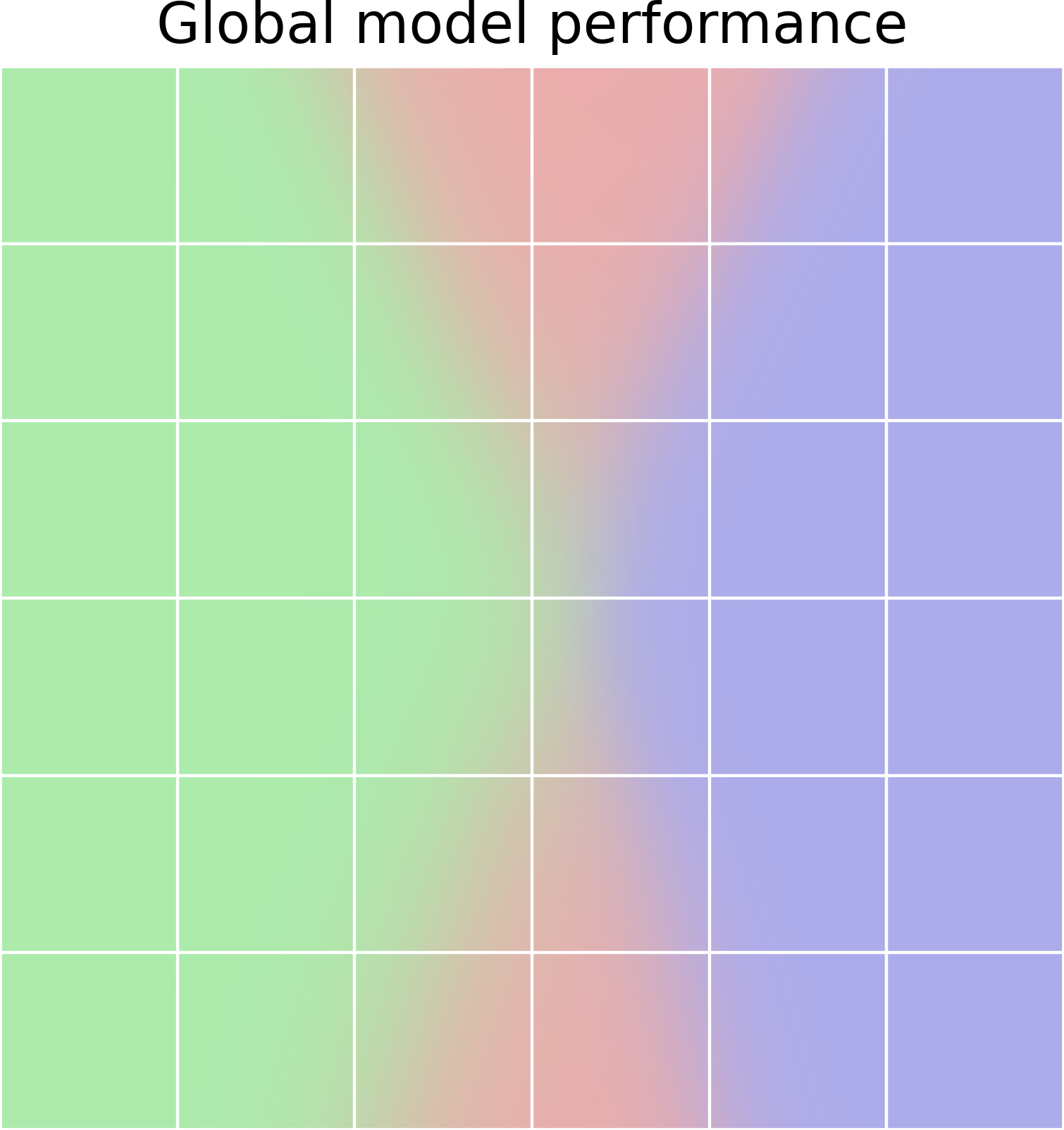}}
            \subcaption{Global model decision boundaries.} \label{fig:global-decision}
            \end{subfigure} 
            \hspace{0.15in}
            \begin{subfigure}[b]{0.19\textwidth}
                \centerline{\includegraphics[width=\columnwidth]{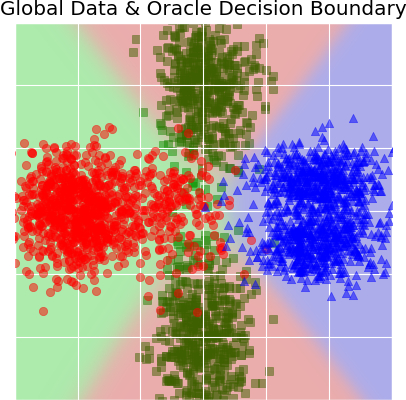}}
            \subcaption{\textit{{Oracle}} decision boundaries learned over all data.} \label{fig:oracle-data-boundary}
            \end{subfigure}  
        \end{center}
        \vspace{-0.2in}
        \caption{After KD, accuracy has improved from $81.2\%$ to $98.4\%$ for one user ( Fig \ref{fig:user-decision-before} - Fig \ref{fig:user-decision-after}), while a global model obtained by parameter-averaging (\textit{without} KD) has $93.2\%$ accuracy (Fig \ref{fig:global-decision}), compared with an oracle model with $98.6\%$ accuracy (Fig \ref{fig:oracle-data-boundary}).}
        \vspace{-0.1in}
    \end{figure*}
    
    \begin{figure*}[hbt!]  
        \begin{center}
            \hspace{0.1in}
            \begin{subfigure}[b]{0.18\textwidth}
                \centerline{\includegraphics[width=\columnwidth]{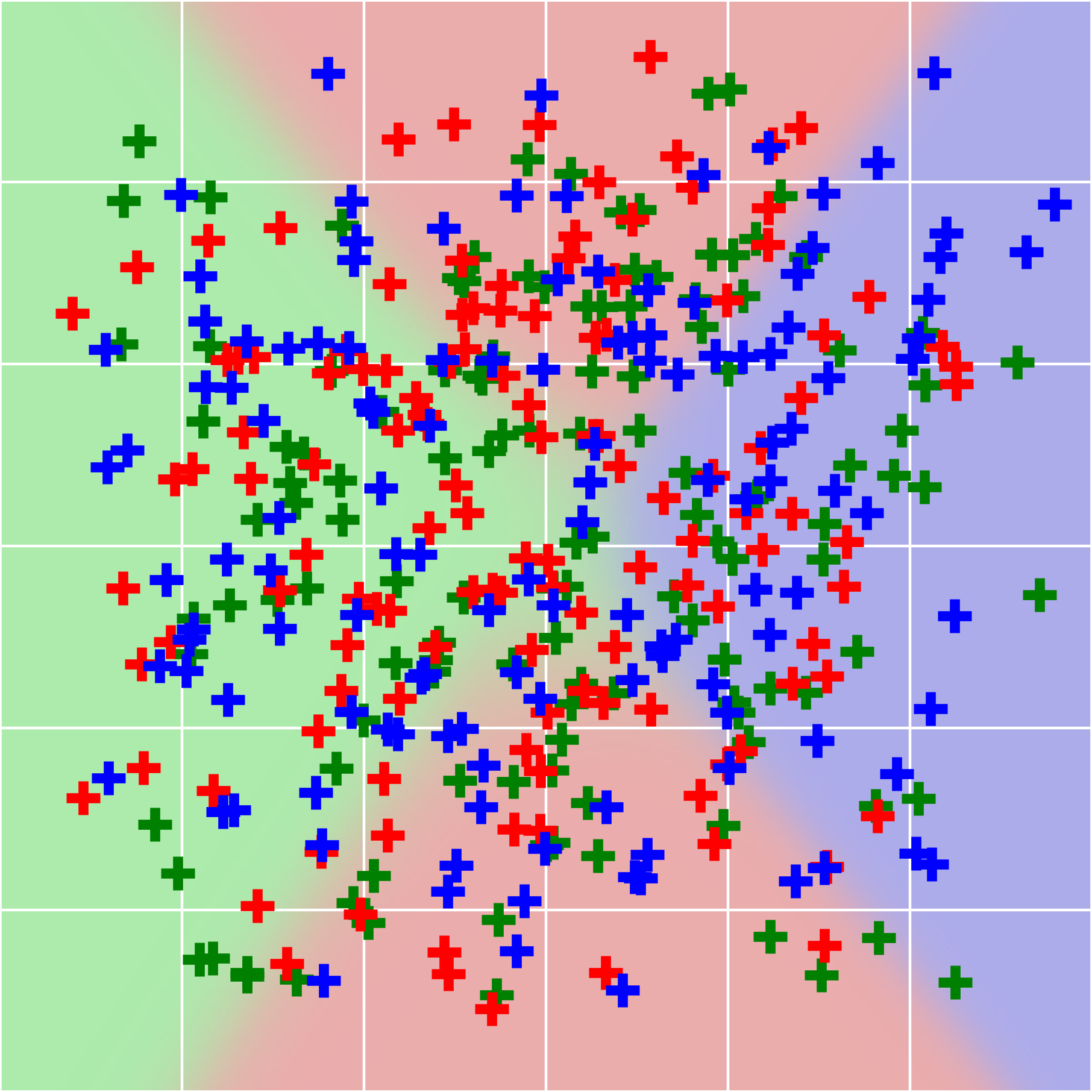}}
            \subcaption{Randomly initialized $r(x|y)$.}
            \end{subfigure} 
            \hspace{0.2in}
            \begin{subfigure}[b]{0.18\textwidth}
                \centerline{\includegraphics[width=\columnwidth]{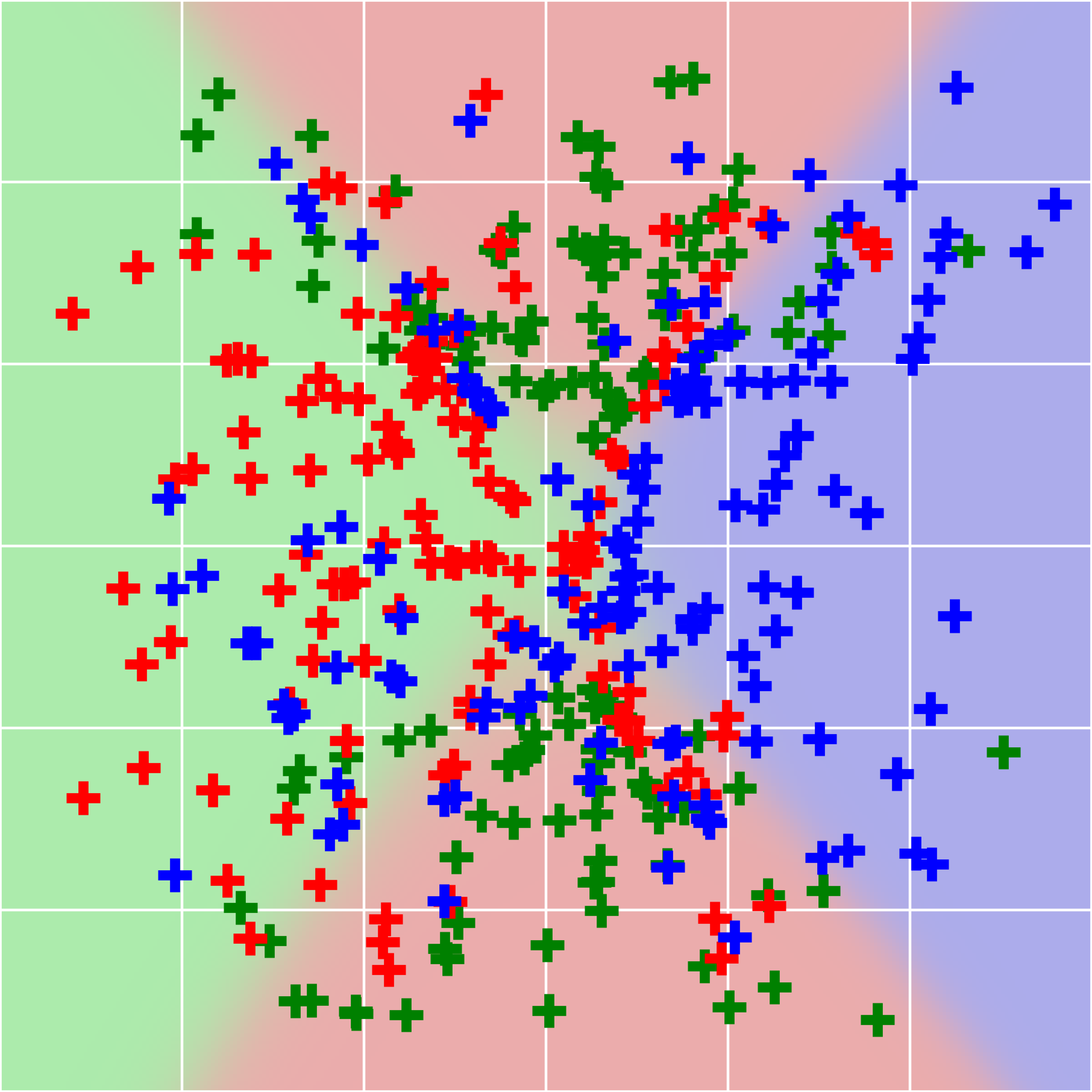}}
            \subcaption{$r(x|y)$  learned after $50$ training steps.}
            \end{subfigure} 
            \hspace{0.2in}
            \begin{subfigure}[b]{0.18\textwidth}
                \centerline{\includegraphics[width=\columnwidth]{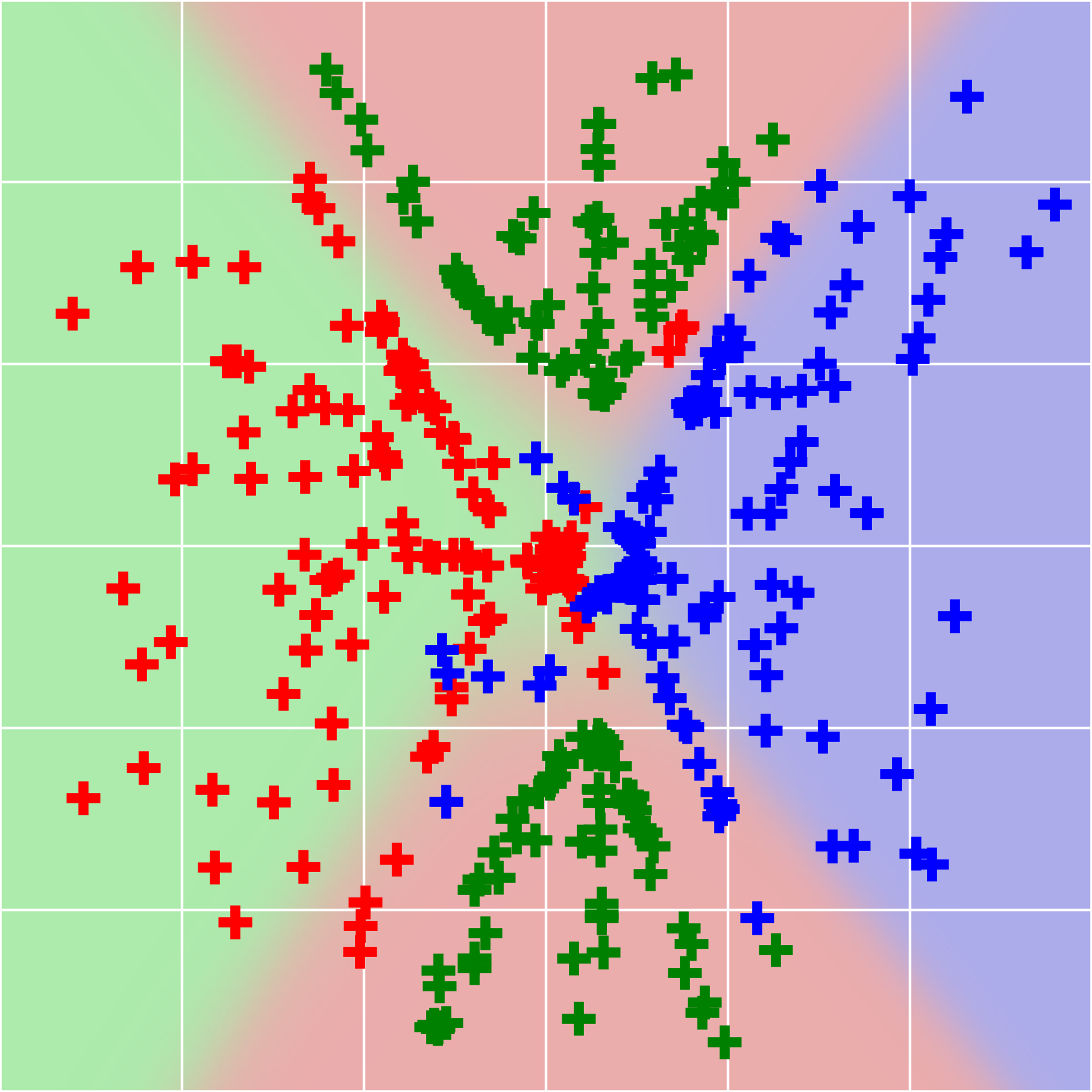}}
            \subcaption{$r(x|y)$  learned after $150$ training steps.}
            \end{subfigure} 
            \hspace{0.2in}
            \begin{subfigure}[b]{0.18\textwidth}
                \centerline{\includegraphics[width=\columnwidth]{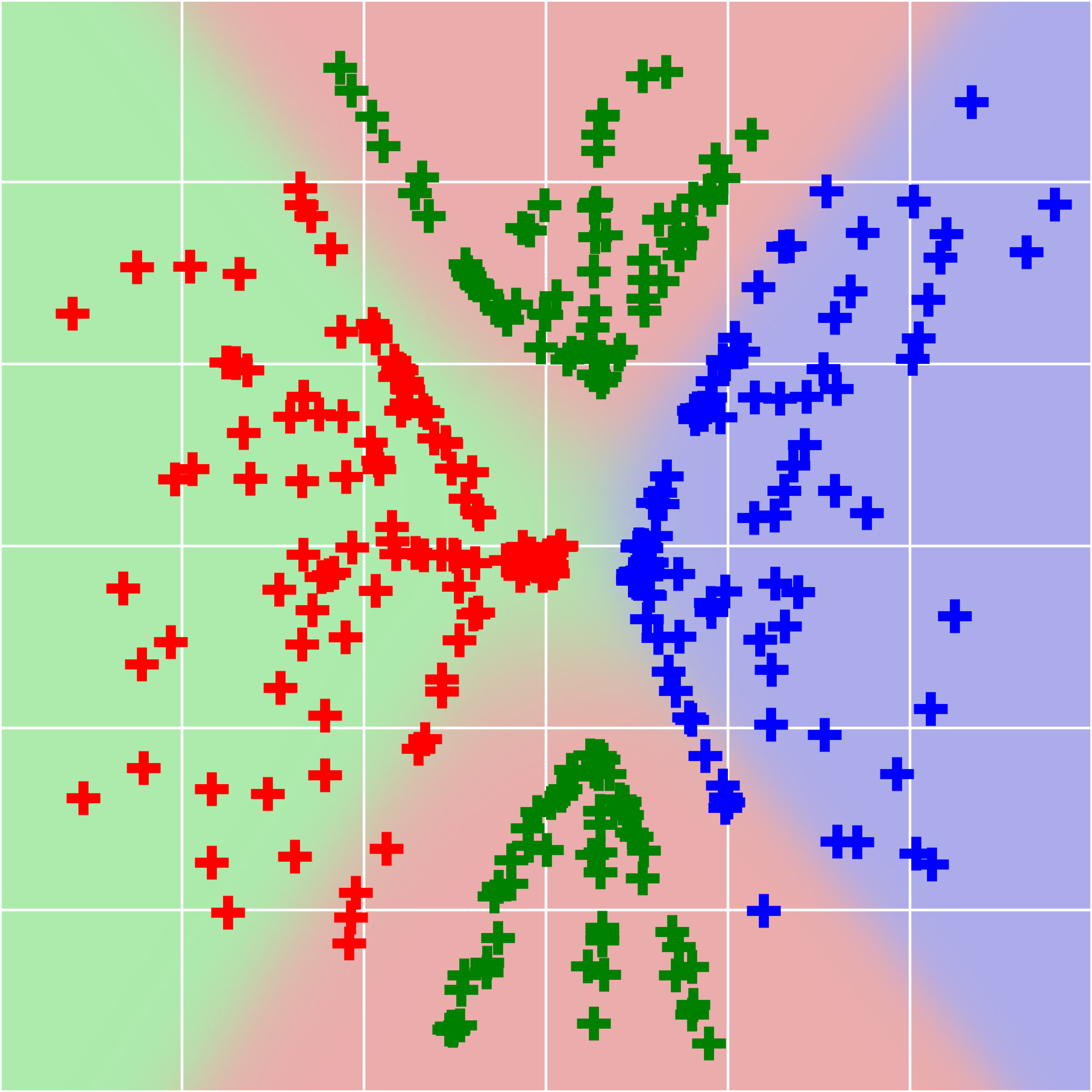}}
            \subcaption{$r(x|y)$  learned after $250$ training steps.}
            \end{subfigure} 
        \end{center}
        \vspace{-0.2in}
        \caption{Samples from the generator gradually approaches to ground-truth distribution, where each user model (teacher) sees limited, disjoint local data. Background color indicates oracle decision boundaries learned over the global data.}\label{fig:gen-data-distribution}
        \vspace{-0.2in}
    \end{figure*}

\section{\approach~Analysis}\label{sec:theoretical-analysis}
\vspace{-0.1in}
In this section, we provide multiple perspectives to understand our proposed approach.
We first visualize \textit{what} knowledge is learned and distilled by \approach, then analyze \textit{why} the distilled knowledge is favorable, from the viewpoint of \textit{distribution matching} and \textit{domain adaptation}, respectively.
We primarily focus on interpreting the rationale behind \approach\ and leave detailed discussion and derivations to the supplementary.

\subsection{Knowledge Distillation for Inductive Bias} 
\vspace{-0.1in}
We illustrate the KD process in \approach~on a FL prototype, which contains three users, each assigned with a disjoint dataset $\De_k,  k \in \{1,2,3\}$. 
%
When trained using only the local data, a user model is prone to learn biased decision boundaries (See Figure~\ref{fig:user-decision-before}). 

Next, a generator $G_w(\cdot|y)$ is learned based on the prediction rule of user models. 
For clear visualizations, we learn $G_w(\cdot|y)$ on the raw feature space $\gY \to \gX \subset \sR^2$ instead of a latent space.
As shown in Figure~\ref{fig:gen-data-distribution}, $r(x|y)$, which denotes the distribution derived from $G_w(x|y)$, gradually coincides with the ground-truth $p(x|y)$ (Figure~\ref{fig:oracle-data-boundary}), even when the individual local models are biased.
In other words, $G_w(x|y)$ can fuse the aggregated information from user models to approximate a global data distribution.

We then let users sample from $G_w(x|y)$, which serves as an inductive bias for users with limited data.
As a result, each user can observe beyond its own training data and adjust their decision boundaries to approach to the ensemble wisdom (Figure~\ref{fig:user-decision-after}).
%


\vspace{-0.1in}
\subsection{Knowledge Distillation for Distribution Matching}
\vspace{-0.1in}
A notable difference between \approach~and prior work is that the knowledge is  distilled to user models instead of the global model. 
As a result, the distilled knowledge, which conveys inductive bias to users, can directly regulate their learning by performing \textit{distribution matching} over the latent space $\gZ$:
\begin{remark} \label{remark:kl-divergence}
    Let $p(y)$ be the prior distribution of labels, and $r(z|y): \gY \to \gZ$ be the conditional distribution derived from generator $G_\vw$.
    Then regulating a user model $\vtheta_k$ using samples from  $r(z|y)$ can minimize the conditional KL-divergence between two distributions, derived from the generator and the user, respectively: 
    \begin{small}
    \begin{align} \label{obj:distribution-matching}
        & \max_{\vtheta_k}~\E_{y \sim p(y),z \sim r(z|y)} \left[ \log p(y|z;{\vtheta_k}) \right]  \nonumber \\ 
        \equiv & \min_{\vtheta_k}~{\KL}[ r(z|y) \Vert p(z|y;\vtheta_k)],
    \end{align} 
    \end{small}
\end{remark}  
\vspace{-0.1in}
where we define $p(z|y;\vtheta_k)$ as the probability that the input feature to the predictor $\vtheta_k$ is $z$ given that it yields a label $y$.
In practice, \Eqref{obj:distribution-matching} is optimized by using empirical samples from the generator: 
$ \{ (z, y ) \vert y \sim \pe(y),z \sim G_w(z|y) \},$
which is consistent with the second term of the local model objective (\Eqref{obj:user}), in that $\forall~y \in \gY$:
\begin{small}
\begin{align*}
    \max_{\vtheta_k}~\E_{z \sim r(z|y)} \left[ \log p(y|z;{\vtheta_k}) \right]  \approx \min_{\vtheta_k}~\E_{z \sim G_\vw(z|y)} \left[  l(h(z;\vtheta_k^p); y) \right].  
\end{align*}
\end{small}
Distinguished from prior work that applies weight regularization to local models~\cite{li2020federated,dinh2020personalized}, 
\approach~ can serve as an alternative and compatible solution to address user heterogeneity,
%
which inherently bridges the gap among user models w.r.t their interpretations of an ideal feature distribution.
%
%


\subsection{Knowledge Distillation for Improved Generalization}
\vspace{-0.1in}
One can also draw a theoretical connection from the knowledge learned by \approach~to an improved generalization bound.
To see this, we first present a performance bound for the aggregated model in FL, which is built upon prior arts from \textit{domain adaptation}~\cite{ben2007analysis,ben2010theory}: 
\begin{theorem} \label{theorem:fl-bounds} \textnormal{(\textbf{Generalization Bounds for FL})}
    \begin{small}
    Consider an FL system with $K$ users.
    Let $\gT_k=\langle \gD_k, c^* \rangle$ and $\gT=\langle \gD, c^* \rangle$ be the $k$-th local domain and the global domain, respectively.
    Let $\gR: \gX \to \gZ$ be a feature extraction function that is simultaneously shared among users.
    Denote $h_k$ the hypothesis learned on domain $\gT_k$, and $h=\frac{1}{K}\sum_{k=1}^K h_k$ the global ensemble of user hypotheses.
    Then with probability at least $1 - \delta$:
    \hspace{-0.2in}
    \begin{align}
    \gL_{\gT}(h) &\equiv \gL_{\gT}\left(\frac{1}{K}\sum_k h_k \right)  \nonumber \\
    & \leq \frac{1}{K} \sum_k \Le_{\gT_k}(h_k) %
    + \frac{1}{K}\sum_k  \left( \dH (\Dz_{k}, \Dz) + \lambda_k  \right) \nonumber\\ 
    & \ \ + \sqrt{\frac{4}{m}\left( d \log\frac{2 e m }{d} + \log \frac{4K}{\delta}\right) },\nonumber %
    \end{align}
    where $\Le_{\gT_k}(h_k)$ is the empirical risk on $\gT_k$, $\lambda_k := \min_h(\gL_{\gT_k}(h) + \gL_{\gT}(h))$ denotes an oracle performance. 
    $\dH (\Dz_{k}, \Dz)$ denotes the divergence measured over a symmetric-difference hypothesis space.
    $\Dz_k$ and $\Dz$ is the \textbf{induced} image of $\gD_k$ and $\gD$ over $\gR$, respectively, 
    \st~
    $\E_{z \sim \Dz_k}[\gB(z)]=\E_{x\sim \gD_k}[\gB(\gR(x))]$ given a probability event $\gB$, and so for $\Dz$.
    \end{small}

\end{theorem}
Specifically, $\gL_\gT(h)$ is usually considered as an ideal upper-bound for the global model in FL~\cite{peng2019federated,lin2020ensemble}.
Two key implications can be derived from Theorem \ref{theorem:fl-bounds}: 
i) Large user \textit{\textbf{heterogeneity}} leads to high distribution divergence ($\dH (\Dz_{k}, \Dz)$), which undermines the quality of the global model;  
ii) More empirical samples ($m$) are favorable to the generalization performance, which softens the numerical constraints.

In other words, the generalization performance can be improved by enriching local users with augmented data that aligns with the global distribution:
\begin{small}
\begin{corollary}\label{corollary-augmentation} 
    Let $\gT$, $\gT_k$, $\gR$ defined as in Theorem~\ref{theorem:fl-bounds}. 
    $\gD_\rA$ denotes an \textbf{augmented}  distribution, and $\gD_k'=\frac{1}{2}(\gD_k + \gD_\rA)$ is a \textbf{mixture} of distributions.
    Accordingly, $\Dz_\rA$, $\Dz_k'$ denotes the \textbf{induced} image of $\gD_\rA$,  $\gD_k'$  over $\gR$, respectively.
    Let $\De_k'= \De_k \cup \De_\rA$ be an empirical dataset of $\gD_k'$, with $|\De_k|=m$, $|\De_k'| = m' > m$.  
    If $\dH (\Dz_A, \Dz) $ is bounded, s.t $\exists~ \epsilon >0, \dH (\Dz_A, \Dz) \leq \epsilon $,
    then with probability $1 - \delta$:
    \vspace{-0.1in}
        \begin{align}
            \hspace{-0.1in}
            \gL_{\gT}(h) & \leq \frac{1}{K} \sum_k \gL_{\gT_k'}(h_k) %
        + \frac{1}{K}\sum_k  ( \dH (\Dz_k', \Dz) + \lambda_k'  ) \nonumber \\
        & \ \ + \sqrt{\frac{4}{m'} \left( d \log\frac{2 e m' }{d} + \log \frac{4K}{\delta} \right) }, 
        \end{align}
    where $\gT_k'=\{\gD_k', c^* \}$ is the updated local domain, $\dH (\Dz_k', \Dz) \leq \dH (\Dz_k, \Dz)$ when $\epsilon$ is small, and $\sqrt{\frac{4}{m'} ( d \log\frac{2 e m' }{d} + \log \frac{4K}{\delta} ) } < \sqrt{\frac{4}{m}( d \log\frac{2 e m }{d} + \log \frac{4K}{\delta})  }$.
\end{corollary}
\end{small}
\vspace{-0.1in}
%
%
Such an augmented distribution $\gD_\rA$ can facilitate FL from multiple aspects: not only does it relax the numerical constraints with more empirical samples ($m' > m$), but it also reduces the discrepancy between the local and global feature distributions ($\dH (\Dz_k', \Dz)$).
This finding coincides the merits of \approach: since the generator $G_w(z|y)$ is learned to recover an aggregated distribution over the feature space, one can treat samples from the generator $\{z |y\sim \pe(y), z\sim G_w(z|y)\}$ as the augmented data from ${\Dz_\rA}$, which naturally has a small deviation from the global induced distribution $\Dz$.
More rigorous analysis along this line is left to our future work.
We elaborate the role of such an augmentation distribution $\gD_A$ in the supplementary.


\section{Related Work}
\vspace{-0.1in}
\textbf{Federated Learning} (FL) is first proposed by \cite{mcmahan2017communication} as a decentralized machine learning paradigm.
Subsequent work along this line tackles different challenges faced by FL, including heterogeneity~\cite{karimireddy2020scaffold,li2020federated,mansour2020three}, privacy~\cite{duchi2014privacy,agarwal2018cpsgd}, communication efficiency~\cite{guha2019one,konevcny2016federated}, and convergence analysis~\cite{kairouz2019advances,qu2020federated,yuan2019convergence}.
Specifically, a wealth of work has been proposed to handle user \textit{\textbf{heterogeneity}}, by regularizing model weight updates~\cite{li2020federated}, allowing personalized user models~\cite{fallah2020personalized,dinh2020personalized}, or introducing new model aggregation schemes~\cite{yurochkin2019bayesian, mansour2020three}.
We refer readers to \cite{li2020federatedsurvey} for an organized discussion of recent progress on FL.

\textbf{Knowledge Distillation} (KD) is a technique to compress knowledge from one or more teacher models into an empty student~\cite{hinton2015distilling,bucilua2006model,ba2014deep,jacobs1991adaptive}.
Conventional KD hinges on a proxy dataset~\cite{hinton2015distilling}. 
More recent work enables KD with fewer data involved, such as dataset distillation~\cite{wang2018dataset}, or core-data selection~\cite{tsang2005core,sener2017active}.
Later there emerges data-free KD approaches which aim to reconstruct samples used for training the teacher~\cite{yoo2019knowledge,micaelli2019zero}.
Particularly, \cite{lopes2017data} extracts the meta-data from the teacher's activation layers.
\cite{yoo2019knowledge} learns a conditional generator which yield samples that maximizes the teacher's prediction probability of a target label.
Along the same spirit, \cite{micaelli2019zero} learns a generator by adversarial training. 
%
Different from prior work, we learn a generative model that is tailored for FL, by ensembling the knowledge of multiple user models over the latent space, which is more lightweight for learning and communication.

\textbf{Knowledge Distillation in Federated Learning} has recently emerged as an effective approach to tackle user heterogeneity.
Most existing work is data-dependent~\cite{lin2020ensemble,sun2020federated,guha2019one,chenfedbe}.
Particularly,~\cite{lin2020ensemble} proposed \textbf{\Fusion}, which performs KD to refine the global model, assuming that an unlabeled dataset is available with samples from the same or similar domains.
Complementary KD efforts have been made to confront data heterogeneity \cite{li2019fedmd,sattler2021fedaux}.
Specifically, \cite{li2019fedmd} transmits the proxy dataset instead of the model parameters.
\fedaux~\cite{sattler2021fedaux} performs \textit{data-dependent} distillation by leveraging an auxiliary dataset to initialize the server model and to {weighted-ensemble} user models, while \approach\ performs knowledge distillation in a data-free manner.
\mix~\cite{yoonfedmix} is a \textit{data-augmented} FL framework, where users share their {\textit{batch-averaged}} data among others to assist local training.
On the country, \approach\ extracts knowledge  from the existing user model parameters, which faces less privacy risks. 
%
%
%
\textbf{\FD} (Federated Distillation) is proposed by \cite{seo2020federated} which extracts from user models the statistics of the logit-vector outputs, and shares this meta-data to users for KD.
We provide detailed comparisons with work along this line in \Secref{sec:experiments}.
%
%

\vspace{-0.1in}
\section{Experiments}\label{sec:experiments}
\vspace{-0.1in}
In this section, we compare the performance of our proposed approach  with other key related work. We leave implementation details and extended experimental results to the supplementary.

\begin{figure*}[hbt!]  
    \begin{center}
        \hspace{-0.1in}
        \begin{subfigure}[b]{0.24\textwidth}
            \centerline{\includegraphics[width=\columnwidth]{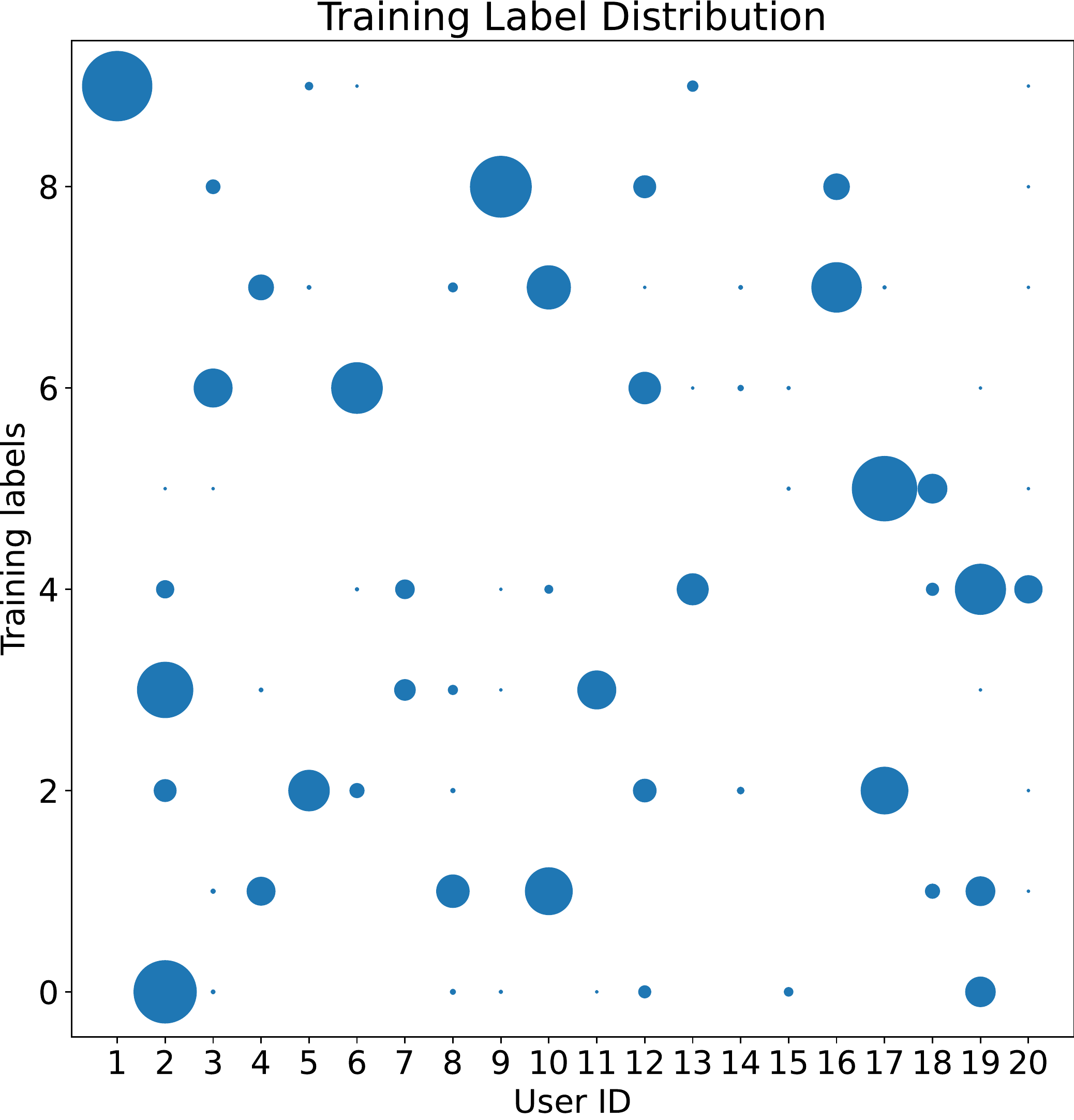}}
        \subcaption{$\alpha=0.05$}
        \end{subfigure} 
        \begin{subfigure}[b]{0.24\textwidth}
            \centerline{\includegraphics[width=\columnwidth]{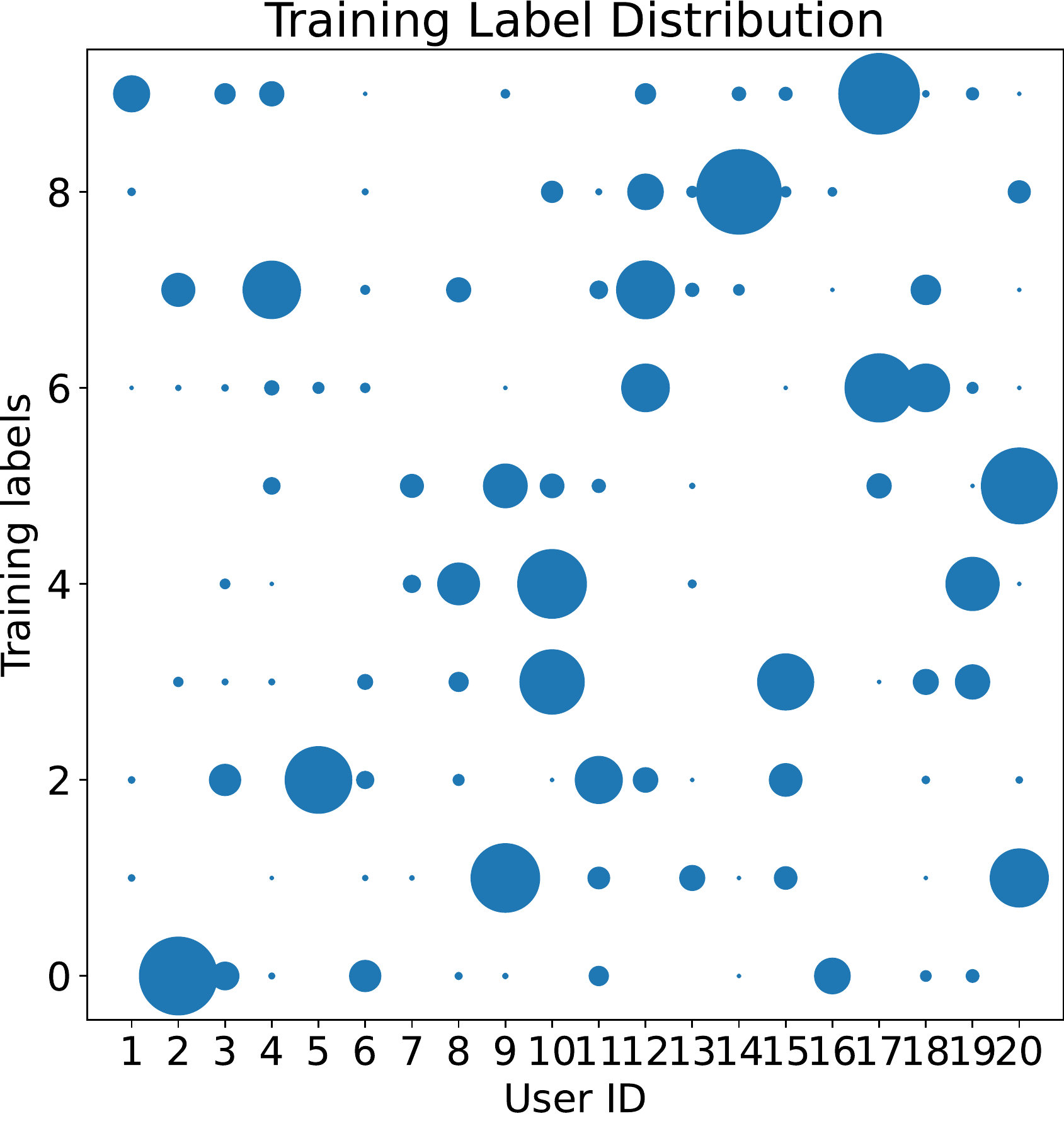}}
        \subcaption{$\alpha=0.1$}
        \end{subfigure} 
        \begin{subfigure}[b]{0.24\textwidth}
            \centerline{\includegraphics[width=\columnwidth]{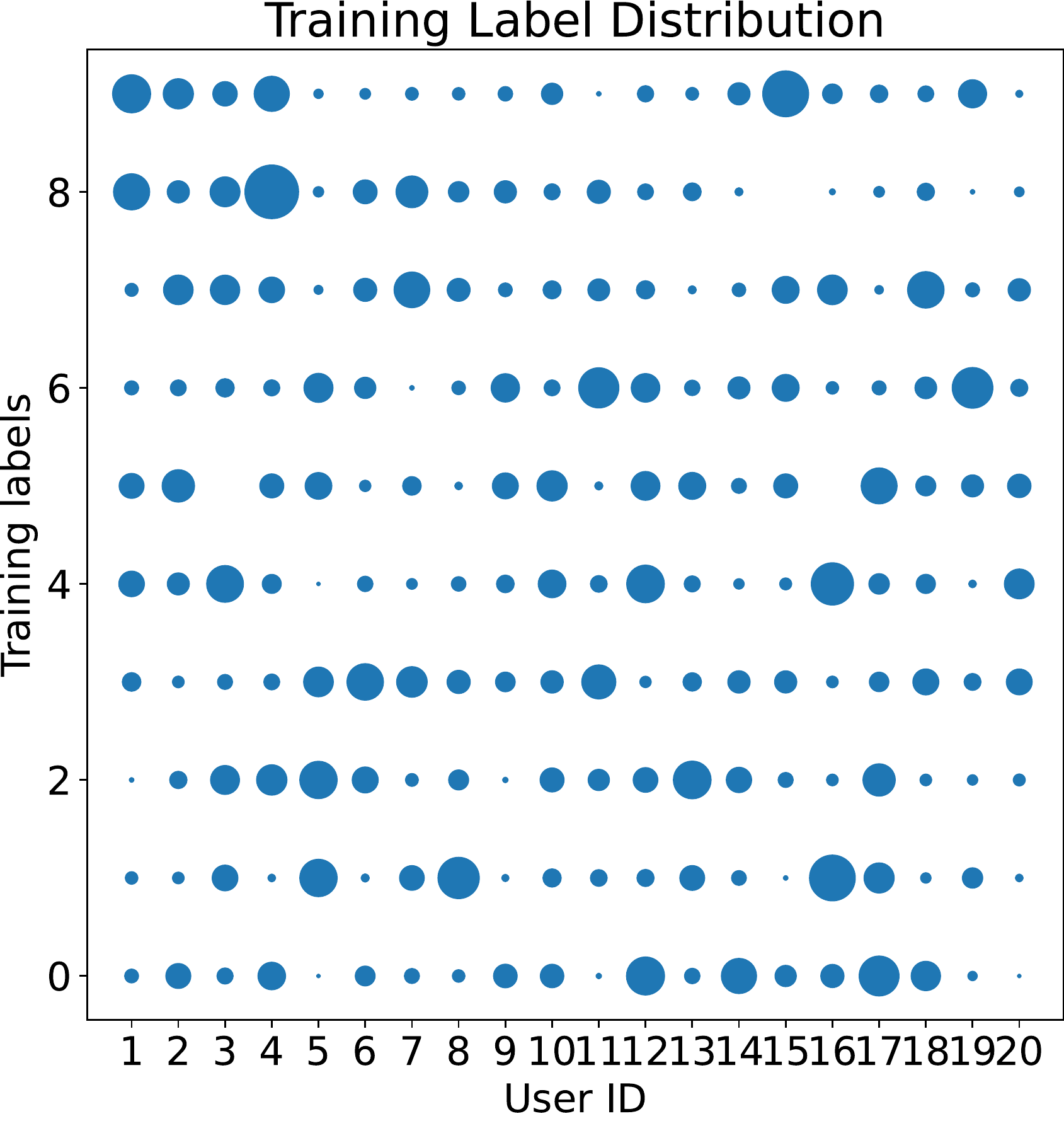}}
        \subcaption{$\alpha=1$}
        \end{subfigure} 
        \begin{subfigure}[b]{0.24\textwidth}
            \centerline{\includegraphics[width=\columnwidth]{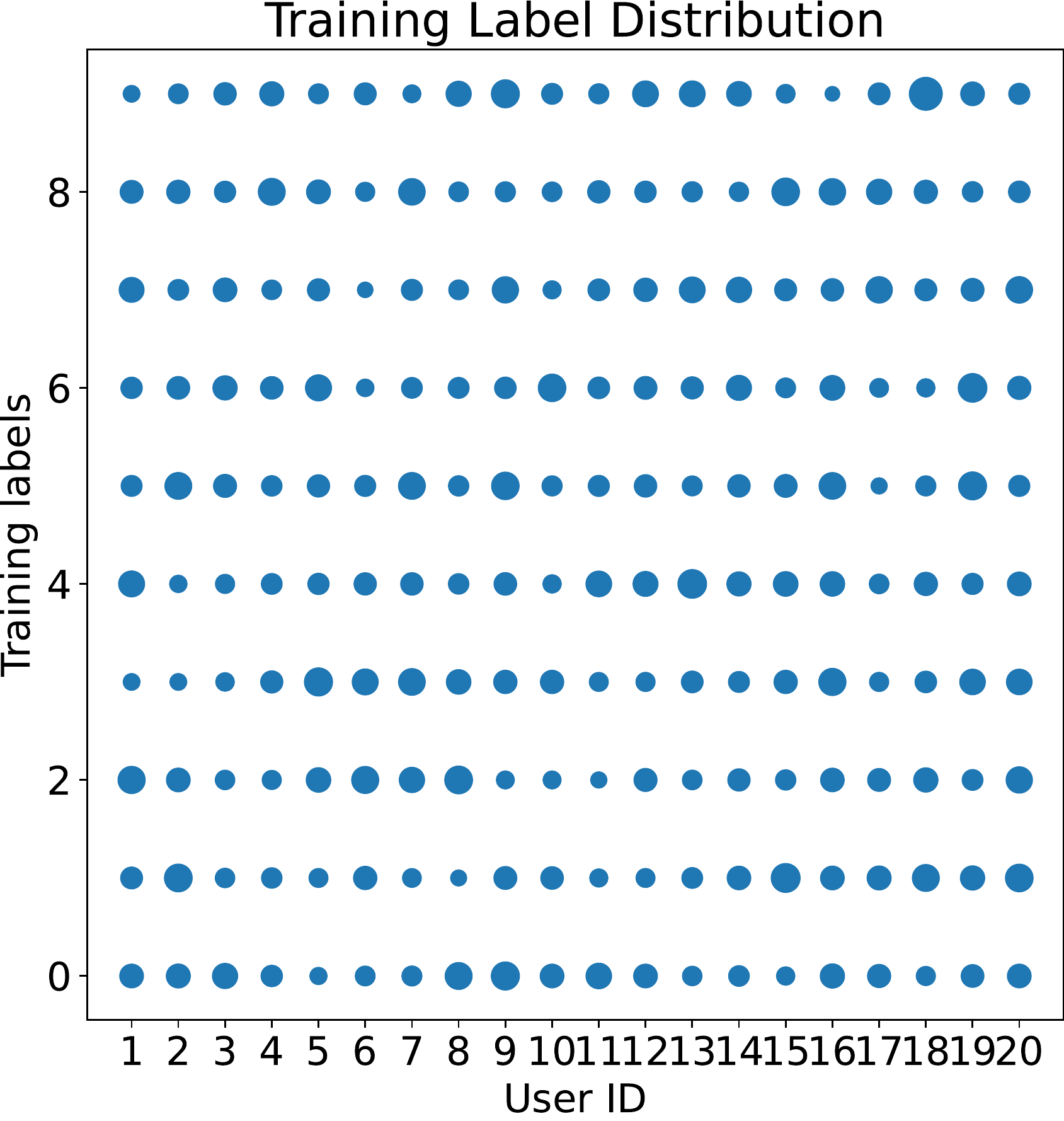}}
        \subcaption{$\alpha=10$}
        \end{subfigure} 
    \vspace{-0.2in}
    \caption{Visualization of statistical heterogeneity among users on \Mdata\ dataset, where the $x$-axis indicates user IDs, the $y$-axis indicates class labels, and the size of scattered points indicates the number of training samples for a label  available to that user.}\label{fig:label-heterogeneity}
    \end{center}
\end{figure*} 

\begin{table*}[htb!]
    \begin{center}
        \scalebox{0.8}{
            \begin{tabular}{llcccccccc}
                \toprule 
\multicolumn{1}{l}{}  &  & \multicolumn{6}{c}{\textbf{Top-1 Test Accuracy.}}  \\ \hline 
Dataset  
&  
Setting 
& \multicolumn{1}{c}{\Avg} 
& \multicolumn{1}{c}{\Prox} 
& \multicolumn{1}{c}{\Ensemble}
& \multicolumn{1}{c}{\textsc{\FD}} 
& \multicolumn{1}{c}{\textsc{\FDFL}} 
& \multicolumn{1}{c}{\Fusion} 
& \multicolumn{1}{c}{\approach} \\ \hline
\multirow{3}{*}{{\begin{tabular}[c]{@{}c@{}}\Mdata, \\ $T$=20\end{tabular}}} 
            &$\alpha$ = 0.05 & 87.70$\pm$2.07 & 87.49$\pm$2.05 & 88.85$\pm$0.68  & 70.56$\pm$1.24 & 86.70$\pm$2.27  & 90.02$\pm$0.96  & \textbf{91.30$\pm$0.74}\\  
            &$\alpha$ = 0.1 & 90.16$\pm$0.59  & 90.10$\pm$0.39 & 90.78$\pm$0.39 & 64.11$\pm$ 1.36 & 90.28$\pm$0.89  & 91.11$\pm$0.43   & \textbf{93.03$\pm$0.32}  \\  
            &$\alpha$ = 1 & 93.84$\pm$0.25  & 93.83 $\pm$ 0.29 & 93.91$\pm$0.28  & 79.88$\pm$0.66  &  94.73$\pm$0.15 & 93.37$\pm$0.40  & \textbf{95.52$\pm$0.07}  \\  
            \midrule

\multirow{3}{*}{\begin{tabular}[c]{@{}c@{}}\Cdata, \\ $T$=20\end{tabular}} 
 &$r = 5/10$ & 87.48$\pm$0.39  & 87.67$\pm$0.39 & 88.48$\pm$0.23 & 76.68$\pm$1.23  & 86.37$\pm$0.41  & 87.01$\pm$1.00  & \textbf{89.70$\pm$0.32}  \\  
 &$r = 5/25$ & 89.13$\pm$0.25  & 88.84$\pm$0.19 & \textbf{90.22$\pm$0.31} & 74.99$\pm$1.57  & 88.05$\pm$ 0.43 &   88.93$\pm$0.79  & 89.62$\pm$0.34  \\  
 &$r = 10/25$ & 89.12$\pm$0.20 & 89.01$\pm$0.33 & 90.08$\pm$0.24 & 75.88$\pm$1.17  & 88.14$\pm$0.37 & 89.25$\pm$0.56   & \textbf{90.29$\pm$0.47} \\  
 \midrule

\multirow{4}{*} {\begin{tabular}[c]{@{}c@{}}\Edata, \\ $T$=20\end{tabular}} 
            &$\alpha$ = 0.05 & 62.25$\pm$2.82  & $61.93\pm$2.31 & 64.99$\pm$0.35 & 60.49$\pm$1.27  & 61.56$\pm$2.15 & \textbf{70.40$\pm$0.79}   & {68.53$\pm$1.17}  \\  
            &$\alpha$ = 0.1 & 66.21$\pm$2.43  & 65.29$\pm$2.94  & 67.53$\pm$1.19 & 50.32$\pm$1.39 & 66.06$\pm$3.18  &  70.94$\pm$0.76 & \textbf{72.15$\pm$0.21}  \\  
            &$\alpha$ = 10 & 74.83$\pm$ 0.69 & 74.24$\pm$0.81  & 74.90$\pm$0.80 & 54.77$\pm$0.33  & 75.55 $\pm$0.94  & {74.36$\pm$0.40} & \textbf{78.43$\pm$0.74}  \\  
            \midrule
\multirow{2}{*} {\begin{tabular}[c]{@{}c@{}}\Edata, \\ $\alpha$=1\end{tabular}} 
            &$T$ = 20 & 74.83$\pm$0.99  & 74.12$\pm$0.88  & 75.12$\pm$1.07  & 46.19$\pm$0.70  & 75.41$\pm$1.05  & 75.43$\pm$0.37  & \textbf{78.48$\pm$1.04} \\  

            &$T$ = 40 & 77.02$\pm$1.09 & 75.93 $\pm$0.95 & 77.68$\pm$0.98  & 46.72$\pm$0.73 & 78.12$\pm$0.90 & 77.58$\pm$0.37  & \textbf{78.92$\pm$ 0.73}\\  
            \midrule
\bottomrule 
\end{tabular}} 
\caption{Performance overview given different  data settings. For \Mdata\ and \Edata, a \textit{smaller} $\alpha$ indicates \textit{higher} heterogeneity. For \Cdata, $r$ denotes the ratio between active users and total users. $T$ denotes the local training steps (communication delay).\label{table:performance-overview}}
\end{center}
\vspace{-0.1in}
\end{table*}

\vspace{-0.2in}
\begin{figure*}[h!bt]  
    \begin{center}
        \begin{subfigure}[b]{0.21\textwidth}
            \centerline{\includegraphics[width=\columnwidth]{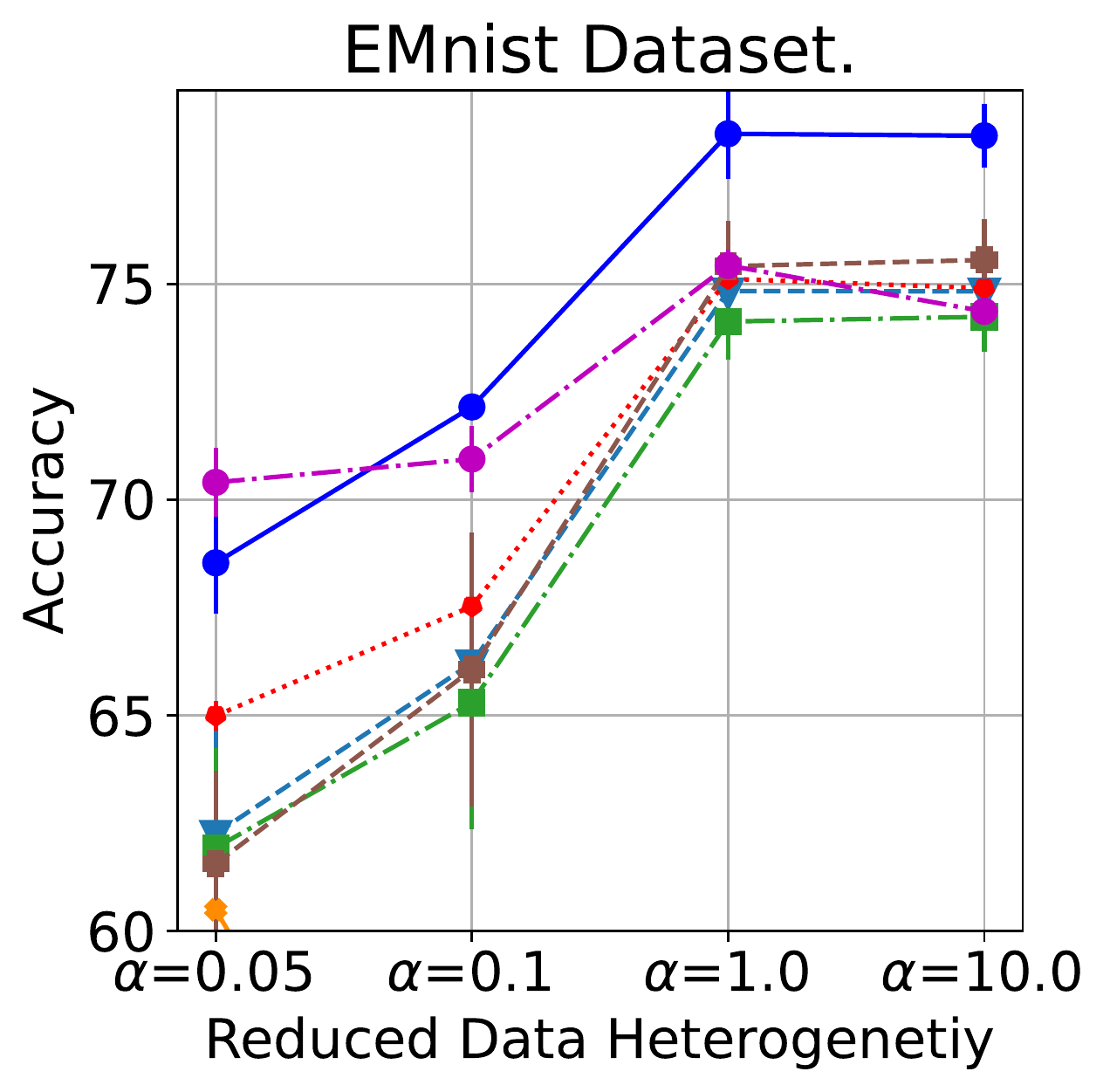}}
        \subcaption{\Edata\ Dataset.} \label{fig:emnist-batch}
        \end{subfigure} 
        %
        \begin{subfigure}[b]{0.3\textwidth}
            \centerline{\includegraphics[width=\columnwidth]{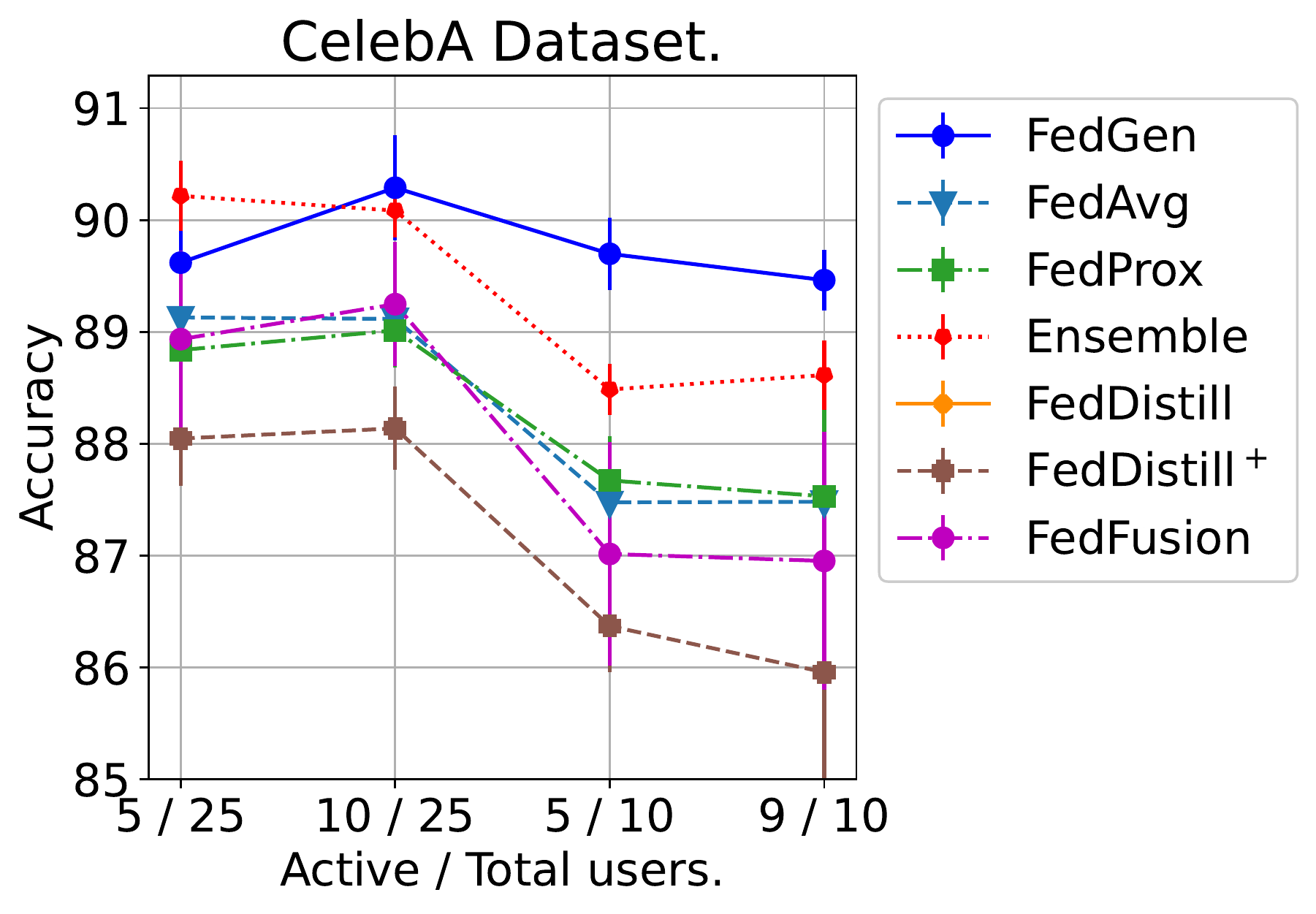}}
            \caption{\Cdata\ Dataset.}\label{fig:celeb-batch}
            \end{subfigure} 
            %
        \hspace{-0.1in}
        \begin{subfigure}[b]{0.21\textwidth}
            \centerline{\includegraphics[width=\columnwidth]{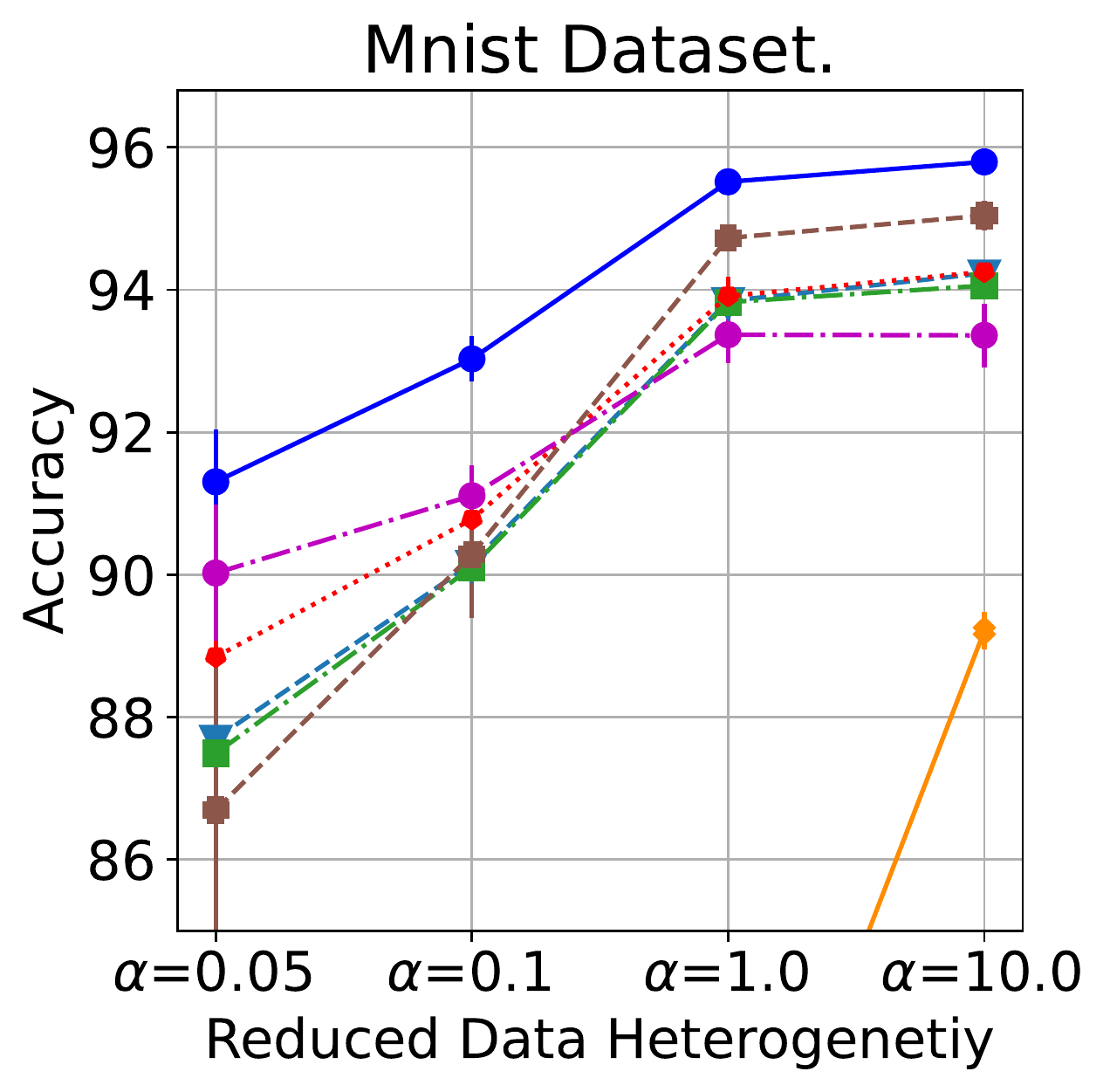}}
        \subcaption{\Mdata on CNN.}\label{fig:mnist-batch-cnn}
        \end{subfigure} 
        %
        \begin{subfigure}[b]{0.21\textwidth}
        \centerline{\includegraphics[width=\columnwidth]{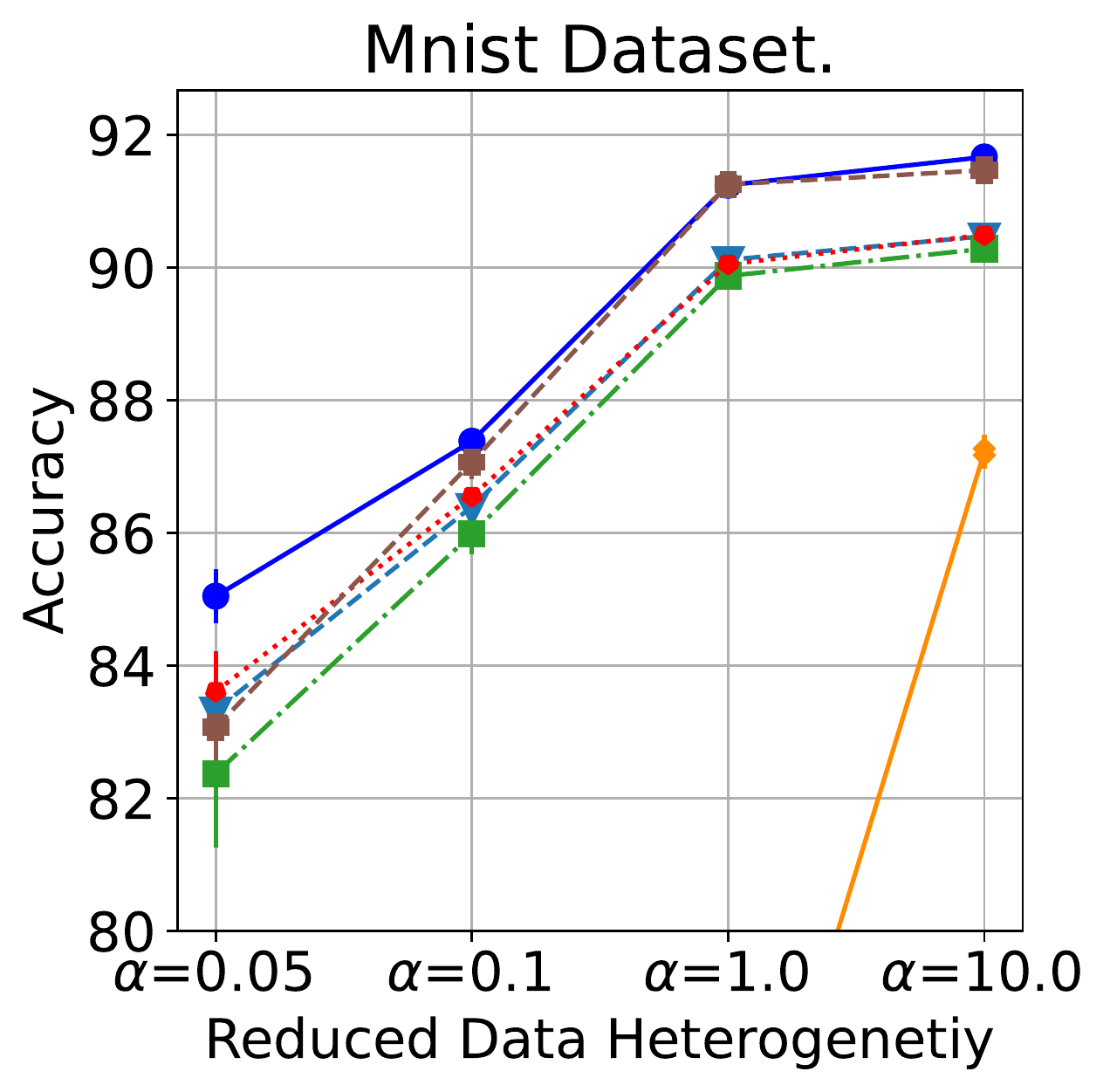}}
        \subcaption{\Mdata\ on MLP.}\label{fig:mnist-batch-dnn}
        \end{subfigure} 
    \end{center}
    \vspace{-0.2in}
    \caption{Visualized performance w.r.t data heterogeneity.}\label{fig:heterogeneity-effects}
\end{figure*}

\begin{figure*}[hbt!]  
    \begin{center}
        \begin{subfigure}[b]{0.21\textwidth}
        \centerline{\includegraphics[width=\columnwidth]{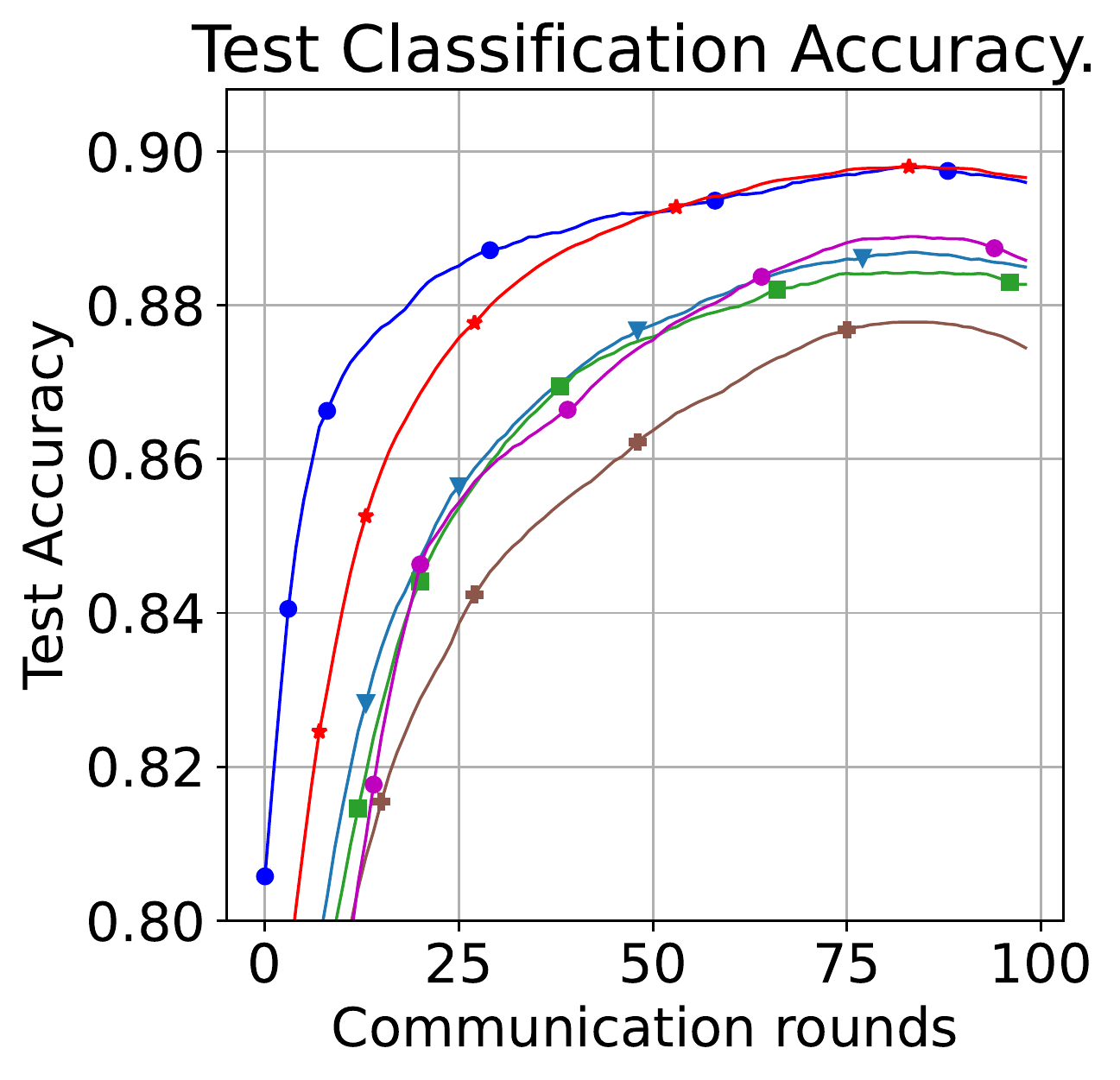}}
        \subcaption{\Cdata, $r=10/25$.}\label{fig:celeb-u25a5}
        \end{subfigure}  
        \begin{subfigure}[b]{0.3\textwidth}
        \centerline{\includegraphics[width=\columnwidth]{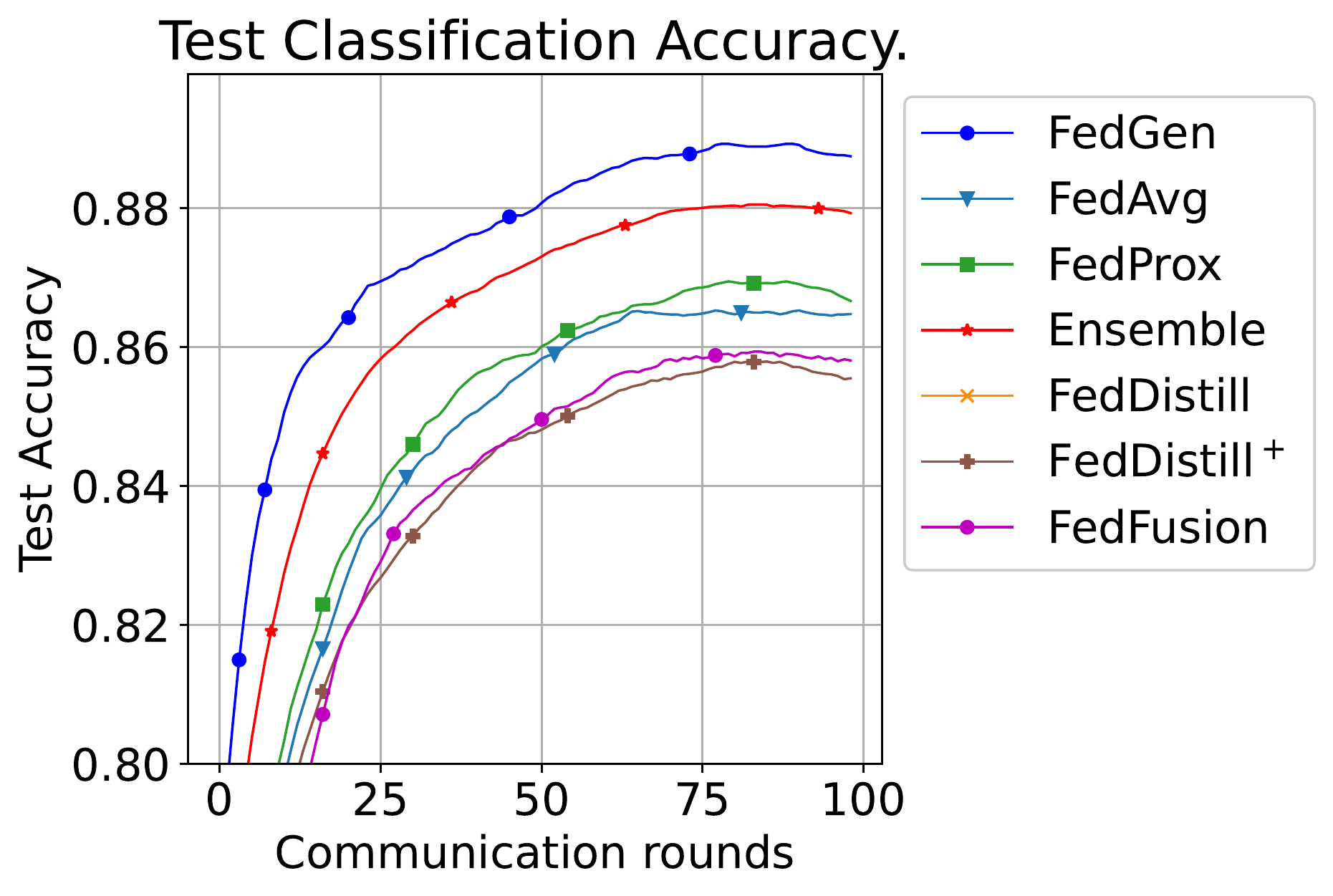}} 
        \subcaption{\Cdata, $r=5/10$.}\label{fig:celeb-u10a5}
        \end{subfigure} 
        \hspace{-0.1in}
        \begin{subfigure}[b]{0.21\textwidth}
            \centerline{\includegraphics[width=\columnwidth]{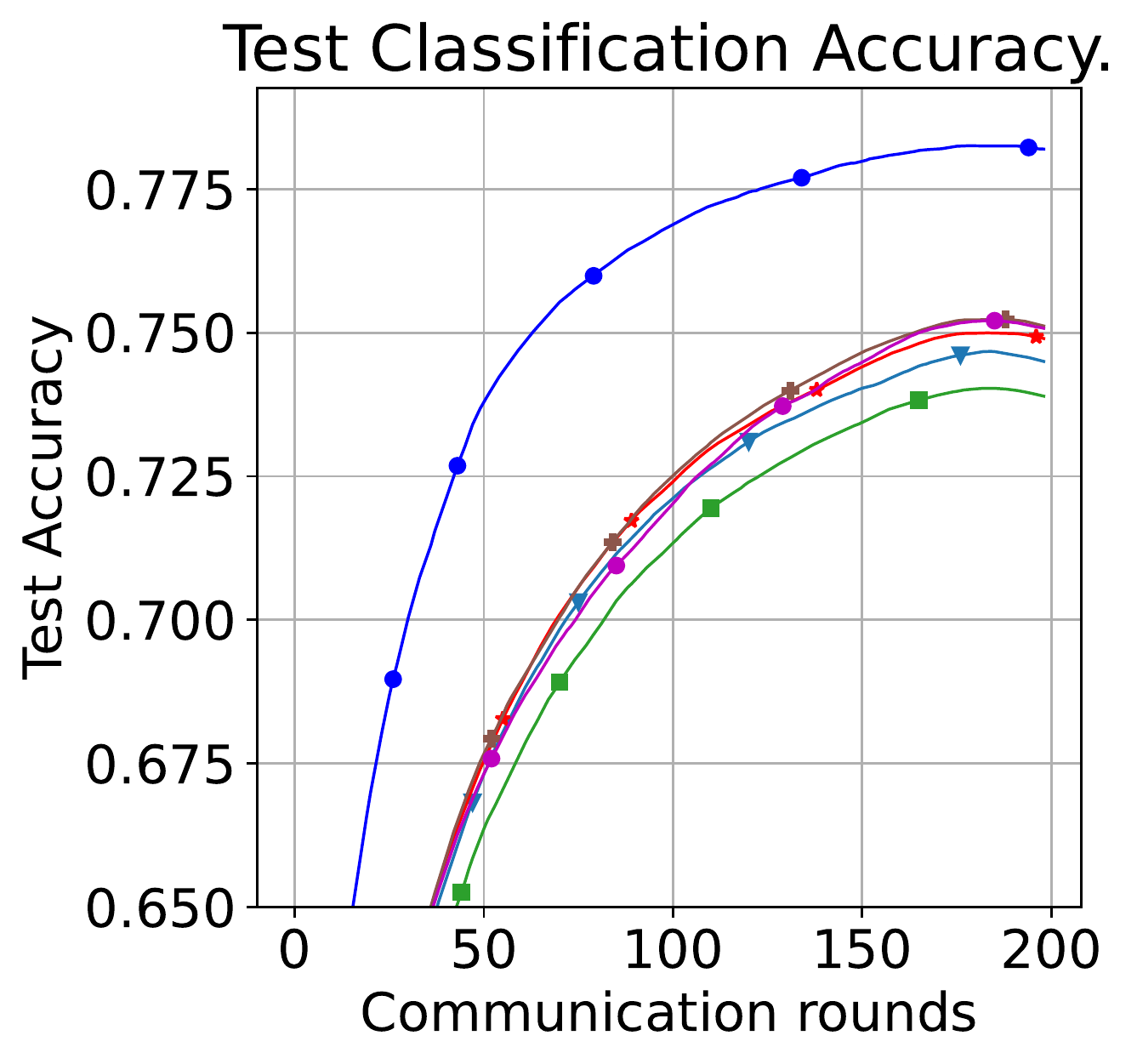}}
            \subcaption{\Edata, $\alpha=1$.}\label{fig:emnist-alpha1-c20}
            \end{subfigure} %
        \begin{subfigure}[b]{0.21\textwidth}
        \centerline{\includegraphics[width=\columnwidth]{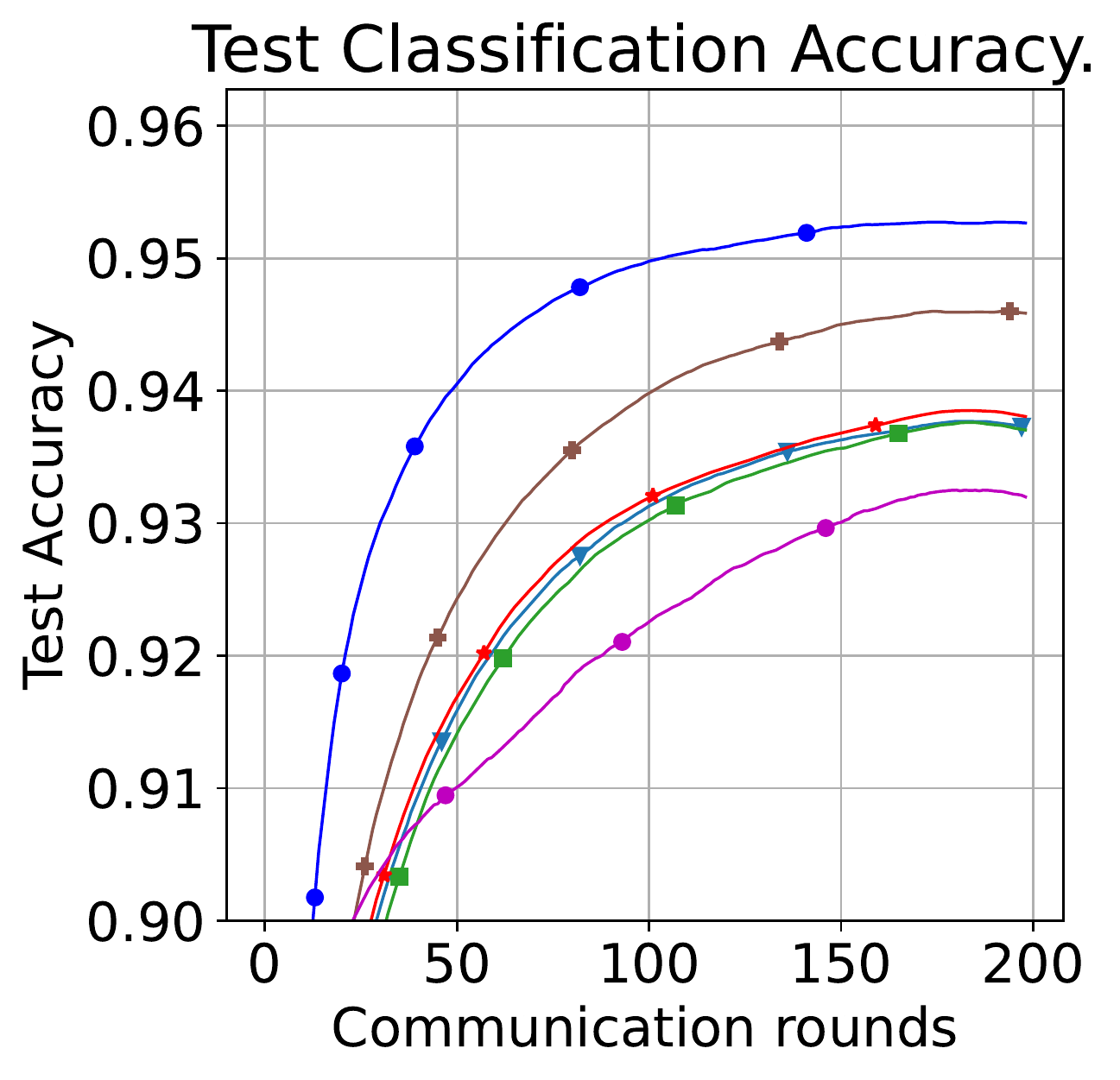}}
        \subcaption{\Mdata, $\alpha=1$.}\label{fig:mnist-alpha1}
        \end{subfigure}  
    \end{center}
    \vspace{-0.2in}
    \caption{Selected learning curves, averaged over 3 random seeds.} 
    \label{fig:learning-curves}
    \vspace{-0.25in}
\end{figure*}

\vspace{0.1in}
\subsection{Setup} \label{sec:exp-setup}
\vspace{-0.1in}
\textbf{Baselines:}
In addition to \textbf{\textsc{FedAvg}} ~\cite{mcmahan2017communication},
\textbf{\textsc{FedProx}} regularizes the local model training with a proximal term in the model objective~\cite{li2020federated}.
\textbf{\textsc{FedEnsemble}} extends \textsc{FedAvg} to ensemble the prediction output of all user models. 
%
\textbf{\textsc{FedDFusion}} is a data-based KD approach~\cite{lin2020ensemble}, for which we provide unlabeled training samples as the proxy dataset. 
%
\textbf{\textsc{FedDistill}}~\cite{jeong2018communication} is a data-free KD approach which shares label-wise average of logit-vectors among users.
It does not share network parameters and therefore experience non-negligible performance drops compared with other baselines.
For a fair comparison, we derive a baseline from \FD, which shares both model parameters and the label-wise logit vectors. We name this stronger baseline as \textbf{\FDFL}.  

\textbf{Dataset:} We conduct experiments on three image datasets: \textbf{\Mdata}~\cite{lecun-mnisthandwrittendigit-2010}, \textbf{\Edata}~\cite{cohen2017emnist}, and \textbf{\Cdata}~\cite{liu2015faceattributes}, as suggested by the \textsc{Leaf} FL benchmark~\cite{caldas2018leaf}. 
Among them, \Mdata~and \Edata~dataset is for digit and character image classifications, and \Cdata~is a celebrity-face dataset which is used to learn a binary-classification task, \ie~to predict whether the celebrity in the picture is smiling.

\textbf{Configurations:} Unless otherwise mentioned, we run 200 global communication rounds, with 20 user models in total and an active-user ratio $r = 50\%$. 
We adopt a local updating step $T=20$, and each step uses a mini batch with size $B=32$. 
We use at most $50\%$ of the total training dataset and distribute it to user models, and use all testing dataset for performance evaluation.
For the classifier, we follow the network architecture of \cite{mcmahan2017communication}, and treat the last MLP layer as the predictor $\vtheta_k^p$ and all previous layers as the feature extractor $\vtheta_k^f$. 
The generator $G_\vw$ is MLP based. It takes a noise vector $\epsilon$ and an one-hot label vector $y$ as the input, which, after a hidden layer with dimension $d_h$, outputs a feature representation with dimension $d$.
To further increase the diversity of the generator output, we also leverage the idea of \textit{diversity loss} from prior work~\cite{mao2019mode} to train the generator model.


\textbf{User heterogeneity}: for \Mdata~and \Edata~ dataset, we follow prior arts~\cite{lin2020ensemble,hsu2019measuring}  to model non-iid data distributions using a Dirichlet distribution \textbf{Dir}$(\alpha)$, in which a smaller $\alpha$ indicates higher data heterogeneity, as it makes the distribution of $p_k(y)$ more biased for a user $k$.  
We visualize the effects of adopting different $\alpha$ on the statistical heterogeneity for the \Mdata\ dataset in Figure \ref{fig:label-heterogeneity}.
For \Cdata, the raw data is naturally non-iid distributed. We further increase the data heterogeneity by aggregating  pictures belonging to different celebrities into disjoint groups, with each group assigned to one user. 

\vspace{-0.1in}
\subsection{Performance Overview:}\label{sec:exp-overview}
\vspace{-0.1in}
From Table \ref{table:performance-overview}, we can observe that \approach~outperforms other baselines with a considerable margin. 

\textbf{Impacts of data heterogeneity:}
\approach~is the only algorithm that is robust against different levels of user heterogeneity while consistently performs well. 
As shown in Figure~\ref{fig:heterogeneity-effects}, the gain of \approach~is more notable when the data distributions are highly heterogeneous (with a small $\alpha$).
This result verifies our motivations, since the advantage of \approach~is induced from the knowledge distilled to local users, which mitigates the discrepancy of latent distributions across users.
This knowledge is otherwise not accessible by baselines such as \Avg\ or \Prox. 

As one of most competitive baselines, the advantage of \Fusion\ vanishes as data heterogeneity becomes mitigated, which gradually becomes comparable to \Avg, as shown in Figure \ref{fig:emnist-batch} and Figure \ref{fig:mnist-batch-cnn}. 
Unlike \Fusion, the performance gain of our approach is consistently significant, which outperforms \Fusion\ in most cases.
This discrepancy implies that our proposed approach, which directly distills the knowledge to user models, can be more effective than fine-tuning the global model using a proxy dataset, especially when the distilled knowledge contains inductive bias to guide local model learning.

As a data-free KD baseline, \FD~experiences non-negligible performance drops, which implies the importance of parameter sharing in FL.
%
%
\FDFL, on the other hand, is vulnerable to data heterogeneities.
As shown in Table \ref{table:performance-overview}, it can outperform \Avg~ when data distributions are near-iid (\eg~when $\alpha \geq 1$ ), thanks to the shared logit statistics as the distilled knowledge, but performs worse than \Avg~when $\alpha$ gets smaller, which indicates that sharing such meta-data alone may not be effective enough to confront user heterogeneity.

\Ensemble~ enjoys the benefit of ensemble predictions from all user models, although its gain is less significant compared with \approach.
We ascribe the leading performance of our approach to the better generalized performance of local models.
Guided by the distilled knowledge, a user model in \approach~can quickly jump out of its local optimum, whose aggregation can be better than the ensemble of potentially biased models as in \Ensemble.

\textbf{Learning efficiency:}
As shown in Figure \ref{fig:learning-curves}, \approach~has the most rapid learning curves to reach a performance and outperforms other baselines.
Although \Fusion~enjoys a learning efficiency higher than other baselines under certain data settings, due the advantages induced from a proxy data,
our approach can directly benefit each local user with actively learned knowledge, whose effect is more explicit and consistent (More illustrations in supplementary). 

{\textbf{Comments on sharing generative model}: Given a compact latent space, the generative model can be lightweight for learning or downloading. 
In practice, we use a generator network with 2 compact MLP layers, whose parameter size is small compared with the user classification model. 
The above empirical results also indicates that the leading performance gain combined with a faster convergence rate can trade off the communication load brought by sharing a generative model.

\begin{figure*}[hbt!]  
    \scalebox{0.9}{
    \begin{minipage}{0.3\textwidth}
      \centering
        \includegraphics[width=\linewidth]{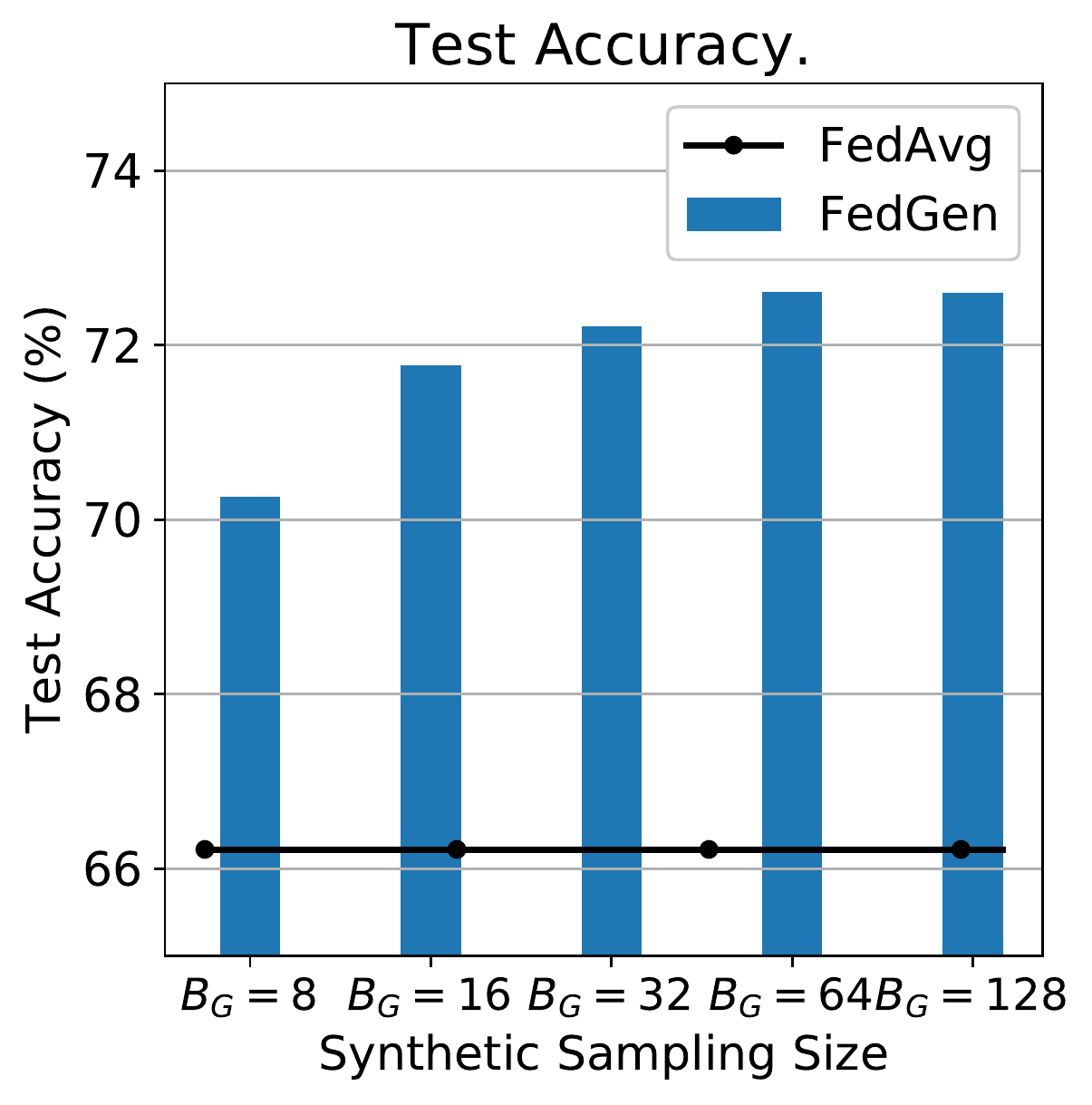}
        \vspace{-0.3in}
      \captionof{figure}{Effects of synthetic samples. \label{fig:samples-analysis}}
    \end{minipage}}
    \hfill
    \scalebox{0.8}{
    \begin{minipage}{0.9\linewidth}
      \centering
        \begin{tabular}{cccccc}
          \toprule
          \multicolumn{1}{c}{} & \multicolumn{5}{c}{\textbf{Effects of the Generator Network Structure.}}  \\ 
          %
          {[$d_\epsilon$, $d_h$]}
          & \multicolumn{1}{c}{$[64,256]$ } 
          & \multicolumn{1}{c}{$[32,256]$} 
          & \multicolumn{1}{c}{$[32,128]$} 
          & \multicolumn{1}{c}{$[16,128]$} 
          & \multicolumn{1}{c}{$[32,64]$} \\  \hline  
          \multicolumn{1}{c}{\textbf{Accuracy(\%)}} & \multicolumn{4}{c}{\Avg=66.22$\pm$2.58}  \\ 
  
          \approach
          & 71.61$\pm$0.25 & 72.09$\pm$0.46 & 72.43$\pm$0.57 & 72.01$\pm$0.76  & 70.98$\pm$0.85       \\            %
          \bottomrule
          \end{tabular}
       \vspace{-0.1in}
      \captionof{table}{Effects of the generator's network structure, using \Edata\ dataset with $\alpha=0.1$\label{table:sensitivity-on-network}.}
      \begin{tabular}{cccccc}
        \toprule
        \multicolumn{1}{c}{} & \multicolumn{4}{c}{\textbf{Performance w.r.t different synthetic sample sizes.}}  \\ 
        \multicolumn{1}{c}{\textbf{Generator sampling size}} 
        & \multicolumn{1}{c}{$B_G=8$} 
        & \multicolumn{1}{c}{$B_G=16$} 
        & \multicolumn{1}{c}{$B_G=32$} 
        & \multicolumn{1}{c}{$B_G=64$} 
        & \multicolumn{1}{c}{$B_G=128$} \\  \hline 

        %
         {\textbf{Local training time (ms)}}  & \multicolumn{4}{c}{\Avg$=47.66\pm1.68$} \\ 
         \approach 
         &57.20$\pm$2.22 &  57.39$\pm$2.21 & 58.17$\pm$2.24  & 58.91$\pm$2.29 & 60.06$\pm$2.32 \\   
         \bottomrule 
        \end{tabular}
       \vspace{-0.1in}
       \captionof{table}{Effects of the number of synthetic samples, using \Edata\ dataset with $\alpha=0.1$. \label{table:sensitivity-on-gen-samples}}
      \end{minipage}}
    \end{figure*}

\vspace{-0.1in}
\subsection{Sensitivity Analysis}
\vspace{-0.05in}

\textbf{Impacts of straggler users:}  
We explore different numbers of total users versus active users on the \textsc{CelebA} dataset, with the active ratios $r$ ranging from $0.2$ to $0.9$. 
Figure \ref{fig:celeb-batch} shows that our approach is next to \Ensemble~when the number of straggler users are high ($r= 0.2$, with 5 out 25 active users per learning round), and is consistently better than all baselines w.r.t to the asymptotic performance given a moderate number of active users.
Combined with Figure~\ref{fig:celeb-u25a5} and Figure~\ref{fig:celeb-u10a5}, one can observe that our approach requires  much less communication rounds to reach high performance, regardless of the setting of straggler users.

\textbf{Effects of different network architectures:} we conduct analysis on the \Mdata\ dataset, using both CNN and MLP network architectures.
%
As shown in Figure~\ref{fig:mnist-batch-dnn} and Figure~\ref{fig:mnist-batch-cnn}, the outstanding performance of \approach~ is consistent across two different network settings, although the overall performance trained with CNN networks is noticeably higher than those with MLP networks.

\textbf{Effects of communication frequency:} We explore different local updating steps $T$ on the \Edata, so that a higher $T$ means longer communication delays before the global communication.
Results in Table \ref{table:performance-overview} indicates that our approach is robust against different levels of communication delays (See supplementary for more results).

\begin{figure}[hbt!]  
    \begin{center}
        \hspace{-0.1in} 
        \begin{subfigure}[b]{0.2\textwidth}
            \centerline{\includegraphics[width=\columnwidth]{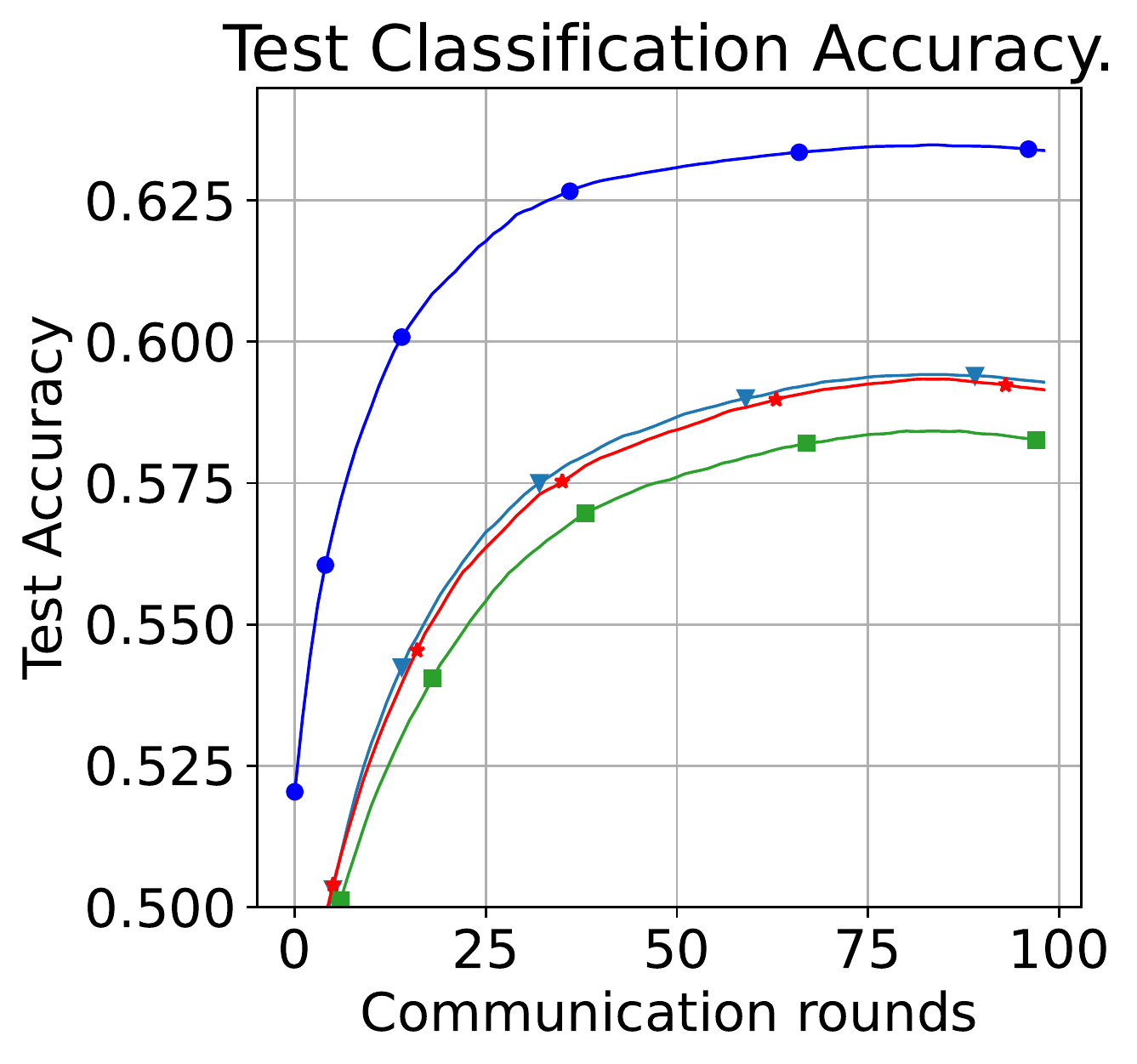}}
        \subcaption{$\alpha=0.05$}
        \end{subfigure} 
        \begin{subfigure}[b]{0.28\textwidth}
            \centerline{\includegraphics[width=\columnwidth]{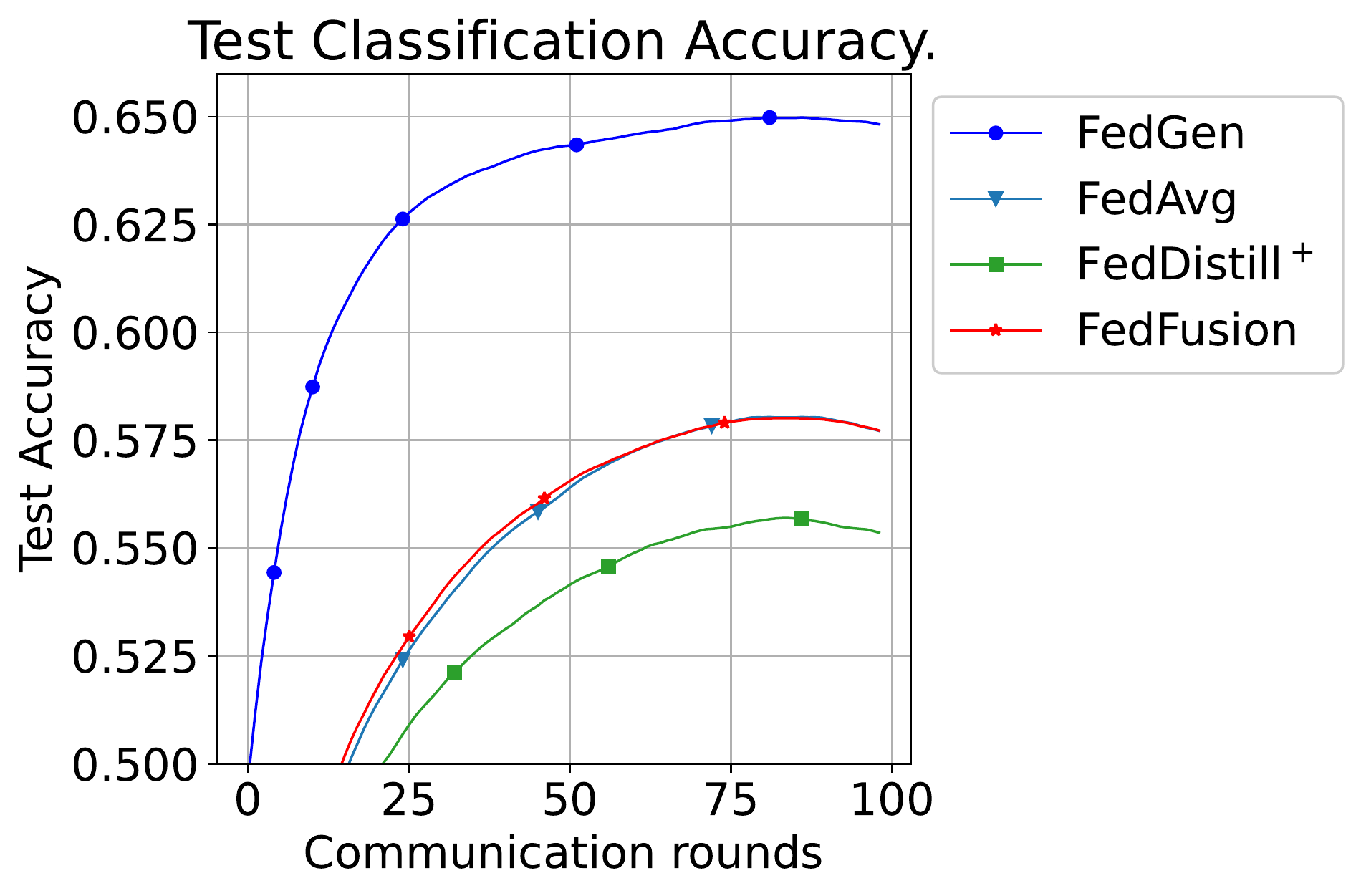}}
        \subcaption{$\alpha=0.1$}
        \end{subfigure} 

        \hspace{-0.6in}
        \begin{subfigure}[b]{0.2\textwidth}
            \centerline{\includegraphics[width=\columnwidth]{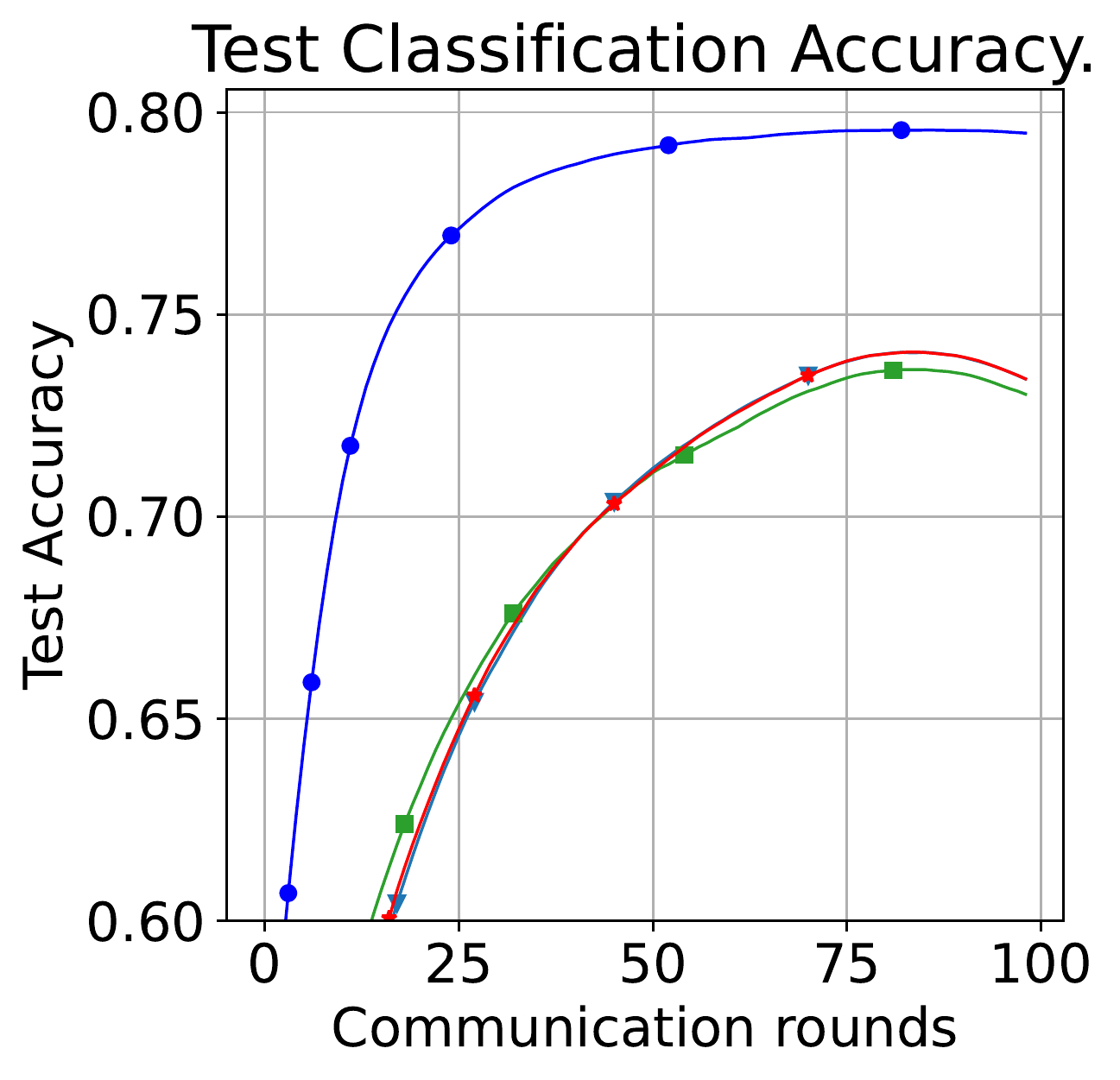}}
        \subcaption{$\alpha=1$}
        \end{subfigure} 
        \begin{subfigure}[b]{0.2\textwidth}
            \centerline{\includegraphics[width=\columnwidth]{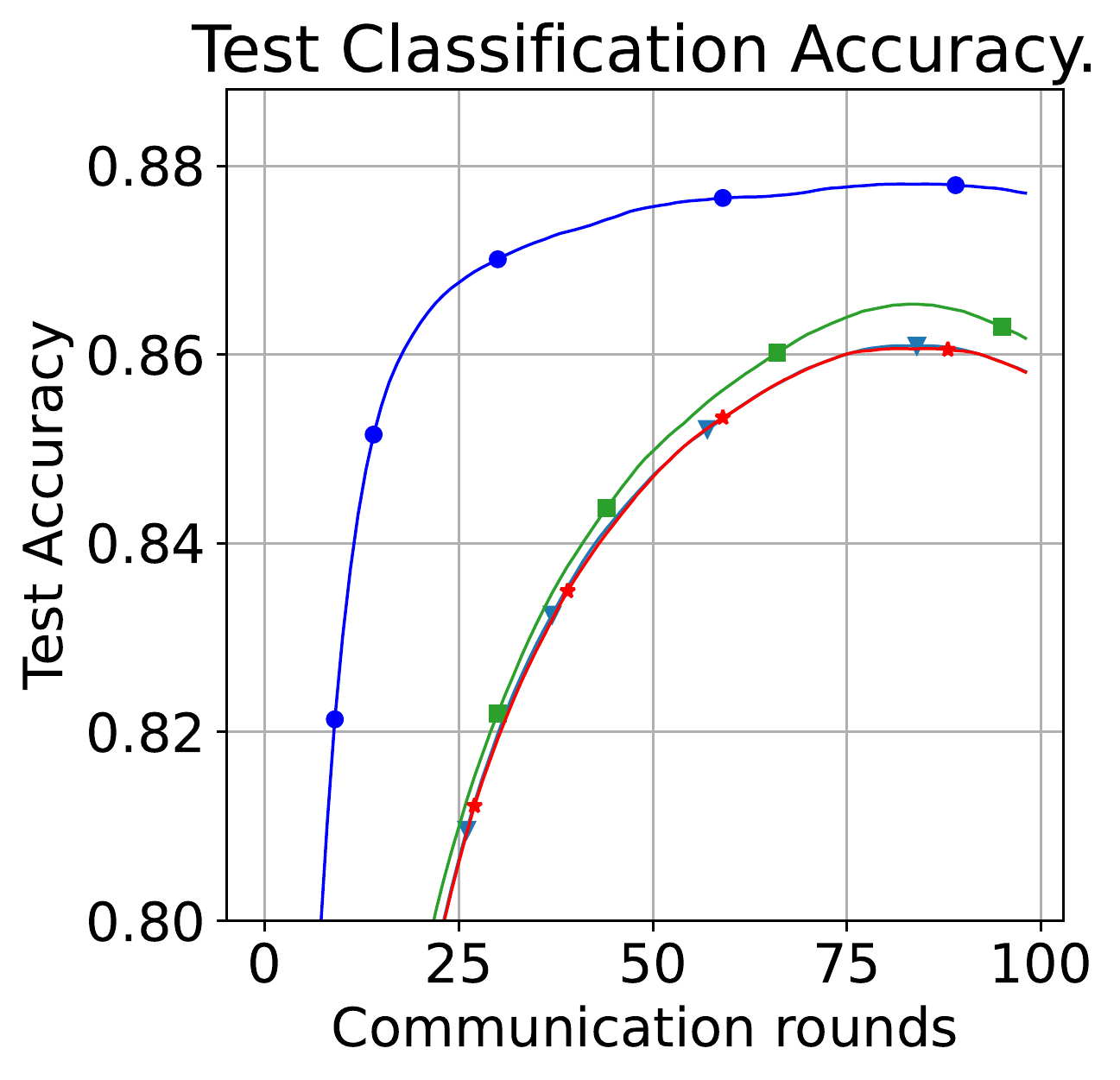}}
        \subcaption{$\alpha=10$}
        \end{subfigure} 
    \end{center}
    \vspace{-0.2in}
    \caption{Learning curves on \Mdata\ with limited param. sharing.}\label{fig:mnist-partial-sharing}
\end{figure} 
  
\begin{table}[htb!]
    \begin{center}
        \scalebox{0.7}{
            \setlength\tabcolsep{4pt}
            \begin{tabular}{lcccc}
                \toprule
                \multicolumn{1}{l}{}  &  & \multicolumn{1}{c}{\textbf{Top 1 Accuracy ($\%$)}}  \\ \hline 
                \multicolumn{1}{c}{Algorithms} 
                & \multicolumn{1}{c}{\Avg} 
                & \multicolumn{1}{c}{\textsc{\FDFL}} 
                & \multicolumn{1}{c}{\Fusion} 
                & \multicolumn{1}{c}{\approach} \\ \midrule 

                $\alpha$ = 0.05  & 59.67$\pm$0.76 & 58.83$\pm$0.62  & 59.62$\pm$0.84 & \textbf{63.60$\pm$0.63} \\
                $\alpha$ = 0.1 & 58.39$\pm$0.74 & 56.25$\pm$0.98  & 58.38$\pm$0.81  & \textbf{65.42$\pm$0.29} \\  
                $\alpha$ = 1 & 74.49$\pm$ 0.57 & 74.24$\pm$0.60 & 74.51$\pm$0.55  & \textbf{79.72$\pm$0.52}  \\ 
                $\alpha$ = 10 &86.35$\pm$0.60& 86.89$\pm$0.26  & 86.28$\pm$0.69  & \textbf{87.92$\pm$0.46} \\  
                \bottomrule
                \end{tabular}
        } 
    \vspace{-0.1in}
    \caption{Performance overview on \Mdata, by only sharing the last prediction layer.\label{table:partial-sharing}}
\end{center} 
\vspace{-0.1in}
\end{table}

\textbf{Effects of the generator's network architecture and sampling size:} 
Extended analysis has verified that \approach\ is \textit{{robust}} across different generator {network architectures} (Table~\ref{table:sensitivity-on-network}).
Moreover, sampling synthetic data from the generator only adds minor training workload to local users (Table ~\ref{table:sensitivity-on-network}).
The gain of \approach\ over \Avg\ is consistently remarkable given {different synthetic sample sizes}, whereas a sufficient number of synthetic samples brings even better performance (Figure~\ref{fig:samples-analysis}). 
%
Especially, in Table ~\ref{table:sensitivity-on-network}, we explored different dimensions for the input noise ($d_\epsilon$) and the hidden layer ($d_{h}$) of the generator while keeping its output layer dimension fixed (\ie\ the dimension of the feature space $\gZ$). 
Table~\ref{table:sensitivity-on-gen-samples} shows the training time for one local update, averaged across users and the communication rounds. $B_G$ denotes the number of \textit{synthetic} samples used for each mini-batch optimization. By default, we set $B_G=B$, and $B$ is the number of \textit{real} samples drawn from the local dataset (see Algorithm 1).

\vspace{-0.05in}
\subsection{Extensions to Flexible Parameter Sharing} \label{sec:exp-partial-sharing} 
\vspace{-0.1in}
Motivated to alleviate privacy and communication concerns, \approach~ is ready to benefit distributed learning without sharing entire model parameters.
To explore this potential, we conduct a case study on \Avg, \FDFL, and \approach, where user models share only the last prediction layer and keep their feature extraction layers localized.
Note that \Fusion\ is not designed to address FL with partial parameter sharing, which requires entire user models for KD. 
For a fair comparison, we modify \Fusion~to let it upload entire user models during the model aggregation phase, but disable the downloading of feature extractors, so that the server model can still be fine-tuned using the proxy data.

Results in Table \ref{table:partial-sharing} show that our approach consistently outperforms other baselines by a remarkable margin, the trend of which is more significant given high data heterogeneity (Figure~\ref{fig:mnist-partial-sharing}).
Its distinguished performance from \Fusion~verifies the efficacy of data-free distillation under this challenging scenario.
This promising results show that \approach~has the potential to further reduce communication workload, not only by fast convergence but also by a flexible parameter sharing strategy.

\vspace{-0.1in}
\section{Conclusions}
\vspace{-0.1in}
In this paper, we propose an FL paradigm that enables efficient knowledge distillation to address user heterogeneity without requiring any external data.
%
Extensive empirical experiments, guided by theoretical implications, have shown that our proposed approach can benefit federated learning with better generalization performance using less communication rounds, compared with the state-of-the-art.

\vspace{-0.1in}
\subsection*{Acknowledgments}
\vspace{-0.1in}
This research was jointly supported by the National Science Foundation IIS-1749940, the Office of Naval Research N00014-20-1-2382, and the National Institue on Aging RF1AG072449.

\bibliography{ref}


\onecolumn
\icmltitle{Data-Free Knowledge Distillation for Heterogeneous Federated Learning: Supplementary Document}

\section{Theoretical Derivations}
\subsection{Notations and Preliminaries}
Let $\gX \subset \R^p$ be the \textit{input} space, $\gZ \subset \sR^d$ be the {\textit{latent}} feature space, and $\gY \subset \sR$ be the output space.
$\gR:  \gX \to \gZ$ denotes a \textbf{\textit{representation function}} that maps inputs into features.
$\gT$ denotes a \textbf{\textit{domain}} (or \textit{task}), which consists of a data distribution $\gD$ over $\gX$ and a ground-truth {\textit{labeling}} function $c^*: \gX \to \gY $.
%
%
Given a domain $\gT := \langle \gD, c^* \rangle$ and a representation function $\gR$, 
we use $\Dz$ to denote the \textit{induced image} of $\gD$ under $\gR$~\cite{ben2007analysis}, \st~given a probability event $\gB$,
\begin{align*}
\E_{z \sim \Dz}[\gB(z)]=\E_{x\sim \gD}[\gB(\gR(x))].
\end{align*}
Accordingly, ${\cz}$ denotes the {\textit{induced}} labeling function under $\gR$:
\begin{align*}
    \cz(z) := \E_{x \sim D}\left[ c^*(x) | \gR(x) = z \right].
\end{align*}
Let $h: \gZ \to \gY $ denote a \textbf{\textit{hypothesis}} that maps features to predicted labels, and $\gH \subseteq \{h: \gZ \to \gY \}$ denote a hypothesis class.
For our analysis, we assume the FL tasks are for binary classification, \ie~$\gY=\{0,1\}$, and the loss function is 0-1 bounded, with $l(\hat{y},y)=|\hat{y}-y|$. Same assumptions have been adopted by various prior art~\cite{ben2007analysis,blitzer2008learning,ben2010theory,lin2020ensemble,ben2007analysis}.

Given two distributions $\gD$ and $\gD'$, $d_\gH(\gD, \gD')$ is defined as the $\gH$-divergence between $\gD$ and $\gD'$, \ie:
\begin{align*}
d_\gH(\gD, \gD'):=2 \sup_{\gA \in \gA_{\gH}} | \rp_{\gD}(\gA) - \rp_{\gD'}(\gA) | \}, 
\end{align*}
where $\gA_{\gH}$ is a set of measurable subsets under $\gD$ and $\gD'$ for certain $h \in \gH$.
Moreover, $\gH \triangle \gH$ is defined as the symmetric difference hypothesis space~\cite{blitzer2008learning}, \ie:
\begin{align*}
    \gH \triangle \gH := \left \{  h(z) \oplus h'(z), h, h' \in \gH \right \}
\end{align*}
where $\oplus$ denotes the XOR operator, so that $h(z) \oplus h'(z)$ indicates that $h$ and $h'$ disagrees with each other.
Accordingly, $\gA_\HH$ is a set of measurable subsets for $\forall~h(z) \oplus h'(z) \in \HH$.
Then $d_{\gH \triangle \gH}(\cdot, \cdot)$ is defined as the \textit{distribution divergence}  induced by the symmetric difference hypothesis space~\cite{blitzer2008learning}:
\begin{align*}
   \dH(\gD, \gD'):=2 \sup_{\gA \in \gA_{\HH}} | \rp_{\gD}(\gA) - \rp_{\gD'}(\gA) | \}.
\end{align*}

Specifically, let $\gD, \gD'$ be two arbitrary distributions on the input space $\gX$, and let $\Dz, \Dz'$ be their induced images over $\gR$.
Then based on the definition of $\dH(\cdot,\cdot)$, one can have: 
\begin{align*} 
\dH (\Dz,\Dz') &= 2 \sup_{ \gA \in \gA_\HH }  \left \vert \E_{x \sim \gD}\left[\rp(\gA(\gR(x))) \right] - \E_{x \sim \gD'} \left[\rp(\gA(\gR(x))) \right] \right \vert \\
 &= 2 \sup_{ \gA \in \gA_\HH }  \left \vert \E_{z \sim \Dz}\left[\rp(\gA(x)) \right] - \E_{z \sim \Dz'} \left[\rp(\gA(z)) \right] \right \vert \\
&= 2 \sup_{\gA \in \gA_{\HH}} | \rp_{\Dz}(\gA) - \rp_{\Dz'}(\gA) | \}. 
\end{align*}

\subsection{Derivations of Remark 1}

\begin{remark*}
    %
    Let $p(y)$ be the prior distribution of labels, and $r(z|y): \gY \to \gZ$ be the conditional distribution derived from generator $G_\vw$.
    Then regulating a user model $\vtheta_k$ using samples from  $r(z|y)$ can minimize the conditional KL-divergence between two distributions, derived from the user and from the generator, respectively: 
    \begin{align*} 
        \max_{\vtheta_k}~\E_{y \sim p(y),z \sim r(z|y)} \left[ \log p(y|z;{\vtheta_k}) \right]   \equiv \min_{\vtheta_k}~{\KL}[ r(z|y) \Vert p(z|y;\vtheta_k)],
    \end{align*} 
\end{remark*}
\begin{proof}
    Expanding the KL-divergence, we have 
    \begin{align*}
        \because ~& {\KL}[ r(z|y) \Vert p(z|y;\vtheta_k)] \equiv \E_{y \sim p(y)} \left[ \E_{z \sim r(z|y)} \left[ \log \frac{ r(z|y)}{ p(z|y;\vtheta)} \right] \right] \\
        & = \E_{y \sim p(y)} \E_{z \sim r(z|y)} \left[ \log r(z|y) \right ] - \E_{y \sim p(y)} \E_{z \sim r(z|y)} \left[   { \log p(z|y;\vtheta)} \right]\\
         & = - \underset{\text{constant w.r.t~} \vtheta_k}{H(r(z|y))} - \E_{y \sim p(y)} \E_{z\sim r(z|y)} [ \log p(z|y;\vtheta)]. 
     \end{align*}
     where $H(r(z|y))$ is constant w.r.t $\vtheta_k$. Therefore when optimizing $\vtheta_k$ we have:
     \begin{align*}
      ~& \min_{\vtheta_k}~{\KL}[ r(z|y) \Vert p(z|y;\vtheta_k)] \\
        \equiv & \min_{\vtheta_k}~ -\E_{y \sim p(y),z \sim r(z|y)} \left[  \log p(z|y;{\vtheta_k}) \right]  \\
         \equiv & \max_{\vtheta_k} \E_{y \sim p(y)} \E_{z\sim r(z|y)} [ \log \frac{p(y|z;\vtheta_k)  p(z)}{p(y)}   ] \\
         \equiv & \max_{\vtheta_k} \E_{y \sim p(y)} \E_{z\sim r(z|y)} [ \log p(y|z;\vtheta_k) + {\log p(z)- \log p(y)}]  \\  
         \equiv & \max_{\vtheta_k} \E_{y \sim p(y)} \E_{z\sim r(z|y)} [ \log p(y|z;\vtheta_k) ] . 
    \end{align*}
   where $H(r(z|y))$ denotes the entropy of the probability distribution $r(z|y)$ which is \textit{not} optimizable w.r.t $\vtheta_k$, and $p(z|y;\vtheta_k):=\frac{p(y|z;\vtheta_k)p(z)}{p(y)}$ is defined as the probability that the input representation to the predictor is $z$ if it yields a label $y$.
\end{proof}


\subsection{Derivations of Theorem 1}

Before deriving Theorem 1, we first present an upper-bound for the generalization performance from prior art~\cite{ben2007analysis}, which analyzes the role of a feature representation function in the context of \textit{domain adaptation}:


\begin{lemma}  \label{lemma:da-bounds}
    \textnormal{\textbf{Generalization Bounds for Domain Adaptation \cite{ben2007analysis,blitzer2008learning}:}}

    Let $\gT_S$ and $\gT_T$ be the source and target domains, whose data distributions are $\gD_S$ and $\gD_T$.
    Let $\gR: \gX \to \gZ$ be a feature representation function, and $\Dz_S, \Dz_T$ be the induced images of $\gD_S$ and $\gD_T$  over $\gR$, respectively. 
     Let $\gH$ be a set of hypothesis with VC-dimension $d$.  
    Then with probability at least $1 - \delta$, $\forall~h \in \gH$:
    %
    \begin{align} 
    \gL_{\gT_T}(h) &\leq \Le_{\gT_S}(h) %
     + \sqrt{\frac{4}{m} \left( d \log\frac{2 e m }{d} + \log \frac{4}{\delta} \right) }  
     + \dH (\Dz_{S}, \Dz_{T}) + \lambda,
    \end{align}
\end{lemma}
where $e$ is the base of the natural logarithm, $\Le_{\gT_S}$(h) is the empirical risk of the source domain given $m$ observable samples, and $\lambda = \min_{h\in \gH} \left( \gL_{\gT_T}(h) + \gL_{\gT_S}(h) \right )$ is the optimal risk on the two domains. 

One insight from Lemma \ref{lemma:da-bounds} is that a good representation function plays a tradeoff between minimizing the empirical risk ($\Le_{\gT_S}(h)$) and the induced distributional discrepancy ( $d_{\gH \triangle \gH}(\Dz_{S}, \Dz_{T})$).
%
Based on Lemma~\ref{lemma:da-bounds}, one can establish Theorem 1 as the following:
\begin{theorem*} \textnormal{(\textbf{Generalization Bounds for FL})}
    Consider an FL system with $K$ users.
    Let $\gT_k=\langle \gD_k, c^* \rangle$ and $\gT=\langle \gD, c^* \rangle$ be the $k$-th local domain and the global domain, respectively.
    Let $\gR: \gX \to \gZ$ be a feature extraction function that is simultaneously shared among users.
    Let $h_k$ denote the hypothesis learned on domain $\gT_k$, and $h=\frac{1}{K}\sum_{k=1}^K h_k$ be the global ensemble of user predictors.
    Then with probability at least $1 - \delta$:
    \vspace{-0.1in}
    \hspace{-0.2in}
    \begin{align}
    &\gL_{\gT}(h) 
    \leq \frac{1}{K} \sum_{k\in [K]} \Le_{\gT_k}(h_k) %
    + \frac{1}{K}\sum_{k \in [K]}  ( \dH (\Dz_{k}, \Dz) + \lambda_k  )
    + \sqrt{\frac{4}{m} \left( d \log\frac{2 e m }{d} + \log \frac{4K}{\delta} \right) },\nonumber %
    \end{align}
    %
    where $\Le_{\gT_k}(h_k)$ is the empirical risk of $h_k$, $\lambda_k := \min_h(\gL_{\gT_k}(h) + \gL_{\gT}(h))$ denotes an oracle performance on $\gT_k$ and $\gT$, and $\Dz_k$ and $\Dz$  is the \textbf{induced} image of $\gD_k$ and $\gD$ from $\gR$, respectively, \st~
    $\E_{z \sim \Dz_k}[\gB(z)]=\E_{x\sim \gD_k}[\gB(\gR(x))]$ given a probability event $\gB$, and so for $\Dz$. 
\end{theorem*}

\begin{proof}
    By treating each one of the local domains $k \in [K]$ as the \textit{source} and the global domain as the \textit{target}, one can have that, $\forall~\delta > 0$, with probability $1 - \frac{\delta} {K}$:
    \begin{align*}
        \gL_\gT(h_k) \leq \Le_{\gT_k}(h_k)+ \dH(\Dz_k, \Dz) + \lambda_k + \sqrt{\frac{4}{m} \left(d \log\frac{2 e m }{d}  + \log \frac{4K}{\delta} \right) }.
    \end{align*}
    Also, due to the convexity of risk function and Jesen inequality, one can have: $$\gL_{\gT}(h) \equiv \gL_{\gT} \left(\frac{1}{K}\sum_{k \in [K]} h_k \right)  \leq  \frac{1}{K} \sum_{k \in [K]} \gL_{\gT} (h_k).$$
    Therefore,
    \begin{align*}
        & \rp \left[ \gL_\gT(h) >  \frac{1}{K}  \sum_{k\in [K]} \left( \Le_{\gT_k}(h_k) %
        + \sum_{k \in [K]}  ( \dH (\Dz_{k}, \Dz) + \lambda_k  )
        + \sqrt{\frac{4}{m}\left( d \log\frac{2 e m }{d} + \log \frac{4K}{\delta}\right) } \right)
         \right] \\
         \leq &\rp \left[ \frac{1}{K} \sum_{k \in [K]} \gL_{\gT} (h_k) > 
         \frac{1}{K} \sum_{k\in [K]} \left ( \Le_{\gT_k}(h_k) %
        + \sum_{k \in [K]}  ( \dH (\Dz_{k}, \Dz) + \lambda_k  )
        + \sqrt{\frac{4}{m}\left( d \log\frac{2 e m }{d} + \log \frac{4K}{\delta}\right) }
          \right) \right]\\
       \leq &\rp \left[ \bigvee_{k \in [K]}  \gL_\gT(h_k)  > 
          \Le_{\gT_k}(h_k) +  \dH (\Dz_{k}, \Dz) + \lambda_k  
        + \sqrt{\frac{4}{m}\left( d \log\frac{2 e m }{d} + \log \frac{4K}{\delta}\right) }
       \right] \\%
        \leq & \sum_{k \in [K]} \frac{\delta}{K} = \delta.
    \end{align*}
\end{proof}
%
Theorem 1 shows that the performance of the aggregated hypothesis is upper-bounded by:
1) the local performance of each user hypothesis ($\Le_{\gT_k}(h_k)$), 
2) the dissimilarity between the global and local distributions over the feature space ($\dH (\Dz_{k}, \Dz)$), 
3) the oracle performance ($\lambda_k$), 
and 4) the numerical constraints regarding the number of empirical samples $m$ and the VC-dimension $d$.

\subsection{Derivations of Corollary 1}

\begin{corollary*} 
    Let $\gT$, $\gT_k$, $\gR$ defined as in Theorem 1. 
    $\gD_\rA$ denotes an \textbf{augmented} data distribution, and $\gD_k'=\frac{1}{2}(\gD_k + \gD_\rA)$ is a \textbf{mixture} of distributions.
    Accordingly, $\Dz_\rA$ and $\Dz_k'$ denote the \textbf{induced} image of $\gD_\rA$ and  $\gD_k'$  over $\gR$, respectively.
    Let $\De_k'= \De_k \cup \De_\rA$ be an empirical dataset of $\gD_k'$, with $|\De_k|$=m, $|\De_k'| = |\De_k| + |\De_\rA|= m'$ .  
    Assume the discrepancy between $\Dz_\rA$ and $\Dz$ is bounded, s.t $\exists~ \epsilon >0, \dH (\Dz_A, \Dz) \leq \epsilon $,
    then with probability $1 - \delta$:
        \begin{align} \label{eq:corollary}
            \hspace{-0.1in}
            \gL_{\gT}(h) & \leq \frac{1}{K} \sum_k \gL_{\gT_k'}(h_k) %
        + \frac{1}{K}\sum_k  ( \dH (\Dz_k', \Dz) ) + \frac{1}{K}\sum_{k}\lambda_k'    
        + \sqrt{\frac{4}{m'} \left( d \log\frac{2 e m' }{d} + \log \frac{4K}{\delta} \right) }, 
        \end{align}
    where $\gT_k'=\{\gD_k', c^* \}$ is the updated local domain, 
    $\lambda_k' = \min_h(\gL_{\gT_k'}(h) + \gL_{\gT}(h))$ denotes the oracle performance,
    %
    %
    and  $\dH (\Dz_k', \Dz) \leq  \dH (\Dz_k, \Dz)$ when $\epsilon$ is small.
\end{corollary*}

\begin{proof}
    \Eqref{eq:corollary} can be directly derived by Theorem 1.
    We now focus on analyzing the relation between $\dH (\Dz_k, \Dz)$ and $\dH (\Dz_k', \Dz)$, which is the data dissimilarity \textbf{\textit{before}} and \textbf{\textit{after}} data augmentation using samples from distribution $\gD_A$, respectively.

Based on the definition of $\dH (\cdot, \cdot) $, one can derive that:
\begin{align*}
    \hspace{-0.6in}
    & \dH (\Dz_k', \Dz)  \\
    = & 2 \sup_{\gA \in \gA_\HH}   \left \vert \E_{z \sim \Dz_k'}\left[\rp(\gA(z) ) \right] - \E_{z \sim \Dz} \left[\rp(\gA(z) ) \right] \right \vert \\
     =  & 2 \sup_{\gA \in \gA_\HH}  \left \vert \E_{z \sim \frac{1}{2} (\Dz_k + \Dz_A)}\left[ \rp(\gA(z)) \right] - \E_{z \sim \Dz} \left[\rp(\gA(z) ) \right] \right \vert \\
     = & 2 \sup_{\gA \in \gA_\HH}  \left \vert \frac{1}{2}\E_{z \sim \Dz_K }\left[ \rp(\gA(z)) \right]  + \frac{1}{2}\E_{z \sim \Dz_A }\left[ \rp(\gA(z)) \right]  - \E_{z \sim \Dz} \left[\rp(\gA(z) ) \right] \right \vert \\
     \leq &  \sup_{\gA \in \gA_\HH}  \left \vert  \E_{z \sim \Dz_K }\left[ \rp(\gA(z)) \right]   - \E_{z \sim \Dz} \left[\rp(\gA(z) ) \right] \right \vert  +  \sup_{\gA \in \gA_\HH}  \left \vert  \E_{z \sim \Dz_A }\left[ \rp(\gA(z)) \right]  - \E_{z \sim \Dz} \left[\rp(\gA(z) ) \right] \right \vert \\  
     = & \frac{1}{2} \dH (\Dz_k,\Dz) + \frac{1}{2} \dH(\Dz_A, \Dz).
\end{align*}
It is clear that $\frac{1}{2} \dH(\Dz_A, \Dz)$, which is bounded by $\epsilon$, affects the dissimilarity between the \textit{induced} image of local and the global distribution,  therefore plays a key role in upper-bounding the global performance ($\gL_\gT(h)$ in \Eqref{eq:corollary}). 
Next, we discuss different scenarios when FL can benefit from such augmented data, and when the quality of augmented distribution $\gD_A$ can limit the generalization performance of the aggregated model. 

\textbf{$\gD_A$ can benefit local users when $\epsilon$ is small:}
To see this, one can assume that: $$\dH (\Dz_A, \Dz) = \epsilon \leq \min_{k} \dH (\Dz_k, \Dz),$$ of which the intuition is that, after feature mapping, the discrepancy between the augmented distribution and the global distribution   
is smaller than the discrepancy between an individual user and the global.
Based on this assumption, one can conclude that $\forall~ \gT_k \in \vgT$:
\begin{align*}
    \dH (\Dz_k', \Dz) = & \frac{1}{2} \dH (\Dz_k,\Dz) + \frac{1}{2} \dH(\Dz_A, \Dz) \\
    \leq & \frac{1}{2} \dH (\Dz_k,\Dz) + \min_{j} \dH (\Dz_j, \Dz) \\ 
    \leq  &  \dH (\Dz_k,\Dz), 
\end{align*}
Therefore, a small $\dH (\Dz_A, \Dz)$ benefits local users w.r.t their generalization performance, by both reducing the data discrepancy and enriching the empirical samples, in that: 
\begin{align*}
    \gL_\gT(h_k) \leq \gL_{\gT_k'}(h_k)  + \lambda_k' + 
    \underbrace{\leq \dH (\Dz_k',\Dz) }_{\leq \dH(\Dz_k, \Dz)}
    + \underbrace{ \sqrt{\frac{4}{m'} \left( d \log\frac{2 e m' }{d} + \log \frac{4}{\delta} \right) } }_{\leq \sqrt{\frac{4}{m} ( d \log\frac{2 e m }{d} + \log \frac{4}{\delta} ) } } \text{\hspace{0.3in} ( Derived from Lemma \ref{lemma:da-bounds}).}\
\end{align*}

\textbf{$\gD_A$ has positive effects on the generalization performance when $\epsilon$ is moderate:}
Instead, one might as well assume that $$\dH (\Dz_A, \Dz) = \epsilon \leq \frac{1}{K}\sum_{k=1}^K \dH (\Dz_k, \Dz),$$ which implies that, after feature mapping over $\gR$, the dissimilarity between $\gD_A$ and the global distribution $\gD$ is at least as small as the \textit{average} dissimilarity between local users and the global.
Based on this assumption, one can derive that:

$$
   \sum_{k} \dH (\Dz_k', \Dz) \leq \sum_k \dH(\Dz_k, \Dz),~ 
    \sqrt{\frac{4}{m'} \left( d \log\frac{2 e m' }{d} + \log \frac{4}{\delta} \right) } \leq  \sqrt{\frac{4}{m} \left( d \log\frac{2 e m}{d} + \log \frac{4}{\delta} \right) },
$$

which can still contribute to a tighter upper-bound for the global performance in \Eqref{eq:corollary}, compared with not using the augmented data.

Conversely, when $\epsilon$ is over-large, which implies that $\gD_A$ is not relevant to the original FL task, it may have negative impacts on the generalization performance.
\end{proof}

\section{Extended Experiments} \label{appendix:experiments}

We first discuss some practical considerations for implementing our algorithm: 
\begin{itemize}
    \item{\textbf{Weighting  user models:}} User models vary in their ability to predict certain labels over  others due to their statistical heterogeneity.
    Therefore, we use the number of training labels available to users to summarize a weight matrix
    $\mLambda=\{ \lambda_k^c | c \in \gY,  k \in \{ 1, 2, \cdots, K\} \}, $ $\st~\forall c,i,j,~\frac{\lambda^c_i}{\lambda^c_j} = \frac{n^c_i}{n^c_j}$ indicates the ratio of training samples for label $c$ between two users $i$ and $j$, and $\sum_{k} \lambda_k^c = 1~\forall c \in \gY$. 
    We then apply this weight matrix to adjust the generator objective as the following:
    \begin{small}
    \begin{align*}
        \min_{\vw} J(\vw):= \E_{y \sim \pe(y)} \E_{z \sim G_\vw(z|y)}\left[\lambda_{k}^{y}  l \left(\sigma \left(\frac{1}{K} \sum\nolimits_{k=1}^{K}  g(z;\vtheta^p_k ) \right), y \right) \right].
    \end{align*}
    \end{small}
    We found that this weighted objective can further mitigate the impact of negative ensemble, especially when a teacher model is too weak to predict certain labels due to lacking training samples of that category.
    \item{\textbf{Stochastic generative learning:}} Built upon prior arts on generative learning~\cite{kingma2013auto}, we use an auxiliary noise vector with dimension $d_n$ to infer the desirable feature representation for a given label $y$, \st~  $z \sim G_{\vw}(\cdot|y) \equiv  G_{\vw}(y, \epsilon | \epsilon \sim \gN(0, I))$. 
    To further increase the diversity of the generator output, we also leverage the idea of \textit{diversity loss} from prior work~\cite{mao2019mode} to train the generator model.
\end{itemize}

\subsection{Prototype Results}\label{sec:prototype-detail}
\begin{wraptable}{r}{6cm} 
    \begin{center}
    \vspace{-0.2in}
    \begin{tabular}{lcccc} 
        \toprule 
            & User 1 & User 2 & User 3 & Oracle  \\
    Before  &    97.1 &    81.3    &    81.2    &  \multirow{2}{*}{98.4}  \\
    After  &   98.6     &   98.3     &  98.2      &          
   \\
   \bottomrule                  
\end{tabular} 
    \vspace{-0.1in}
    \hspace{0.3in} \caption{Accuracy (\%) before and after KD. \label{table:prototype}} 
    \end{center}
    \end{wraptable} 
    We adopt an one-round FL setting for the prototype experiment, for which the dataset distributions of local users, as well as their model decision boundaries \textit{before} and \textit{after} knowledge distillation, are illustrated in Figure \ref{appendix-fig:data-distribution-verbose}.
    Accuracy of user models on the global dataset is also summarized in Table \ref{table:prototype}, from which one can observe that the generalization performance of user models have been notably improved by the distilled knowledge. 
\vspace{-0.1in}
\begin{figure*}[hbt!]  
    \begin{center}
        \hspace{-0.1in}
        \begin{subfigure}[b]{0.8\textwidth}
            \centerline{\includegraphics[width=\columnwidth]{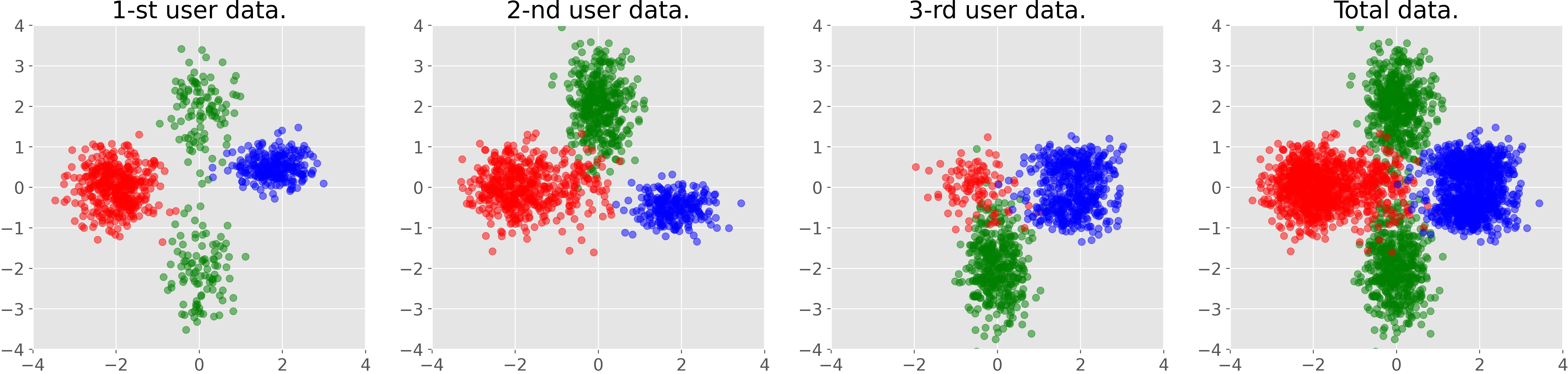}}
        \subcaption{Local user data distribution and total data distribution.}
        \end{subfigure} 
 
        \begin{subfigure}[b]{0.8\textwidth}
            \centerline{\includegraphics[width=\columnwidth]{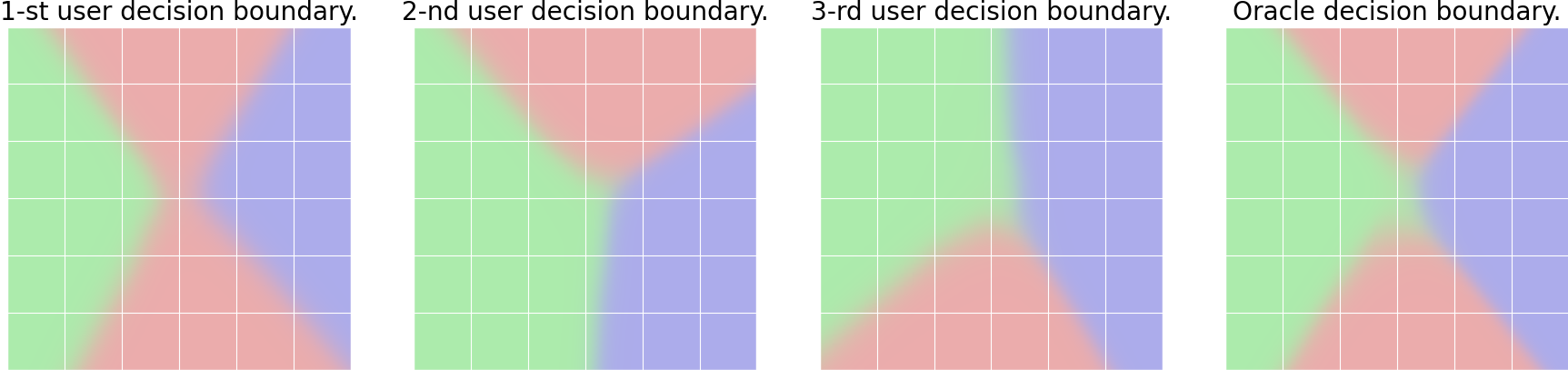}}
        \subcaption{User models generate biased decision boundaries \textbf{before} KD, provided with incomplete local data.}
        \end{subfigure} 
        
        \begin{subfigure}[b]{0.8\textwidth}
            \centerline{\includegraphics[width=\columnwidth]{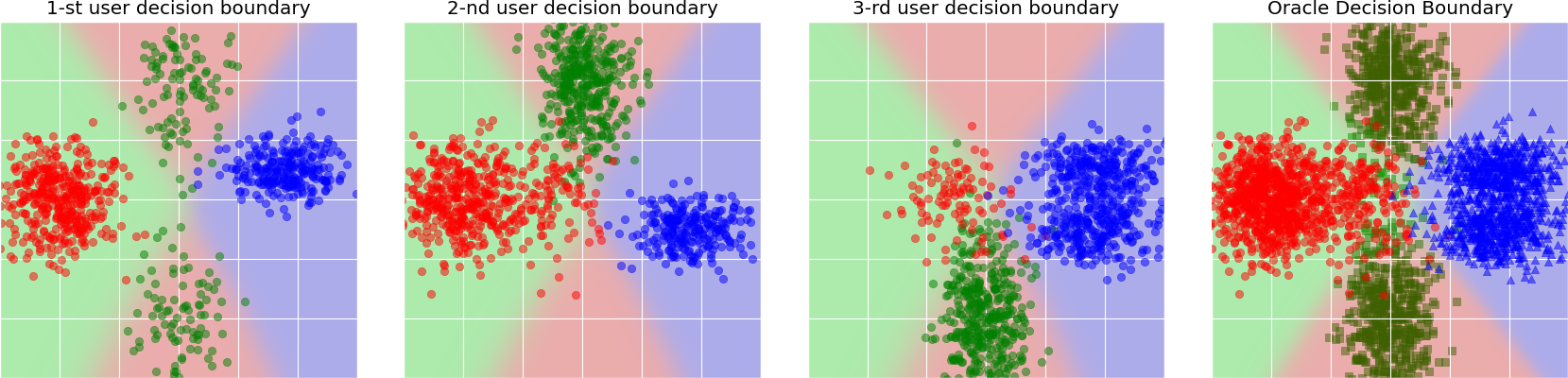}}
        \subcaption{Decision Boundaries of user models are improved \textbf{after} KD.}
        \end{subfigure}  
    \end{center}
    \vspace{-0.1in}
    \caption{Knowledge distillation process for the prototype experiment.}\label{appendix-fig:data-distribution-verbose}
\end{figure*}

\subsection{Experimental Setup} \label{appendix:experiment-setup}
\vspace{-0.1in}
We provide the network architecture for the generator and the classifier in  Table \ref{appendix-table:generator} and Table \ref{appendix-table:network}.
For the generator $G_\vw$, we adopt a two-MLP layer network. It takes a noise vector $\epsilon$ and an one-hot label vector $y$ as the input, which, after a hidden layer with dimension $d_h$, outputs a  feature representation with dimension $d$.
For the classifier, we adopt a network architecture with a CNN module followed by a MLP module.
Hyperparameter settings for the experiments are provided in Table \ref{appendix-table:run}.
%
\begin{table}[htb!]
    \begin{center}
        \hspace{-0.4in}
        \begin{minipage}{.4\linewidth}
    \begin{tabular}{ccc} 
   Dataset & Hyperparameter & Value \\ 
    \toprule 
    \multirow{1}{*}{\Cdata}   & $d_n, d_h, d$  & 32, 128, 32 \\
        \midrule
    \multirow{1}{*}{  \Mdata \& \Edata  }   & $d_n, d_h, d$  & 32, 256, 32  \\
   \bottomrule                  
\end{tabular}
    \caption{Network architecture for the generator $G_\vw$.\label{appendix-table:generator}} 
\end{minipage} 
\hspace{0.5in}
    \begin{minipage}{.45\linewidth}
\begin{tabular}{ccc} 
    Dataset & Hyperparameter & Value \\ 
     \toprule 
     \multirow{2}{*}{\Cdata}   & CNN Module  & [16, M, 32, M, 64] \\
         & MLP Module  &   [784, 32]   \\
         \midrule 
     \multirow{2}{*}{  \begin{tabular}[c]{@{}c@{}}\Mdata \\ \& \Edata \end{tabular}  }   & CNN Module  & [6, 16] \\
         & MLP Module  &   [784, 32]   \\
    \bottomrule                  
 \end{tabular}
    \caption{Network architecture for the classification model.  \label{appendix-table:network}} 
\end{minipage} 
\end{center}
\end{table}

\begin{table}[htb]
    \begin{center}
       \scalebox{1}{
    \begin{tabular}{clc} 
    & Hyperparameter & Value \\ 
    \toprule 
    \multirow{7}{*}{  Shared Parameters }   & Learning rate  & 0.01 \\
        & Optimizer  &  sgd   \\
        & Local update steps ($T$)  &  20   \\
        & Batch size ($B$) & 32   \\
        & Communication rounds &  200   \\
        & \# of total users  &  20   \\
        & \# of active users  &  10   \\
        \midrule
    \multirow{2}{*}{} \Fusion~  & Ensemble Optimizer  &  adam   \\
        & Generator learning rate  & $10^{-4}$ \\
        & Ensemble batch size  &  128   \\
        \midrule
    \multirow{2}{*}{} \approach~  & Generator Optimizer  &  adam   \\
        & Generator learning rate  & $10^{-4}$ \\
        & Generator inference size  &  128   \\
        & User distillation batch size  &  32   \\
        \midrule 
    \multirow{1}{*}{} \FD \& \FDFL~  & Distillation coefficient &  0.1   \\
        \midrule 
    \multirow{1}{*}{} \Prox~  & Proximal coefficient &  0.1   \\
   \bottomrule                  
\end{tabular}
       }
\caption{We use the above configurations for experiments unless mentioned otherwise.\label{appendix-table:run}} 
\end{center}
\end{table}

 \subsection{\approach~with Partial Parameter Sharing}\label{sec:algorithm-partial}
 %
    Algorithm \ref{alg:fedGen-partial} summarizes an variant approach of \approach\ for a specific FL setting, where only the last prediction layer is shared among users while keeping the feature extraction layers localized.
 %
 
 \begin{small}
     \begin{algorithm}[htb!]
        \caption{\approach\ with Partial Parameter Sharing}
        \label{alg:fedGen-partial}
     \begin{algorithmic}[1]
        \STATE {\bfseries Require:} Tasks $\gT_k, k \in \{1,\cdots, K\}$;\\
                 \hspace{0.1in} Global predictor $\vtheta^p$, local parameters $\{\vtheta_k=[\vtheta_k^f; \vtheta_k^p]\}_{k=1}^K$;\\
                 \hspace{0.1in} Generator parameter $\vw$; $\pe(y)$ uniformly initialized;\\ 
                 \hspace{0.1in} Learning rate $\alpha$, $\beta$, local steps $T$, batch size $B$, local label counter $c_k$. 
        \REPEAT 
         \STATE{Server selects active users $\gA$ uniformly at random, then broadcast $\vw, \vtheta^p$, $\pe(y)$ to  $\gA$. }
         
        \FOR{all user $k \in \gA$ in parallel}
            \STATE{$\vtheta_k^p \leftarrow \vtheta^p$,}
            \FOR{$t=1$, $\dots, T$}  
                 \STATE{\begin{small} $\{ x_i, y_i \}_{i=1}^B \sim \gT_k$, $\{ \ze_i \sim G_\vw(\cdot|\ye_i),\ye_i \sim \pe(y) \}_{i=1}^B $.\end{small}}
                \STATE{Update label counter $c_k$.}%
                \STATE{\hspace{-0.2in} $\vtheta_k \leftarrow \vtheta_k - \beta \nabla_{\vtheta_k} J(\vtheta_k).$ \hspace{0.2in}  }
                 
            \ENDFOR
             
            \STATE{User sends $\vtheta_k^p$, $c_k$ back to server.}
         \ENDFOR
          \STATE{Server updates $\vtheta^p \leftarrow \frac{1}{|\gA|}\sum_{k \in \gA} \vtheta_k^p$, and $\pe(y)$ based on $\{c_k\}_{k\in\gA}$.} 
          \STATE{$\vw \leftarrow \vw - \alpha \nabla_{\vw} J(\vw).$ }  
        \UNTIL{training~stop}
     \end{algorithmic}
     \end{algorithm}
     \end{small}
 
 

 \subsection{Extended Experimental Results} \label{appendix:exp-learning-curves}
 We elaborate the learning curves trained on the \Mdata, \Cdata, and \Edata\ dataset in Figure \ref{appendix-fig:mnist-curves}, Figure \ref{appendix-fig:celeb-curves}, and Figure \ref{appendix-fig:emnist-curves}, respectively, with their performance  summarized in Table \ref{appendix:performance-overview}.

 \begin{figure*}[hbt!]  
     \begin{center}
         \hspace{-0.1in}
         \begin{subfigure}[b]{0.2\textwidth}
             \centerline{\includegraphics[width=\columnwidth]{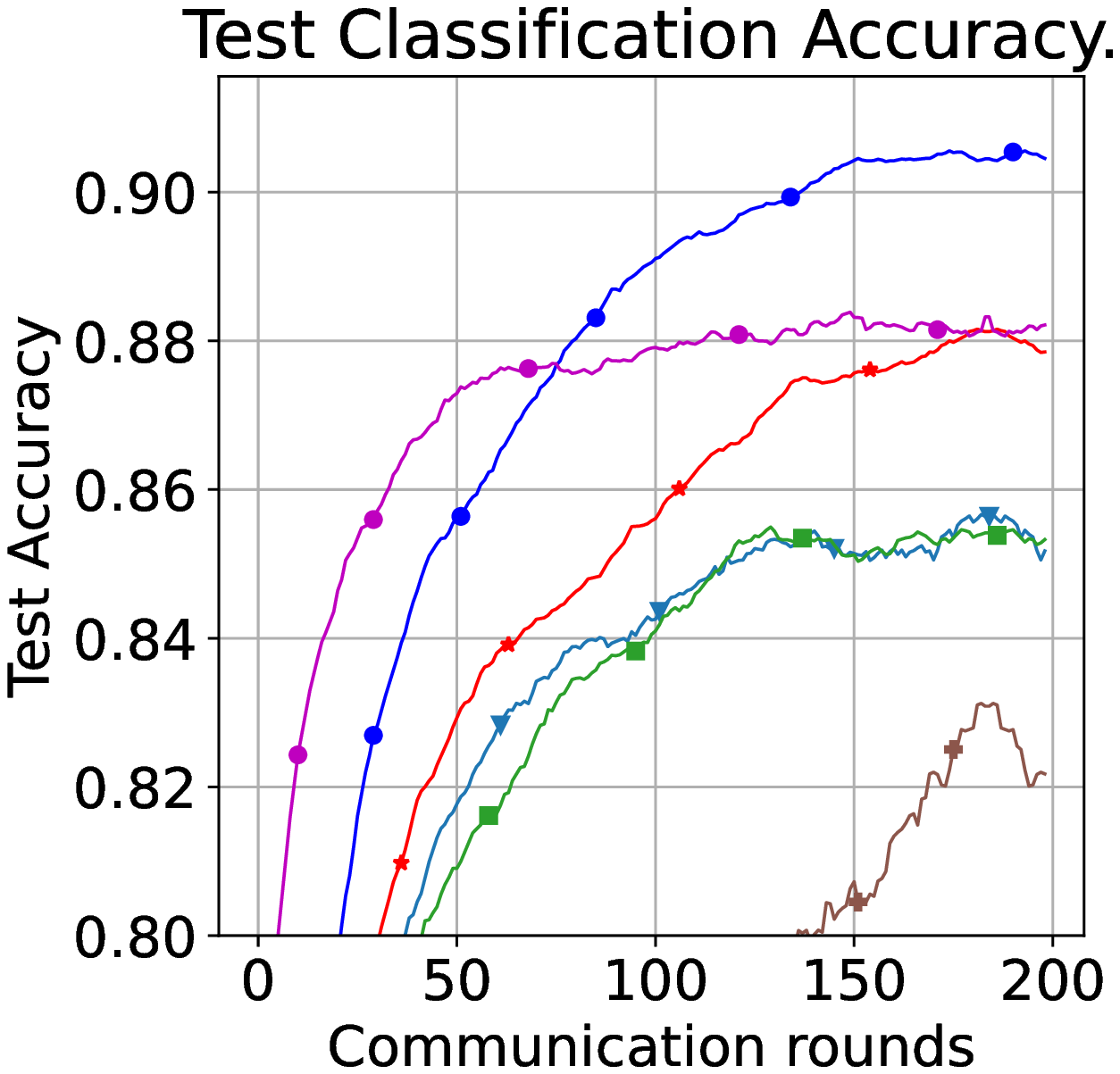}}
         \subcaption{$\alpha=0.05$}
         \end{subfigure} 
         \begin{subfigure}[b]{0.29\textwidth}
             \centerline{\includegraphics[width=\columnwidth]{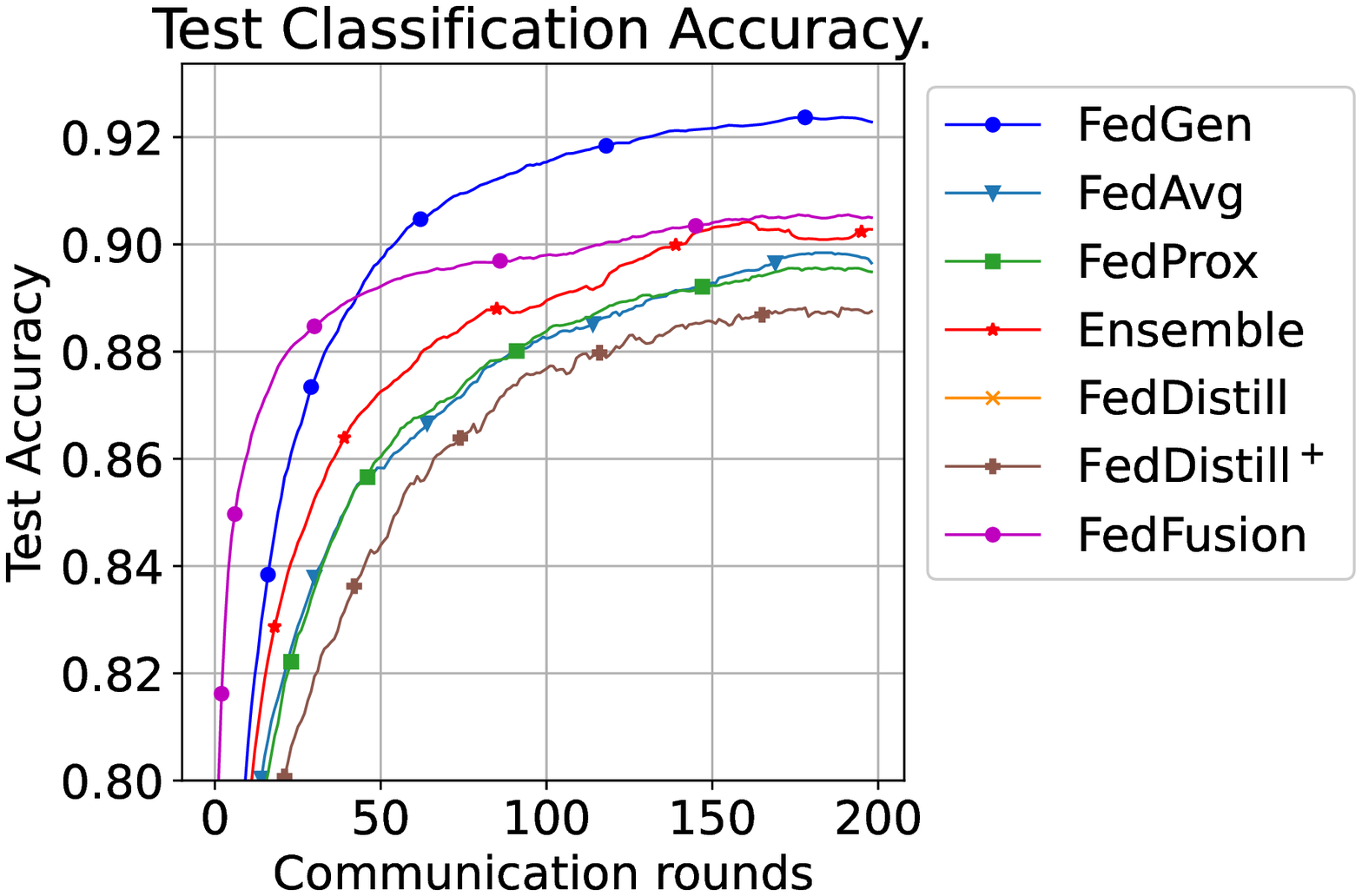}}
         \subcaption{$\alpha=0.1$}
         \end{subfigure} 
         \begin{subfigure}[b]{0.2\textwidth}
             \centerline{\includegraphics[width=\columnwidth]{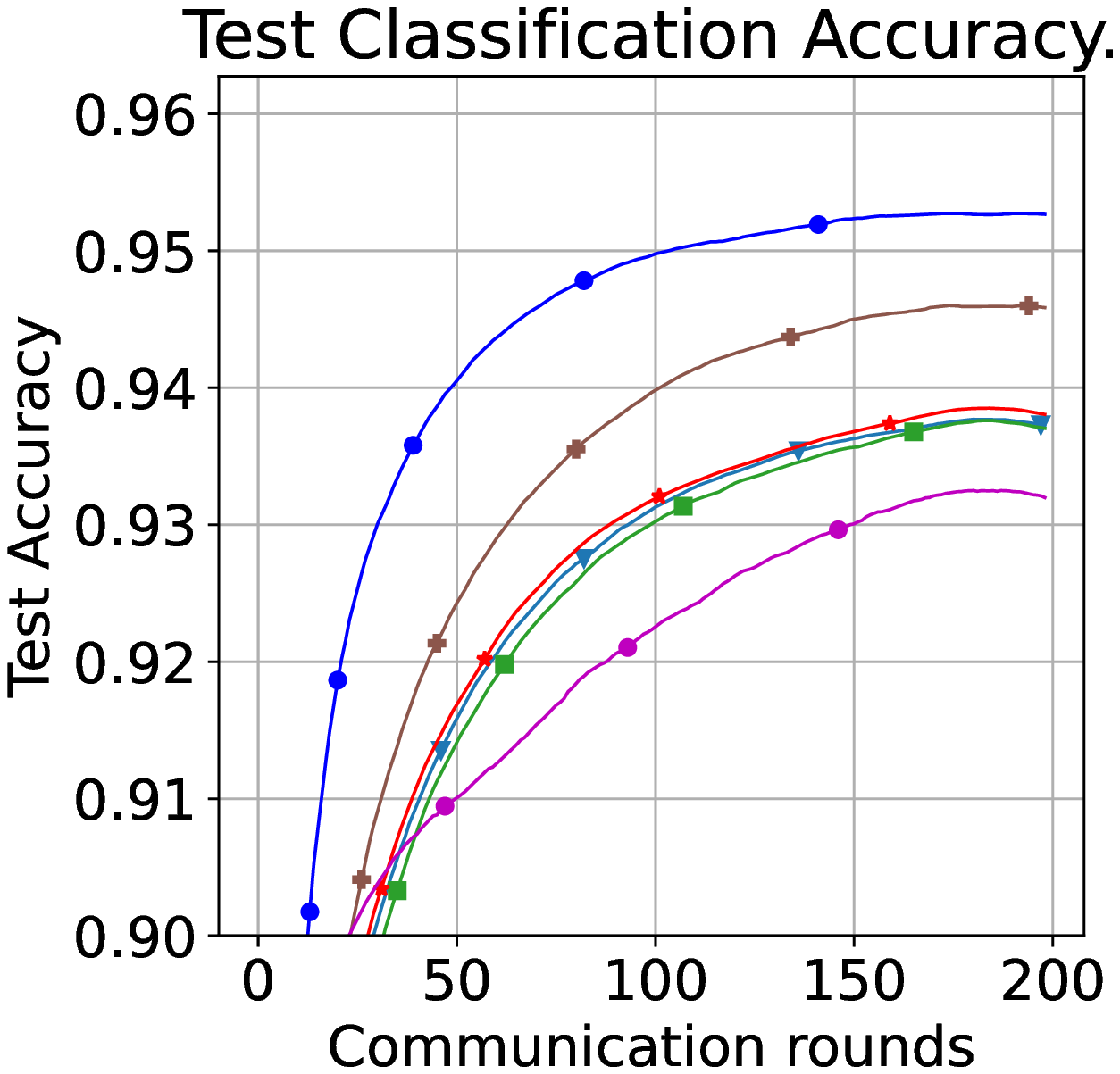}}
         \subcaption{$\alpha=1$}
         \end{subfigure} 
         \begin{subfigure}[b]{0.2\textwidth}
             \centerline{\includegraphics[width=\columnwidth]{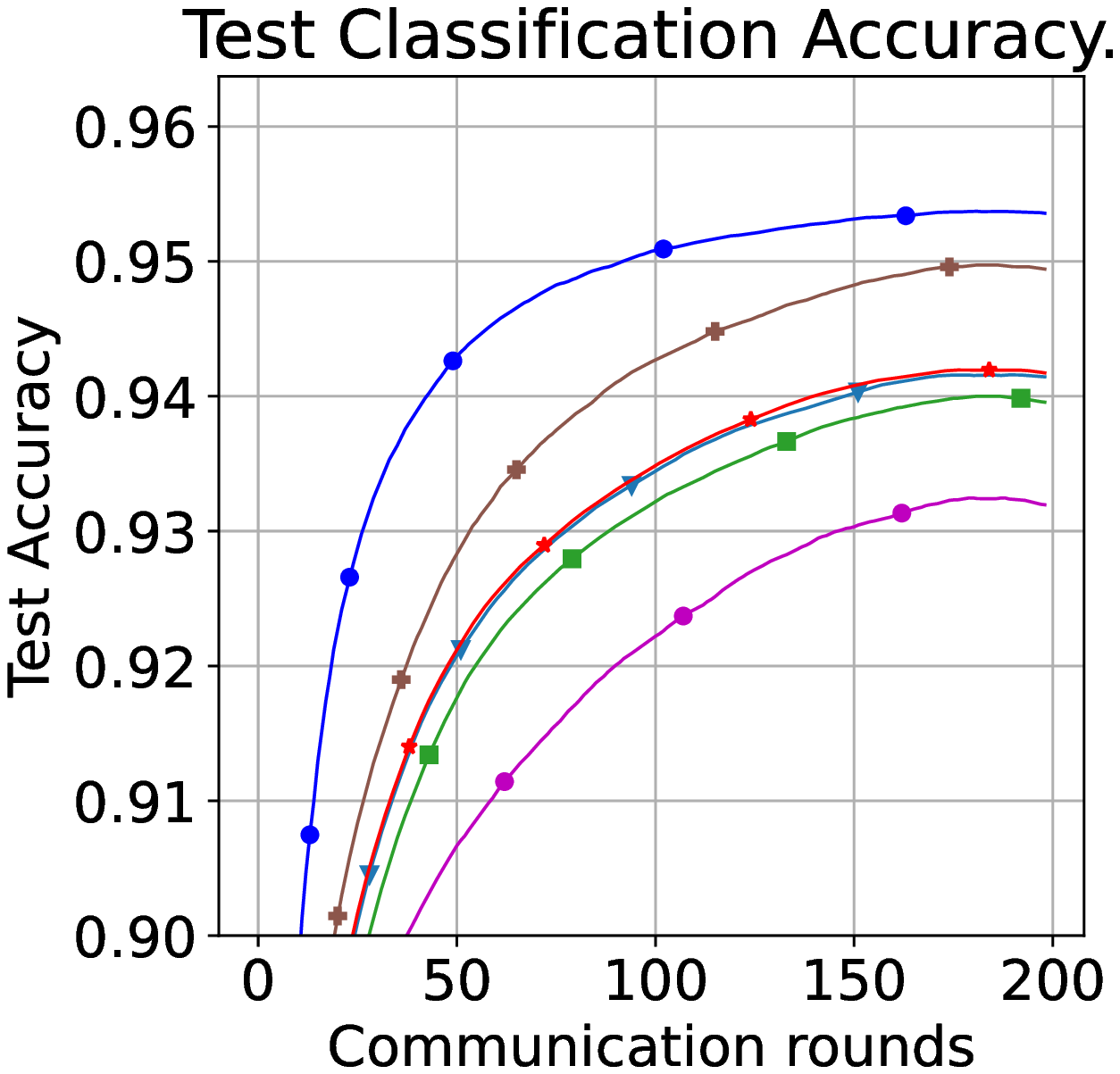}}
         \subcaption{$\alpha=10$}
         \end{subfigure} 
     \end{center}
     \vspace{-0.2in}
     \caption{Performance curves on \Mdata~ dataset, where a smaller $\alpha$ denotes larger data heterogeneity.}\label{appendix-fig:mnist-curves}
 \end{figure*}

 \begin{figure*}[hbt!]  
     \begin{center}
         \hspace{-0.1in} 
         \begin{subfigure}[b]{0.2\textwidth}
             \centerline{\includegraphics[width=\columnwidth]{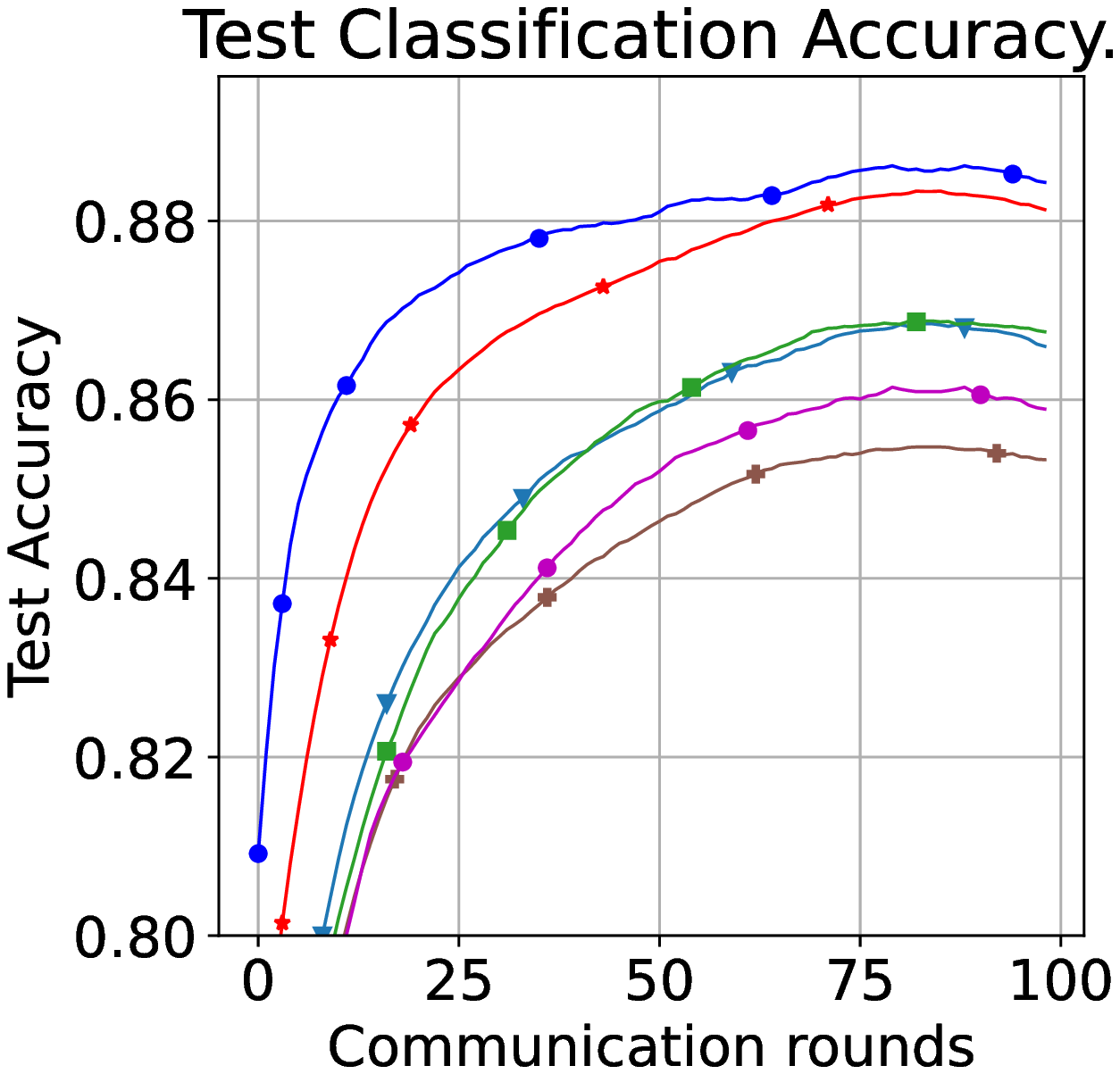}}
         \subcaption{\Cdata, $r=9/10$.}
         \end{subfigure} 
         \hspace{-0.1in}
         \begin{subfigure}[b]{0.29\textwidth}
             \centerline{\includegraphics[width=\columnwidth]{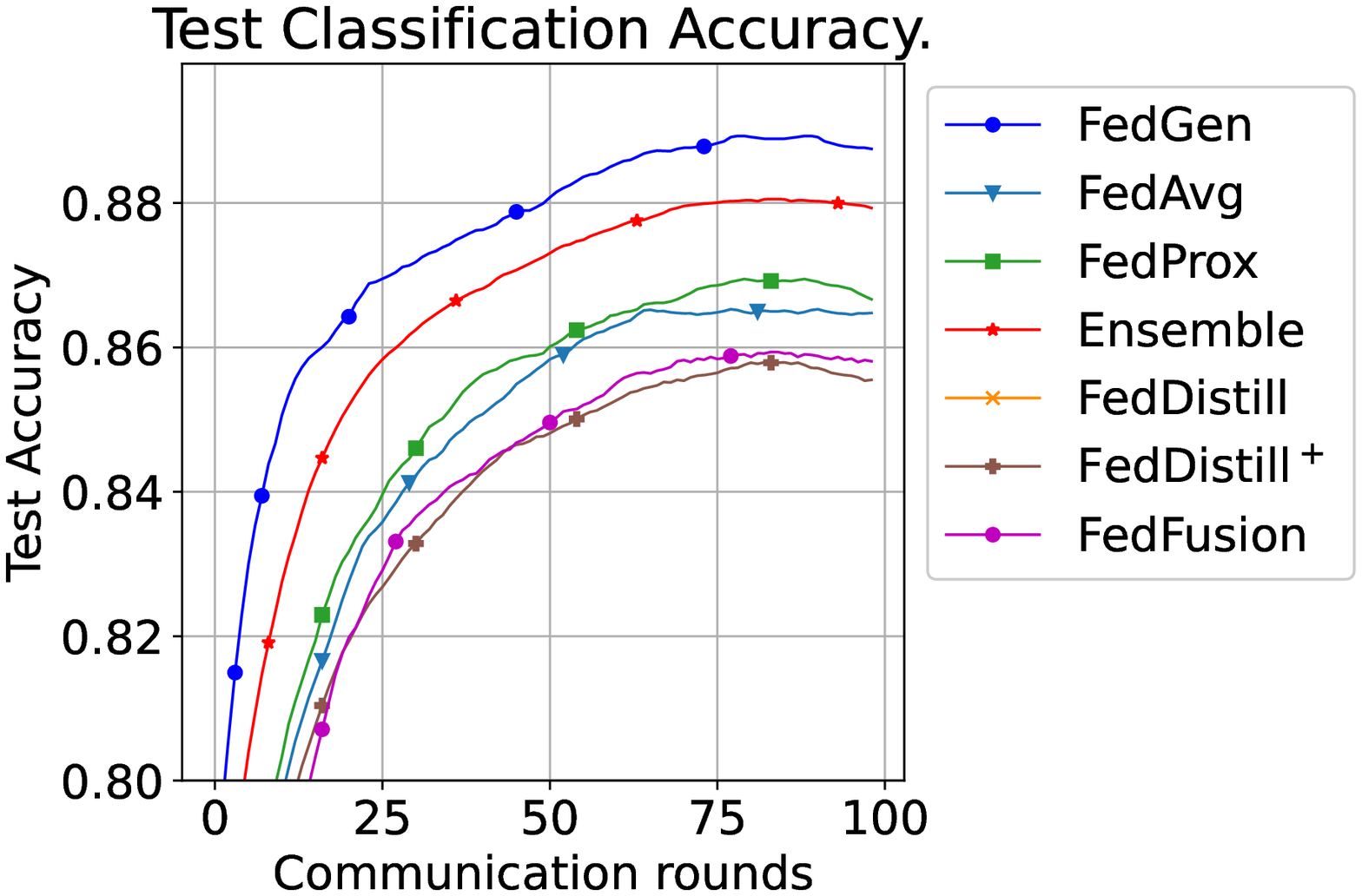}}
         \subcaption{\Cdata, $r=5/10$.}
         \end{subfigure} 
         \begin{subfigure}[b]{0.2\textwidth}
             \centerline{\includegraphics[width=\columnwidth]{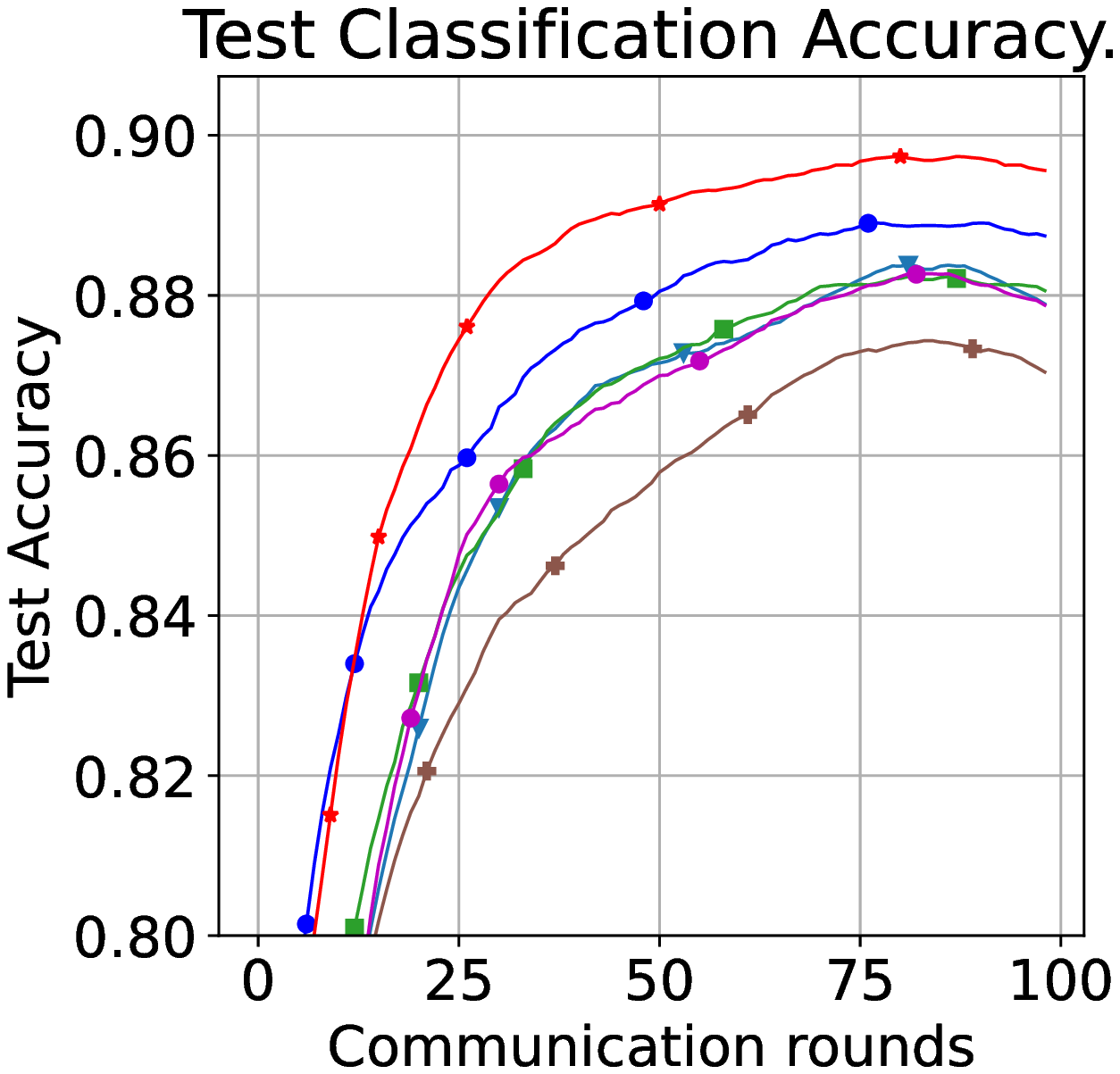}}
         \subcaption{\Cdata, $r=5/25$.}
         \end{subfigure} 
         \begin{subfigure}[b]{0.2\textwidth}
             \centerline{\includegraphics[width=\columnwidth]{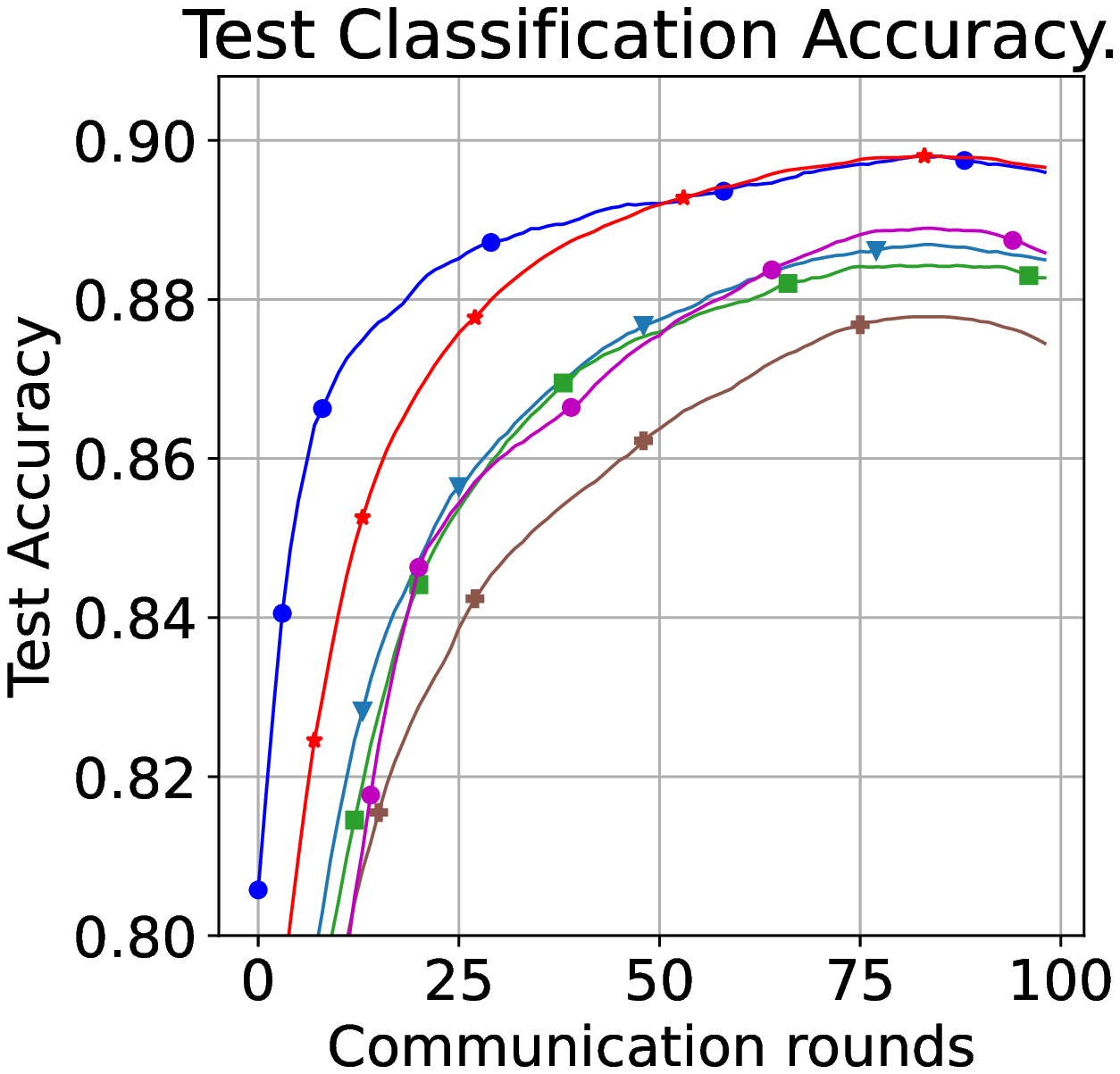}}
         \subcaption{\Cdata, $r=10/25$.}
         \end{subfigure} 
     \end{center}
     \vspace{-0.1in}
     \caption{Performance curves on \Cdata~ dataset.}\label{appendix-fig:celeb-curves}
 \end{figure*}

 \begin{figure*}[hbt!]  
         \begin{center}
             \hspace{-0.1in}
             \begin{subfigure}[b]{0.2\textwidth}
                 \centerline{\includegraphics[width=\columnwidth]{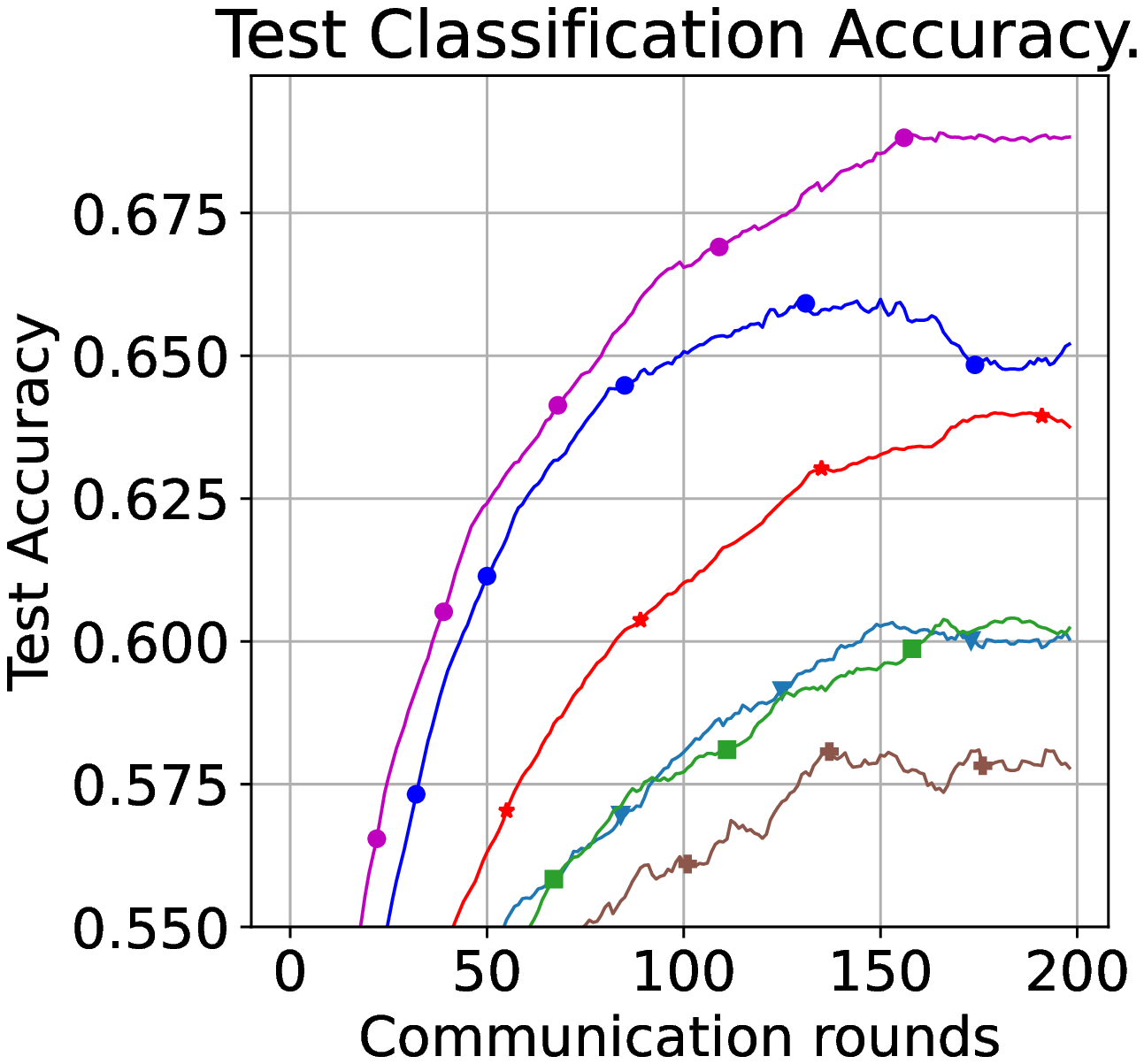}}
             \vspace{-0.1in}
             \subcaption{$\alpha=0.05$, $T=20$}
             \end{subfigure} 
             \begin{subfigure}[b]{0.29\textwidth}
                 \centerline{\includegraphics[width=\columnwidth]{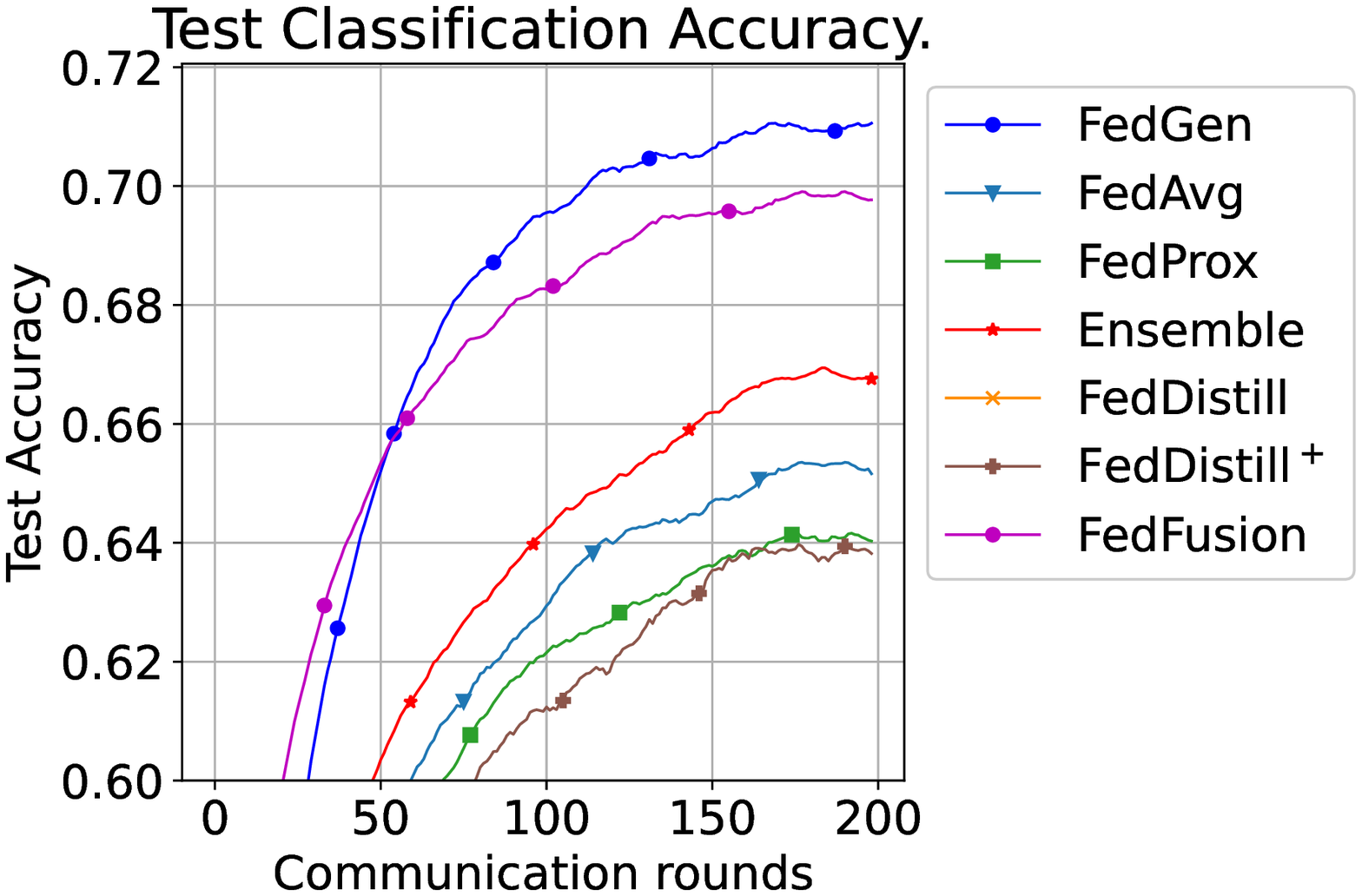}}
             \vspace{-0.1in}
             \subcaption{$\alpha=0.1$, $T=20$}
             \end{subfigure} 
             \begin{subfigure}[b]{0.2\textwidth}
                 \centerline{\includegraphics[width=\columnwidth]{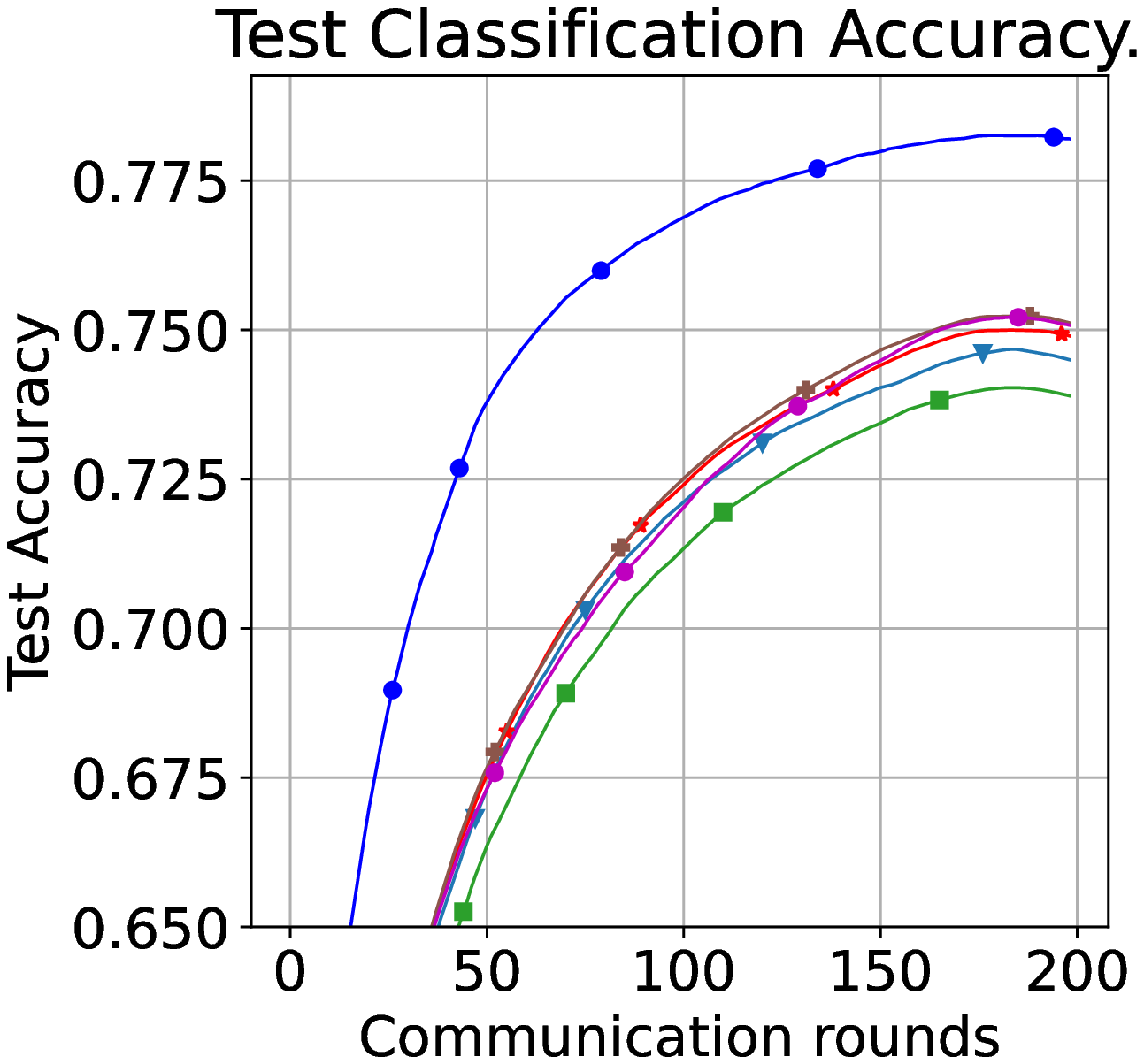}}
             \vspace{-0.1in}
             \subcaption{$\alpha=1$, $T=20$}
             \end{subfigure} 
             \begin{subfigure}[b]{0.2\textwidth}
                 \centerline{\includegraphics[width=\columnwidth]{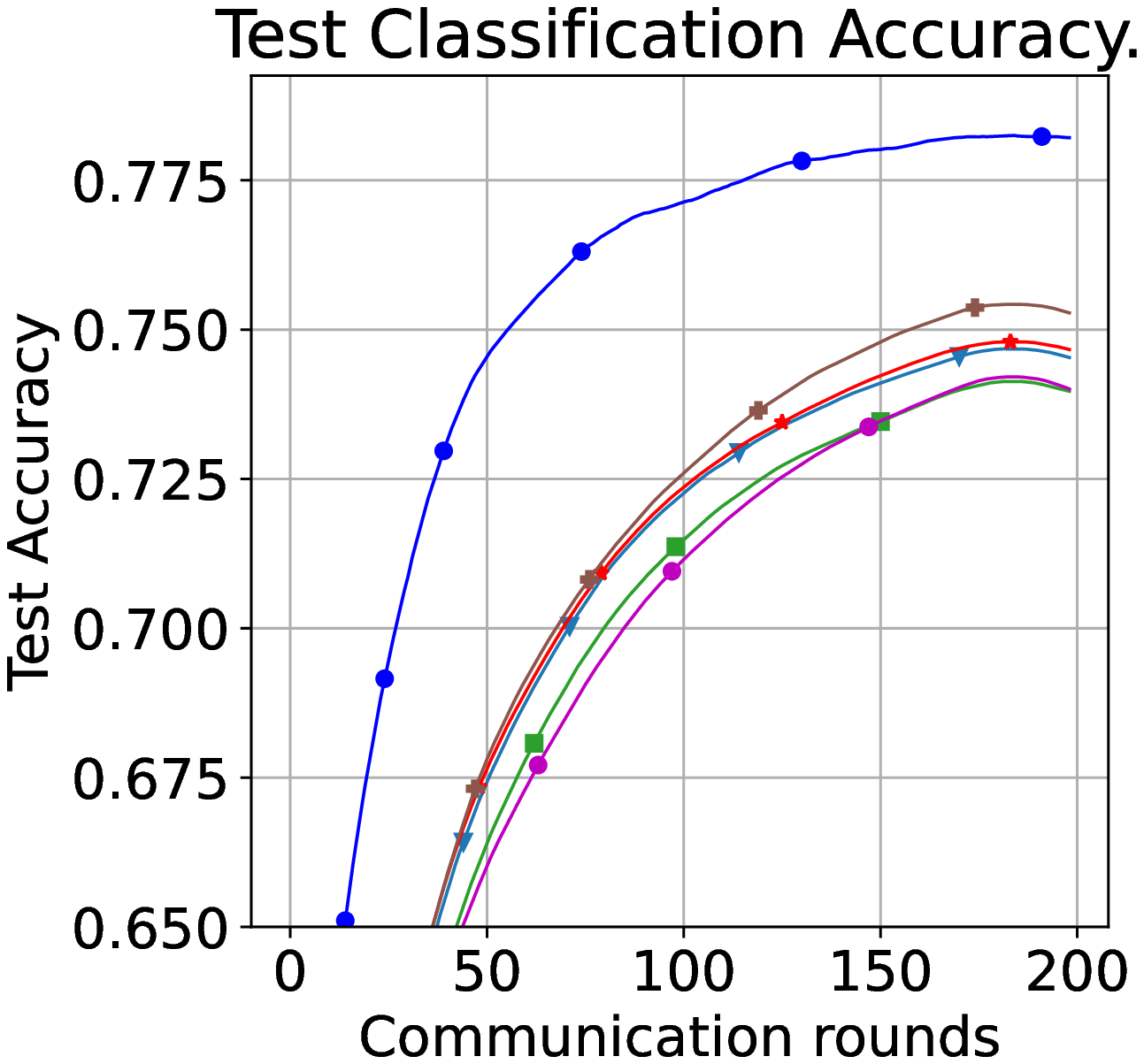}}
             \vspace{-0.1in}
             \subcaption{$\alpha=10$, $T=20$}
             \end{subfigure} 

             \hspace{-0.1in}
             \begin{subfigure}[b]{0.2\textwidth}
                 \centerline{\includegraphics[width=\columnwidth]{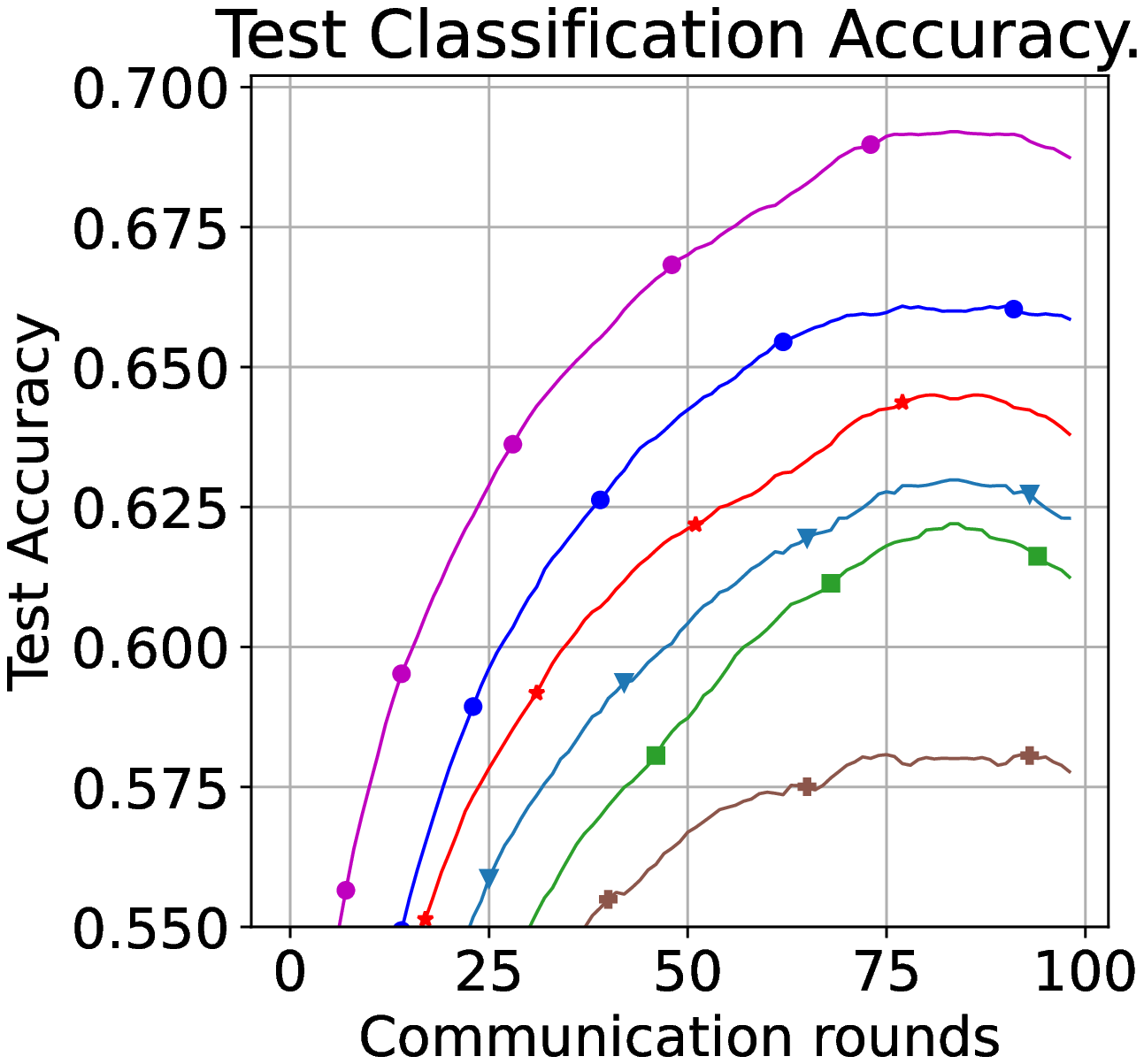}}
             \vspace{-0.1in}
             \subcaption{$\alpha=0.05$, $T=40$}
             \end{subfigure} 
             \begin{subfigure}[b]{0.29\textwidth}
                 \centerline{\includegraphics[width=\columnwidth]{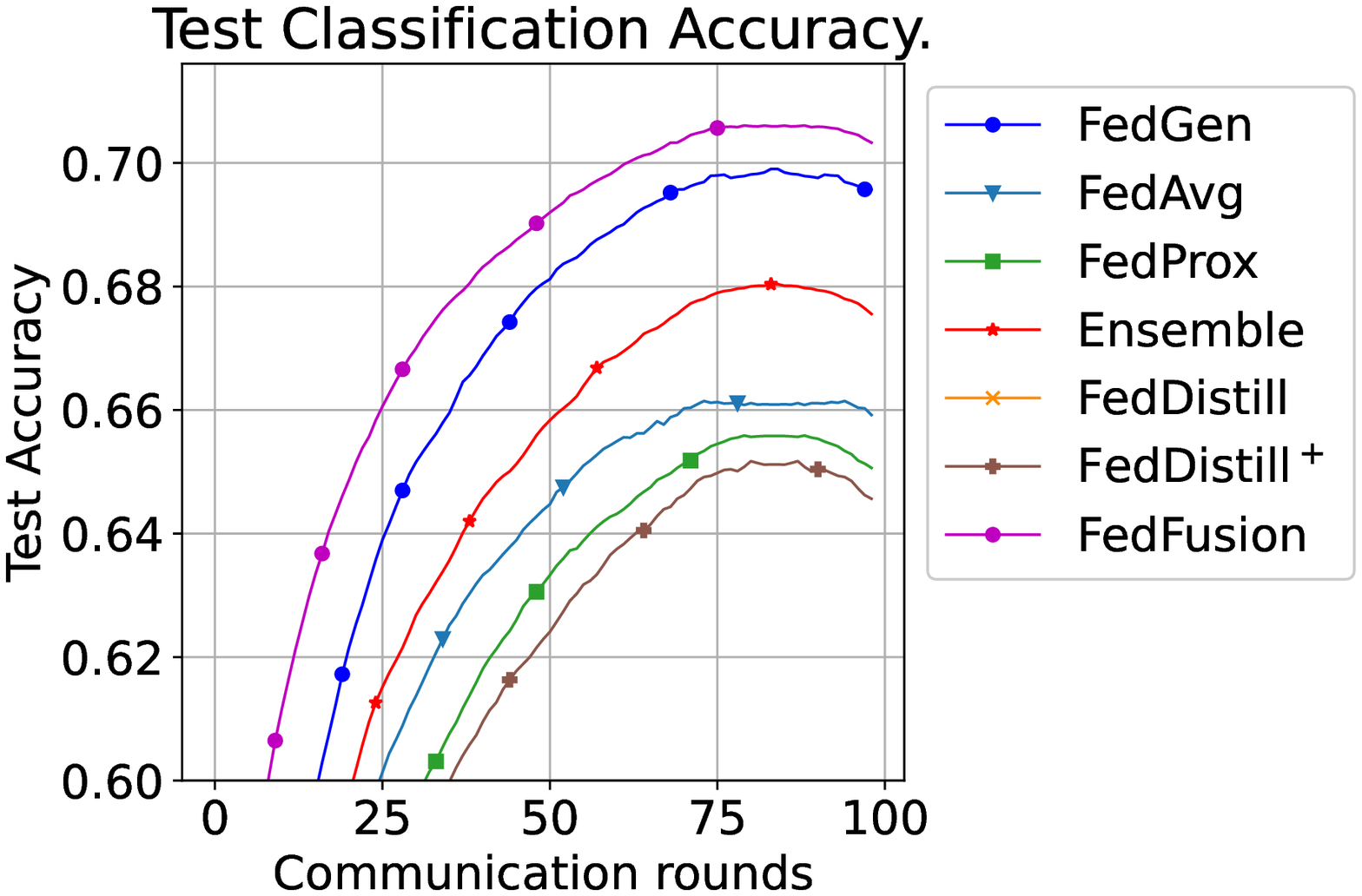}}
             \vspace{-0.1in}
             \subcaption{$\alpha=0.1$, $T=40$}
             \end{subfigure} 
             \begin{subfigure}[b]{0.2\textwidth}
                 \centerline{\includegraphics[width=\columnwidth]{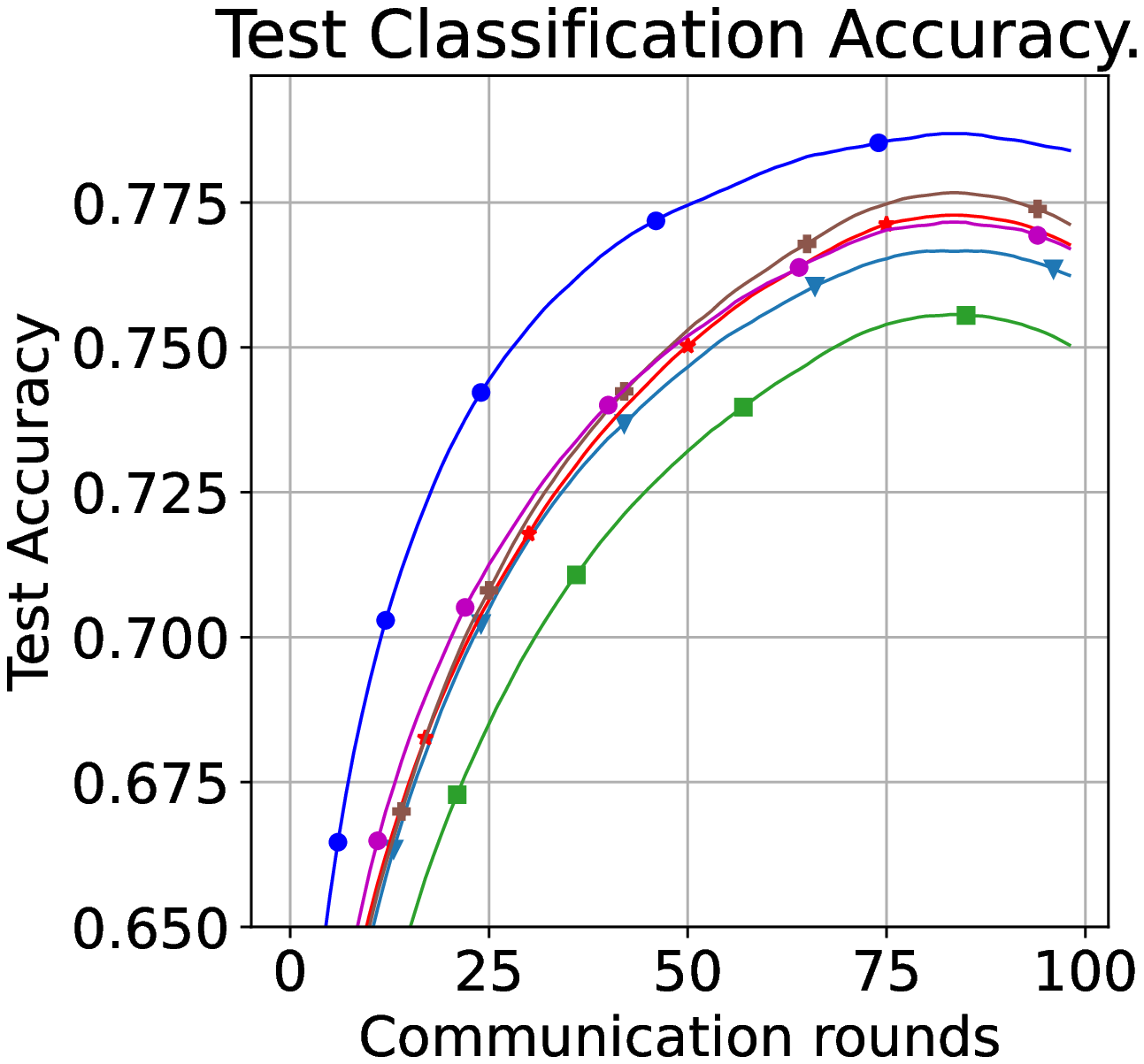}}
             \vspace{-0.1in}
             \subcaption{$\alpha=1$, $T=40$}
             \end{subfigure} 
             \begin{subfigure}[b]{0.2\textwidth}
                 \centerline{\includegraphics[width=\columnwidth]{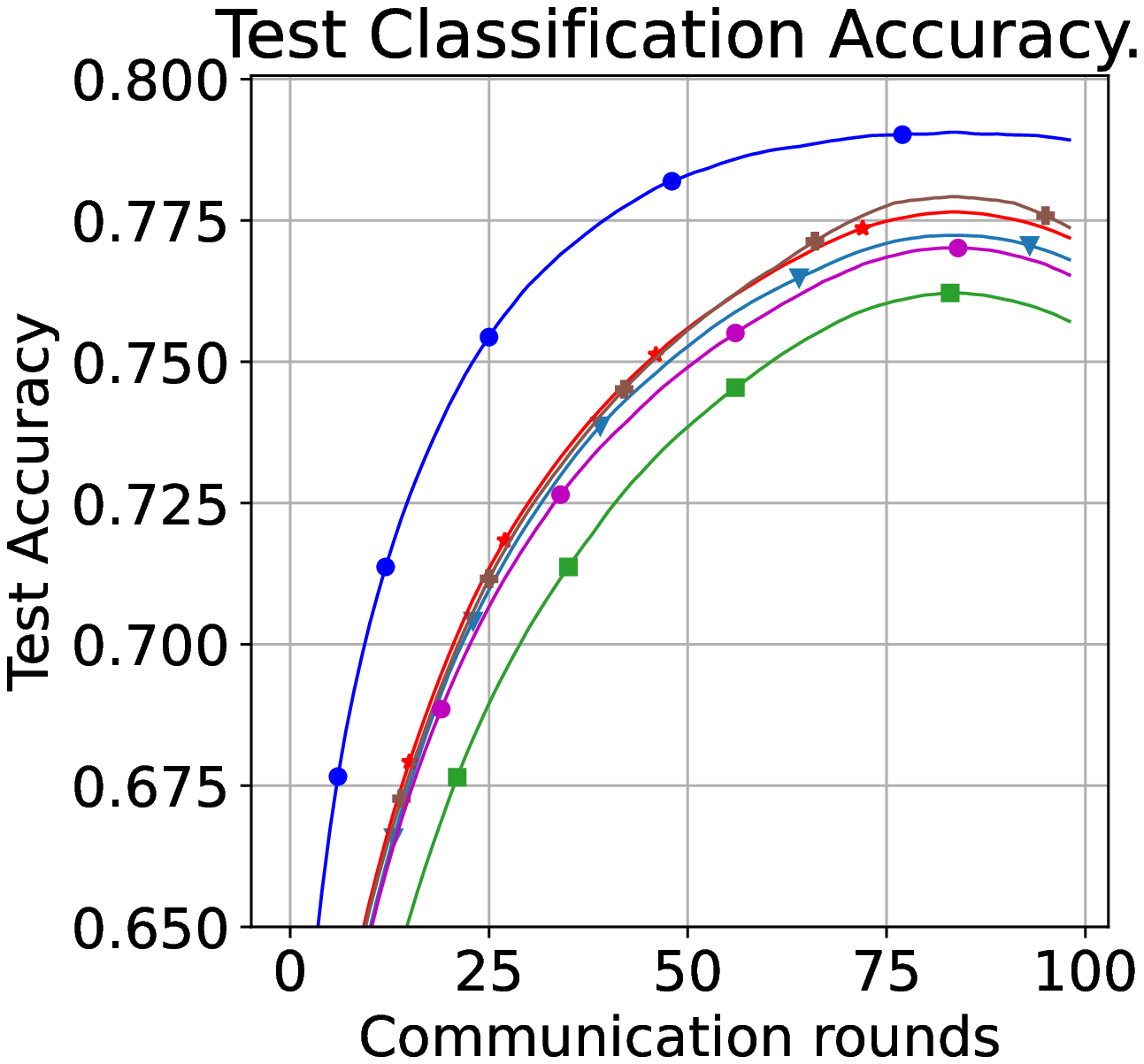}}
             \vspace{-0.1in}
             \subcaption{$\alpha=10$, $T=40$}
             \end{subfigure}  
             %
         \end{center}
         \vspace{-0.2in}
         \caption{Performance curves on \Edata~ dataset, under different data heterogeneity and communication frequencies.}\label{appendix-fig:emnist-curves}
 \end{figure*} 
 
 \begin{small}
     \begin{table*}[htb!]
         \begin{center}
             \scalebox{0.9}{
                \setlength\tabcolsep{2.5pt}
                 \begin{tabular}{llcccccccc}
                     \toprule 
     \multicolumn{1}{l}{}  &  & \multicolumn{5}{c}{\textbf{Top-1 Test Accuracy.}}  \\ \hline 
     Dataset  
     &  
     Setting 
     & \multicolumn{1}{c}{\Avg} 
     & \multicolumn{1}{c}{\Prox} 
     & \multicolumn{1}{c}{\Ensemble}
     & \multicolumn{1}{c}{\textsc{\FD}} 
     & \multicolumn{1}{c}{\textsc{\FDFL}} 
     & \multicolumn{1}{c}{\Fusion} 
     & \multicolumn{1}{c}{\approach} \\ \hline
     \multirow{4}{*}{\Mdata } 
                 &$\alpha$ = 0.05 & 87.70$\pm$2.07 & 87.49$\pm$2.05 & 88.85$\pm$0.68  & 70.56$\pm$1.24 & 86.70$\pm$2.27  & 90.02$\pm$0.96  & \textbf{91.30$\pm$0.74}\\  
                 &$\alpha$ = 0.1 & 90.16$\pm$0.59  & 90.10$\pm$0.39 & 90.78$\pm$0.39 & 64.11$\pm$ 1.36 & 90.28$\pm$0.89  & 91.11$\pm$0.43   & \textbf{93.03$\pm$0.32}  \\  
                 &$\alpha$ = 1 & 93.84$\pm$0.25  & 93.83 $\pm$ 0.29 & 93.91$\pm$0.28  & 79.88$\pm$0.66  &  94.73$\pm$0.15 & 93.37$\pm$0.40  & \textbf{95.52$\pm$0.07}  \\  
                 &$\alpha$ = 10 & 94.23$\pm$0.13  & 94.06$\pm$0.10 & 94.25$\pm$0.11 & 89.21$\pm$0.26 & 95.04$\pm$0.21 & 93.36$\pm$0.45 & \textbf{95.79$\pm$0.10} \\  
                 \midrule
     \multirow{3}{*}{\Cdata } 
     &$r = 5/10$ & 87.48$\pm$0.39  & 87.67$\pm$0.39 & 88.48$\pm$0.23 & 76.68$\pm$1.23  & 86.37$\pm$0.41  & 87.01$\pm$1.00  & \textbf{89.70$\pm$0.32}  \\  
     &$r = 5/25$ & 89.13$\pm$0.25  & 88.84$\pm$0.19 & \textbf{90.22$\pm$0.31} & 74.99$\pm$1.57  & 88.05$\pm$ 0.43 &   88.93$\pm$0.79  & 89.62$\pm$0.34  \\  
     &$r = 10/25$ & 89.12$\pm$0.20 & 89.01$\pm$0.33 & 90.08$\pm$0.24 & 75.88$\pm$1.17  & 88.14$\pm$0.37 & 89.25$\pm$0.56   & \textbf{90.29$\pm$0.47} \\  
      \midrule
     \multirow{4}{*} {\begin{tabular}[c]{@{}c@{}}\Edata, \\ $T$=20\end{tabular}} 
                 &$\alpha$ = 0.05 & 62.25$\pm$2.82  & $61.93\pm$2.31 & 64.99$\pm$0.35 & 60.49$\pm$1.27  & 61.56$\pm$2.15 & \textbf{70.40$\pm$0.79}   & {68.53$\pm$1.17}  \\  
                 &$\alpha$ = 0.1 & 66.21$\pm$2.43  & 65.29$\pm$2.94  & 67.53$\pm$1.19 & 50.32$\pm$1.39 & 66.06$\pm$3.18  &  70.94$\pm$0.76 & \textbf{72.15$\pm$0.21}  \\  
                 &$\alpha$ = 1 & 74.83$\pm$0.99  & 74.12$\pm$0.88  & 75.12$\pm$1.07  & 46.19$\pm$0.70  & 75.41$\pm$1.05  & 75.43$\pm$0.37  & \textbf{78.48$\pm$1.04} \\  
                 &$\alpha$ = 10 & 74.83$\pm$ 0.69 & 74.24$\pm$0.81  & 74.90$\pm$0.80 & 54.77$\pm$0.33  & 75.55 $\pm$0.94  & {74.36$\pm$0.40} & \textbf{78.43$\pm$0.74}  \\  
                 \midrule
     \multirow{4}{*} {\begin{tabular}[c]{@{}c@{}}\Edata, \\ $T$=40\end{tabular}} 
                 &$\alpha$ = 0.05  & 64.51$\pm$1.13 & 63.60$\pm$0.69  & 65.74$\pm$0.45  & 60.73$\pm$1.62   & 60.73$\pm$1.06  & \textbf{70.46$\pm$1.16}   & {67.64  $\pm$0.75}   \\  
                 &$\alpha$ = 0.1 & 67.71$\pm$1.31  & 66.79$\pm$0.77  & 68.96$\pm$0.66  & 49.54$\pm$1.18  & 67.01$\pm$0.38 & \textbf{71.55$\pm$0.43} & {70.90 $\pm$0.49}\\  
                 &$\alpha$ = 1 & 77.02$\pm$1.09 & 75.93 $\pm$0.95 & 77.68$\pm$0.98  & 46.72$\pm$0.73 & 78.12$\pm$0.90 & 77.58$\pm$0.37  & \textbf{78.92$\pm$ 0.73}\\  
                 &$\alpha$ = 10 & 77.52$\pm$0.66 & 76.54$\pm$0.71  & 77.92$\pm$0.62 & 54.85$\pm$0.44 & 78.37$\pm$0.76  & 77.31$\pm$0.45   & \textbf{79.29$\pm$0.53}  \\  
                 \midrule
     \bottomrule 
     \end{tabular}} 
     \end{center}
     \caption{Performance overview under different data heterogeneity settings. For \Mdata~ and \Edata, user data follows the Dirichlet distribution with hyperparameter $\alpha$, with a \textbf{smaller} $\alpha$ indicating higher heterogeneity. For \Cdata, $r$ denotes the ratio between active users and total users. $T$ denotes the local training steps (communication delay). All above experiments use batch size $B$=32.
     \label{appendix:performance-overview}}
     \end{table*}
     \end{small}

%
\comment{ 
\begin{figure*}[hbt!]  
    \begin{center}
        \begin{subfigure}[b]{0.28\textwidth}
            \centerline{\includegraphics[width=\columnwidth]{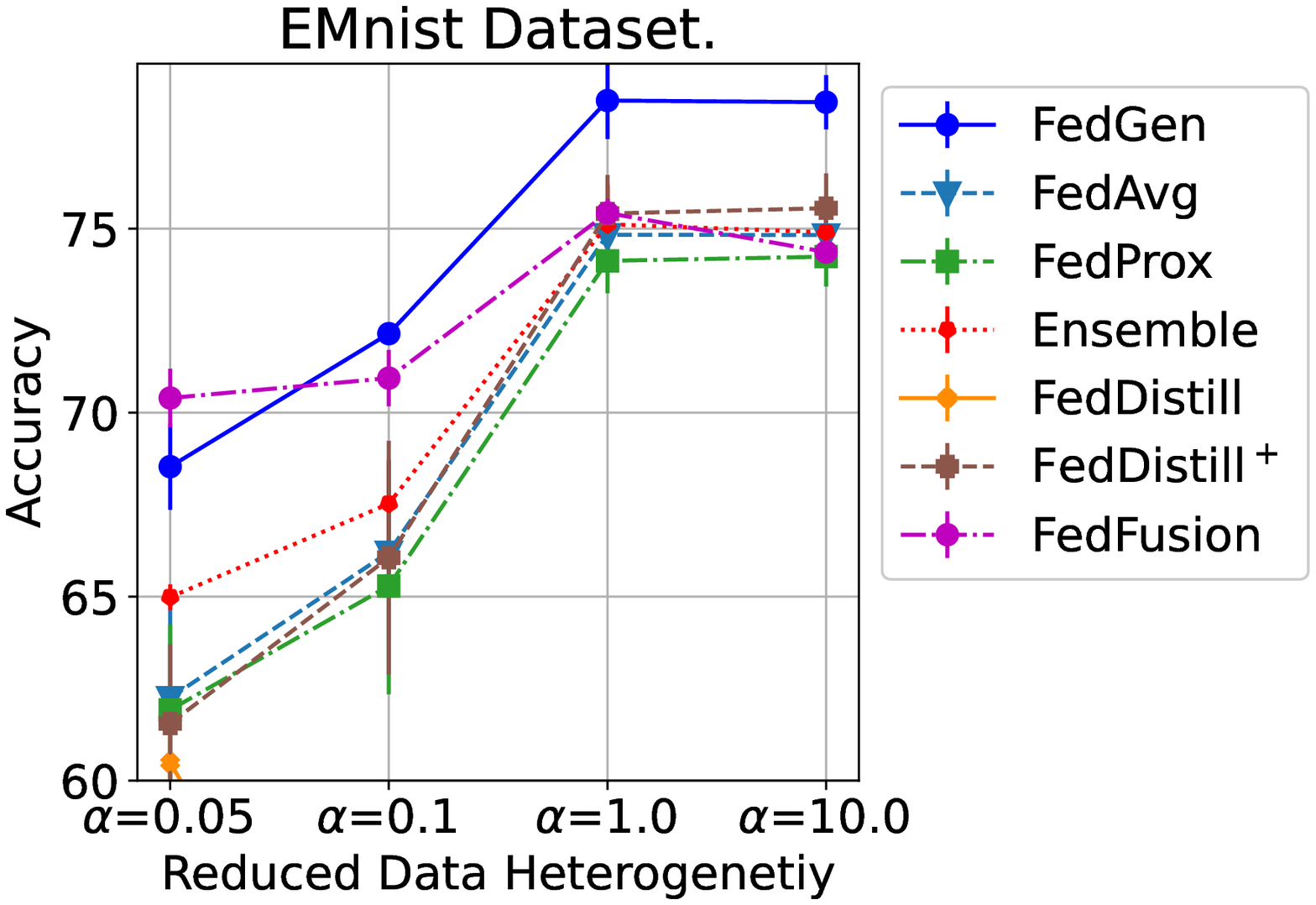}}
        \subcaption{local steps $T=20$.}
        \end{subfigure} 
        \hspace{-0.1in}
        \begin{subfigure}[b]{0.2\textwidth}
            \centerline{\includegraphics[width=\columnwidth]{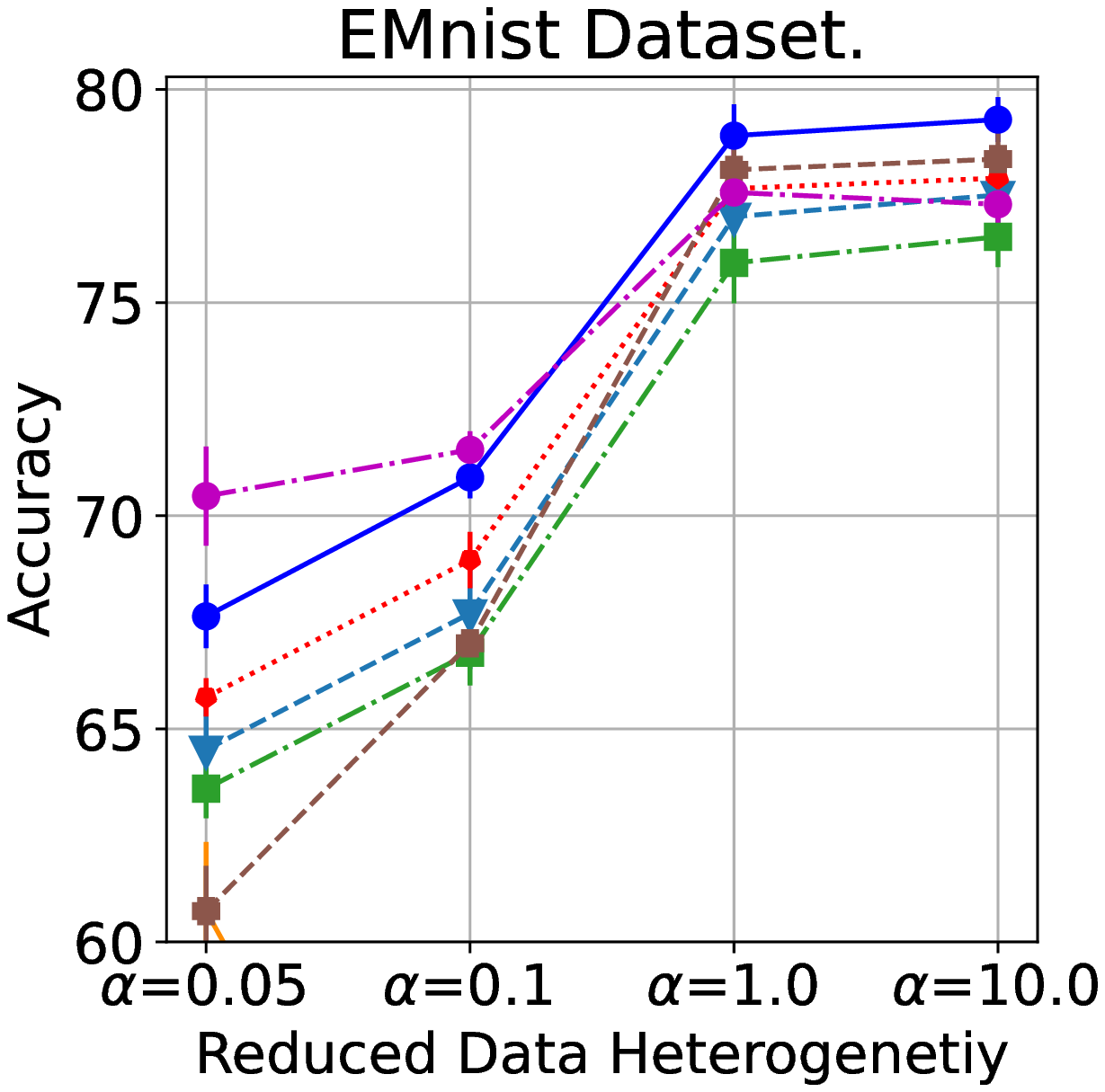}}
        \subcaption{local steps $T=40$.}
        \end{subfigure}  
          %
    \end{center}
    \vspace{-0.1in}
    \caption{Performance Overview on \Edata~ dataset, given different communication frequencies. 
    }\label{appendix-fig:emnist-overview}
\end{figure*} 
}

\comment{
\textbf{Impacts of Optimizer Selections:}
A result comparison using two different optimizers are given below:
We found that some baselines are sensitive to the choice of different optimizers.

\judycom{[RESULTS HERE.]} 

\begin{figure*}[hbt!]  
    \begin{center}
        \begin{subfigure}[b]{0.2\textwidth}
            \centerline{\includegraphics[width=\columnwidth]{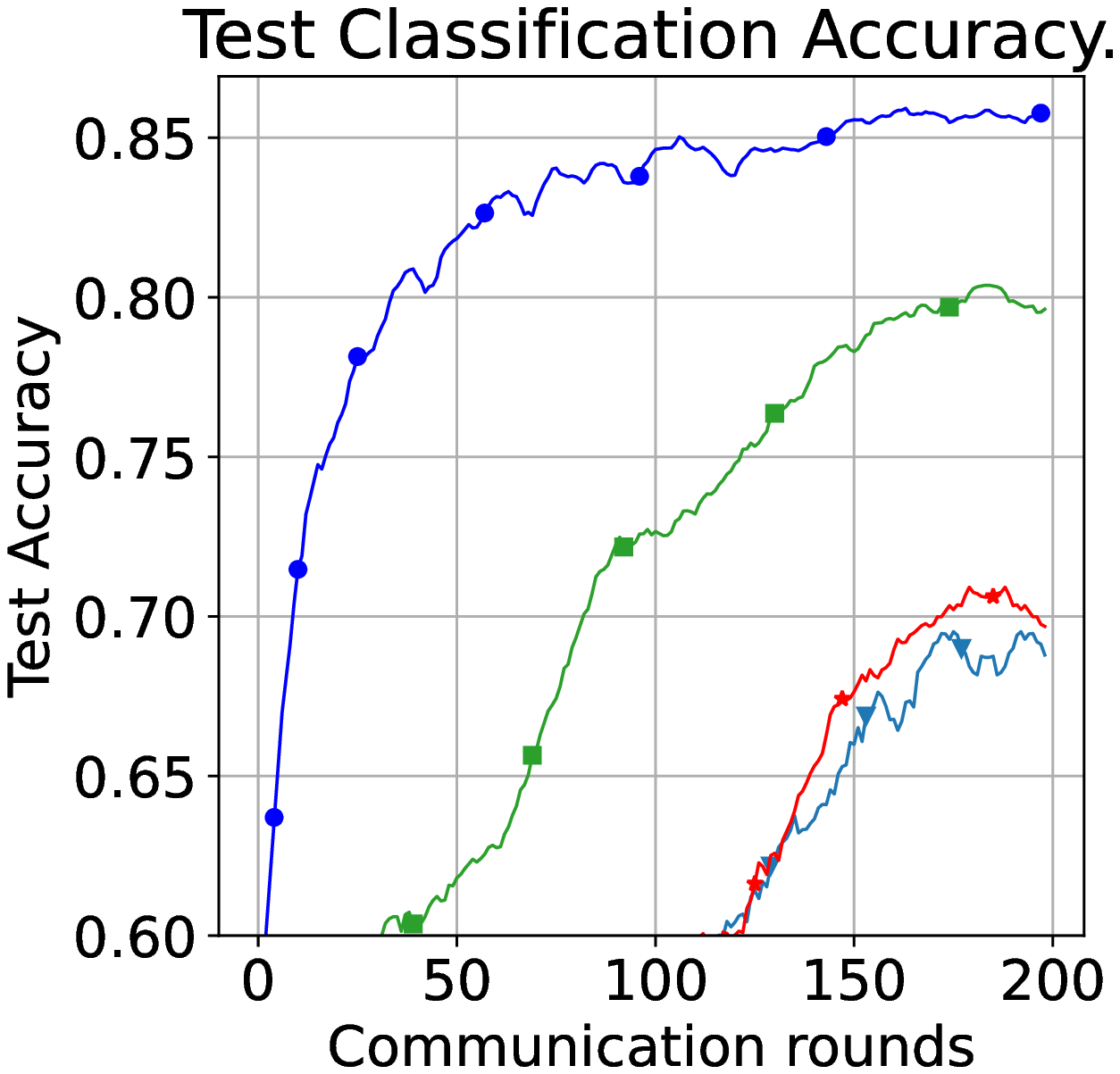}}
        \subcaption{$\alpha=0.05$, \textit{Adam} optimizer. }
        \end{subfigure} 
        \begin{subfigure}[b]{0.29\textwidth}
            \centerline{\includegraphics[width=\columnwidth]{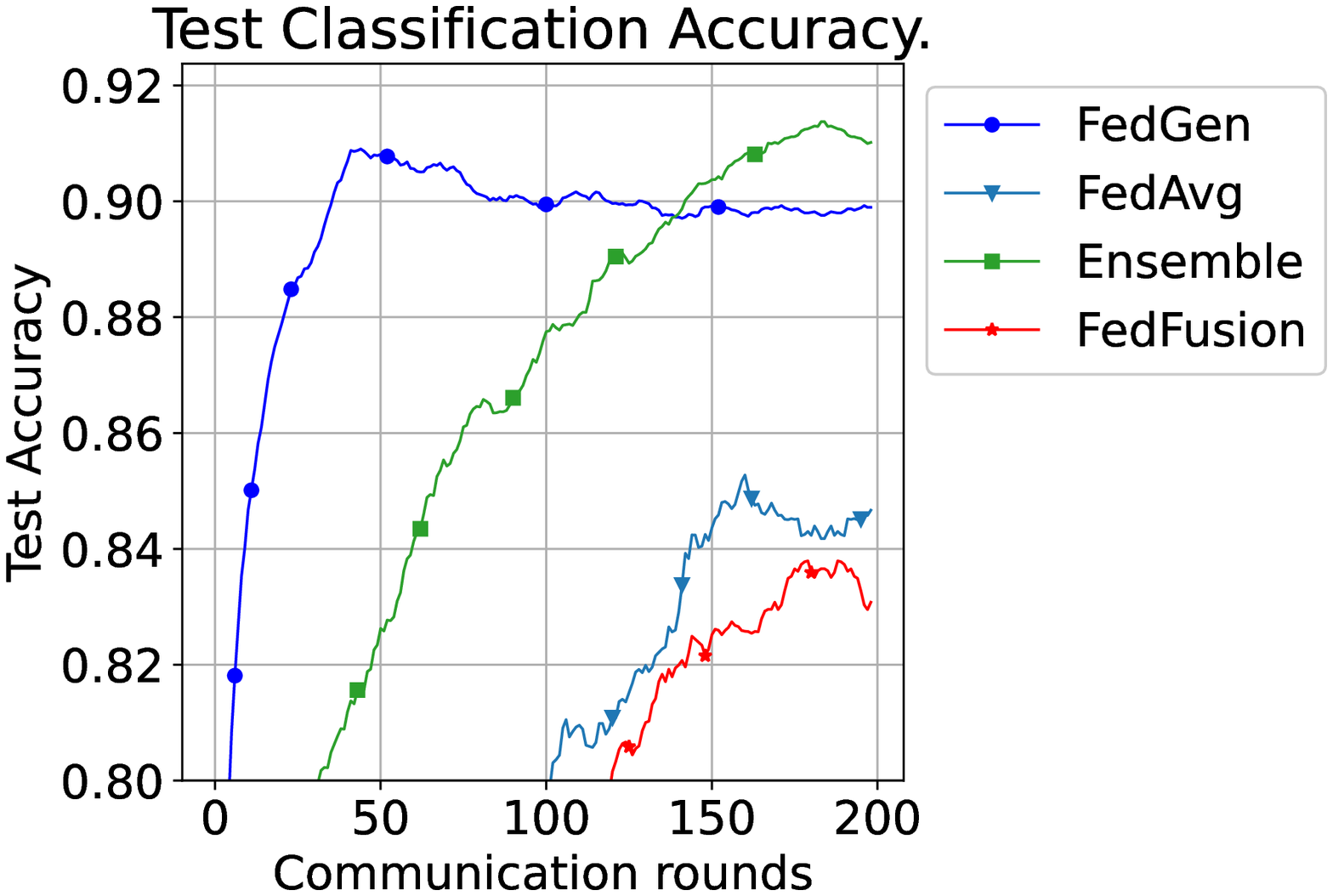}}
        \subcaption{$\alpha=0.1$, \\\textit{Adam} optimizer. }
        \end{subfigure} 
        \begin{subfigure}[b]{0.2\textwidth}
            \centerline{\includegraphics[width=\columnwidth]{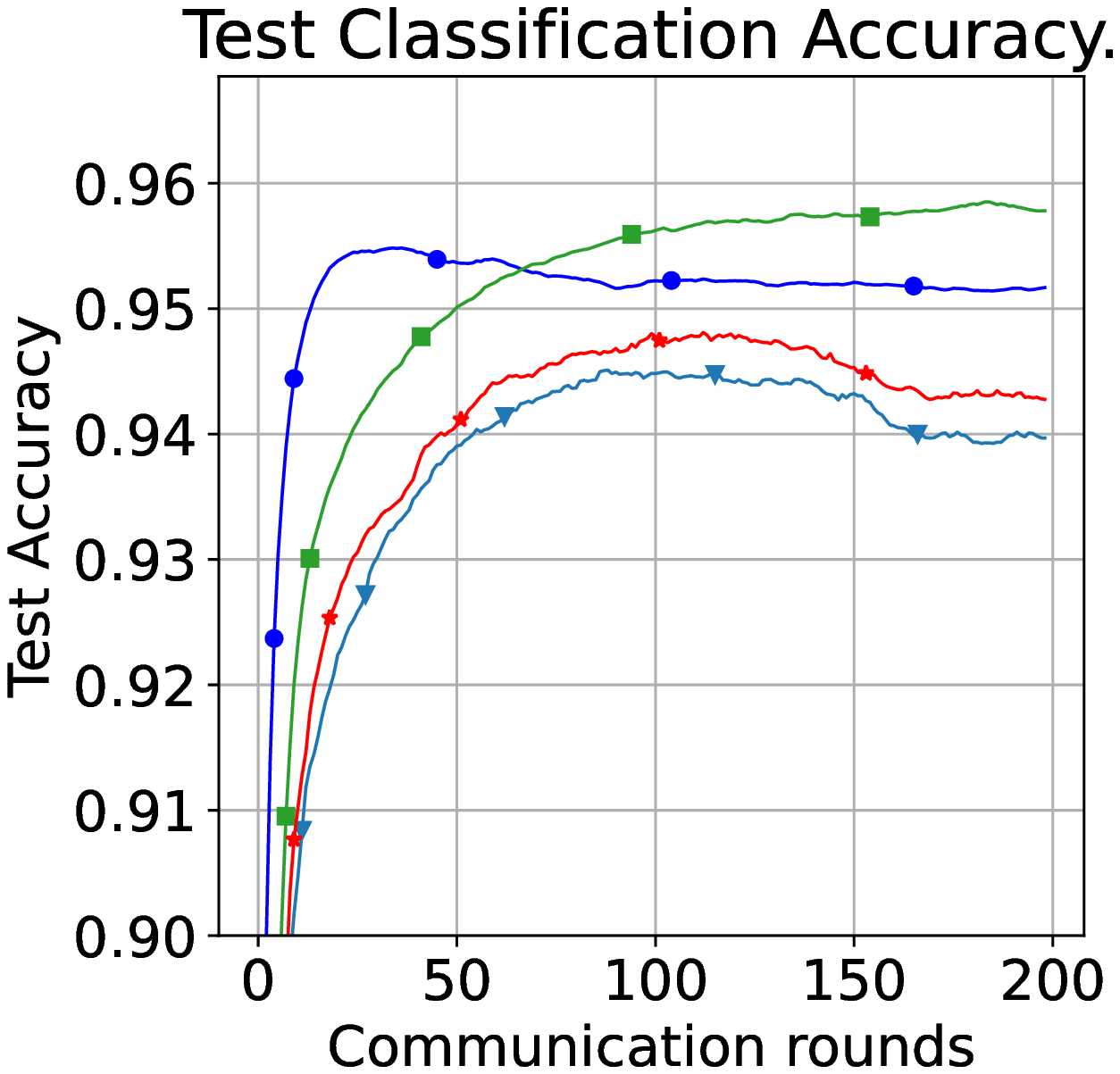}}
        \subcaption{$\alpha=1$, \\\textit{Adam} optimizer. }
        \end{subfigure} 
        \begin{subfigure}[b]{0.2\textwidth}
            \centerline{\includegraphics[width=\columnwidth]{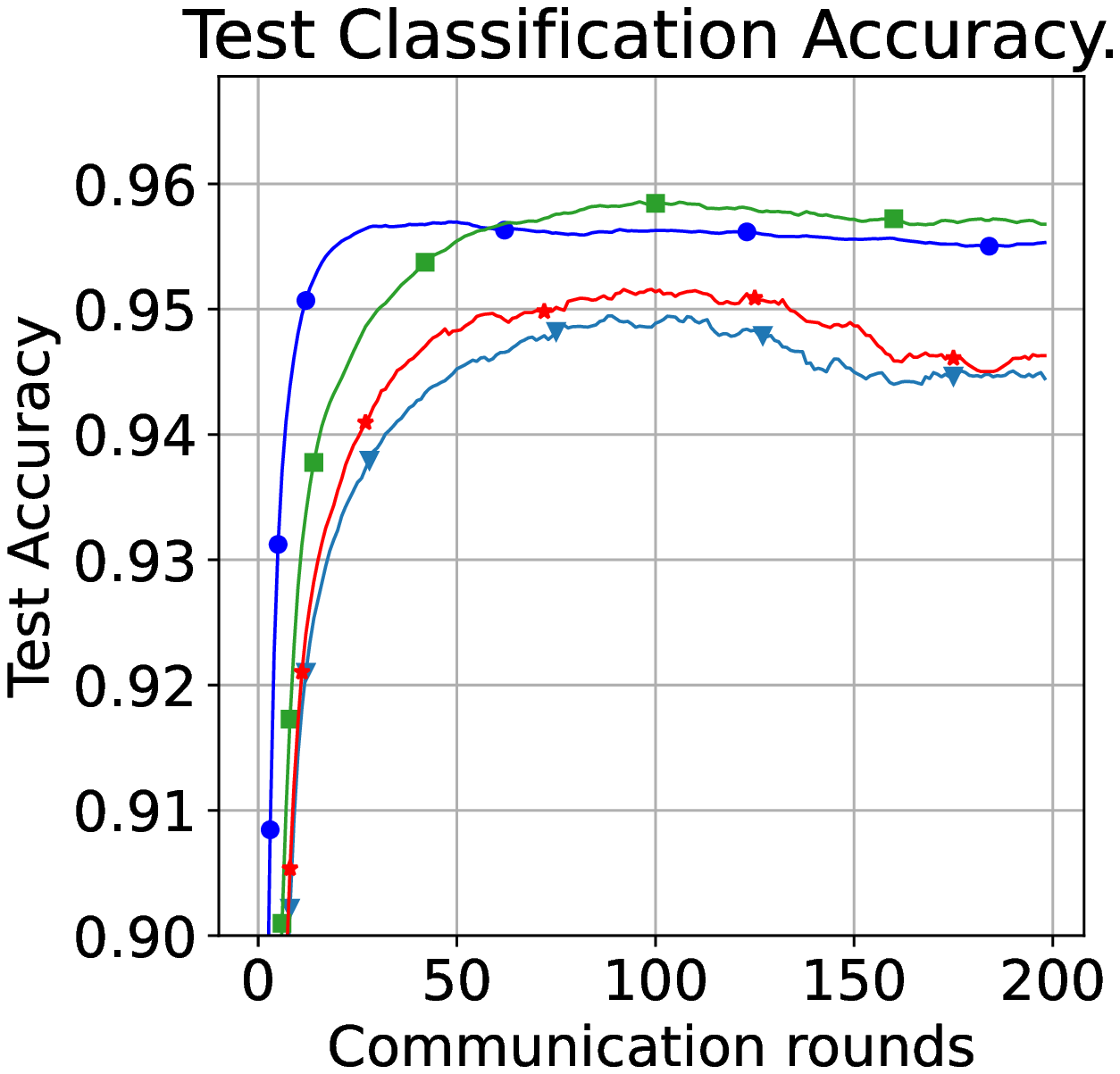}}
        \subcaption{$\alpha=10$, \\\textit{Adam} optimizer.}
        \end{subfigure} 
    \end{center}
    \vspace{-0.1in}
    \caption{Performance on \Mdata, using \textsc{Adam} optimizer for local user.}\label{appendix-fig:mnist-adam-optimizer}
\end{figure*} 
}

\comment{
\begin{figure*}[hbt!]  
    \begin{center}
        \hspace{-0.1in} 
        \begin{subfigure}[b]{0.2\textwidth}
            \centerline{\includegraphics[width=\columnwidth]{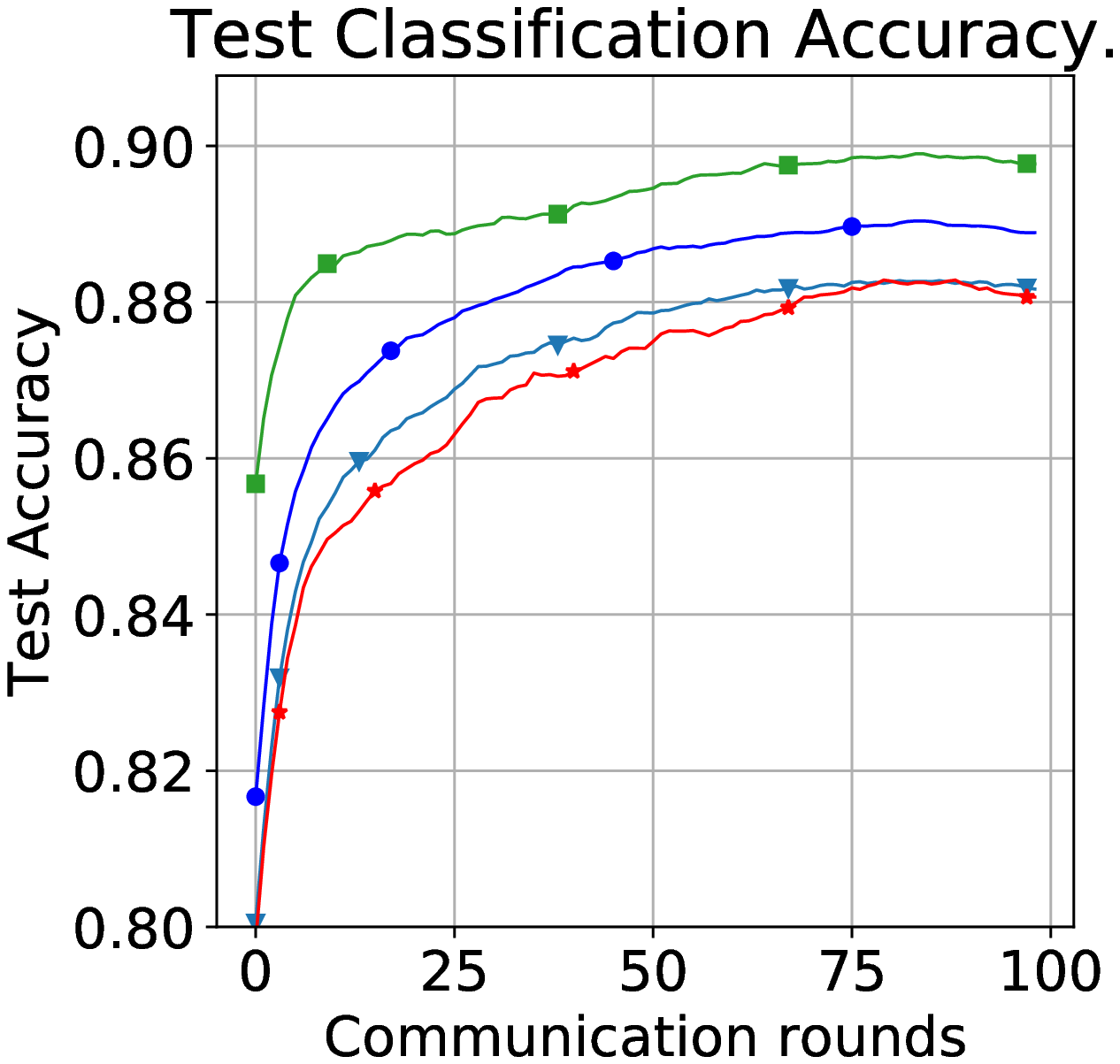}}
        \subcaption{\Cdata, $r=9/10$.}
        \end{subfigure} 
        \hspace{-0.1in}
        \begin{subfigure}[b]{0.29\textwidth}
            \centerline{\includegraphics[width=\columnwidth]{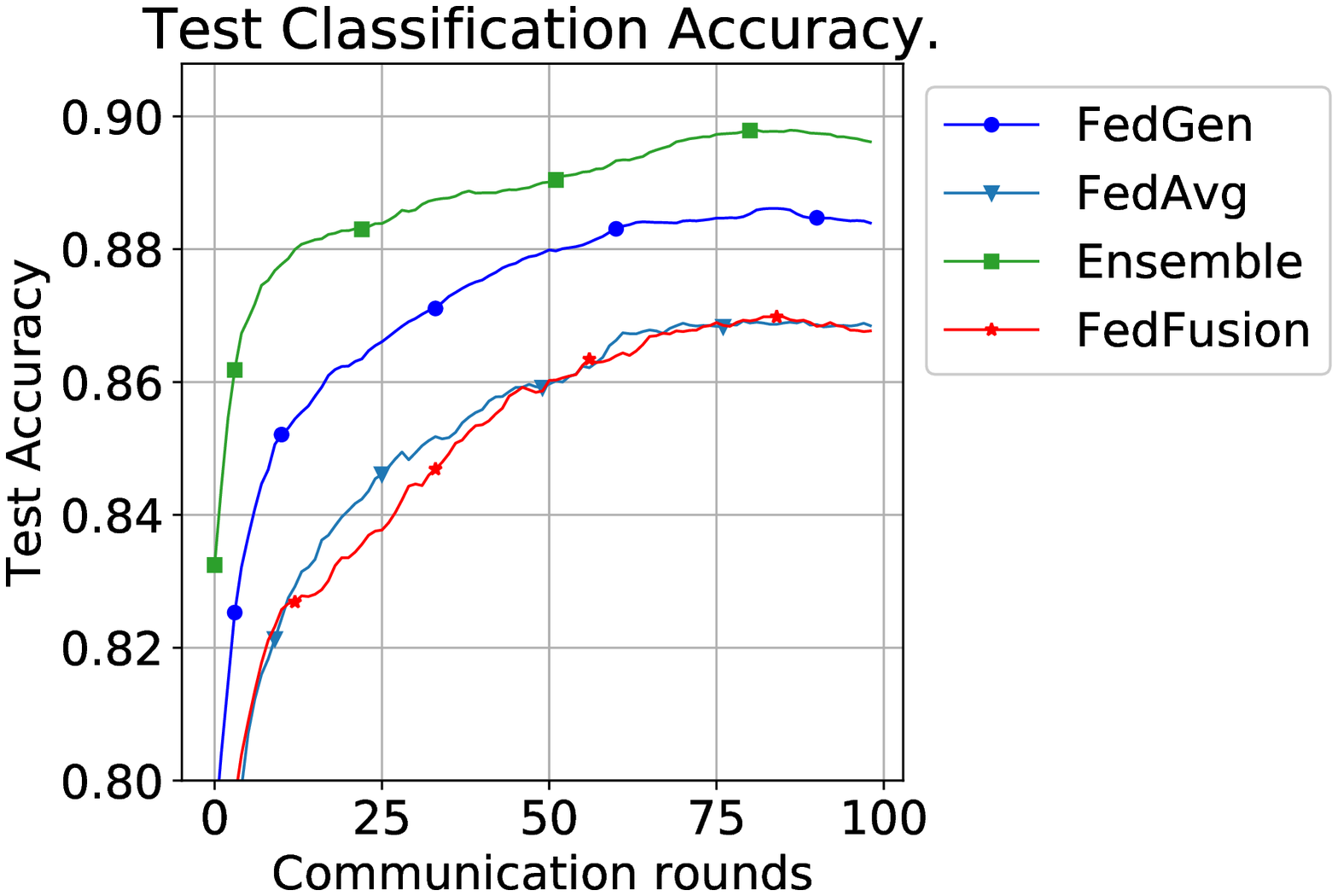}}
        \subcaption{\Cdata, $r=5/10$.}
        \end{subfigure} 
        \begin{subfigure}[b]{0.2\textwidth}
            \centerline{\includegraphics[width=\columnwidth]{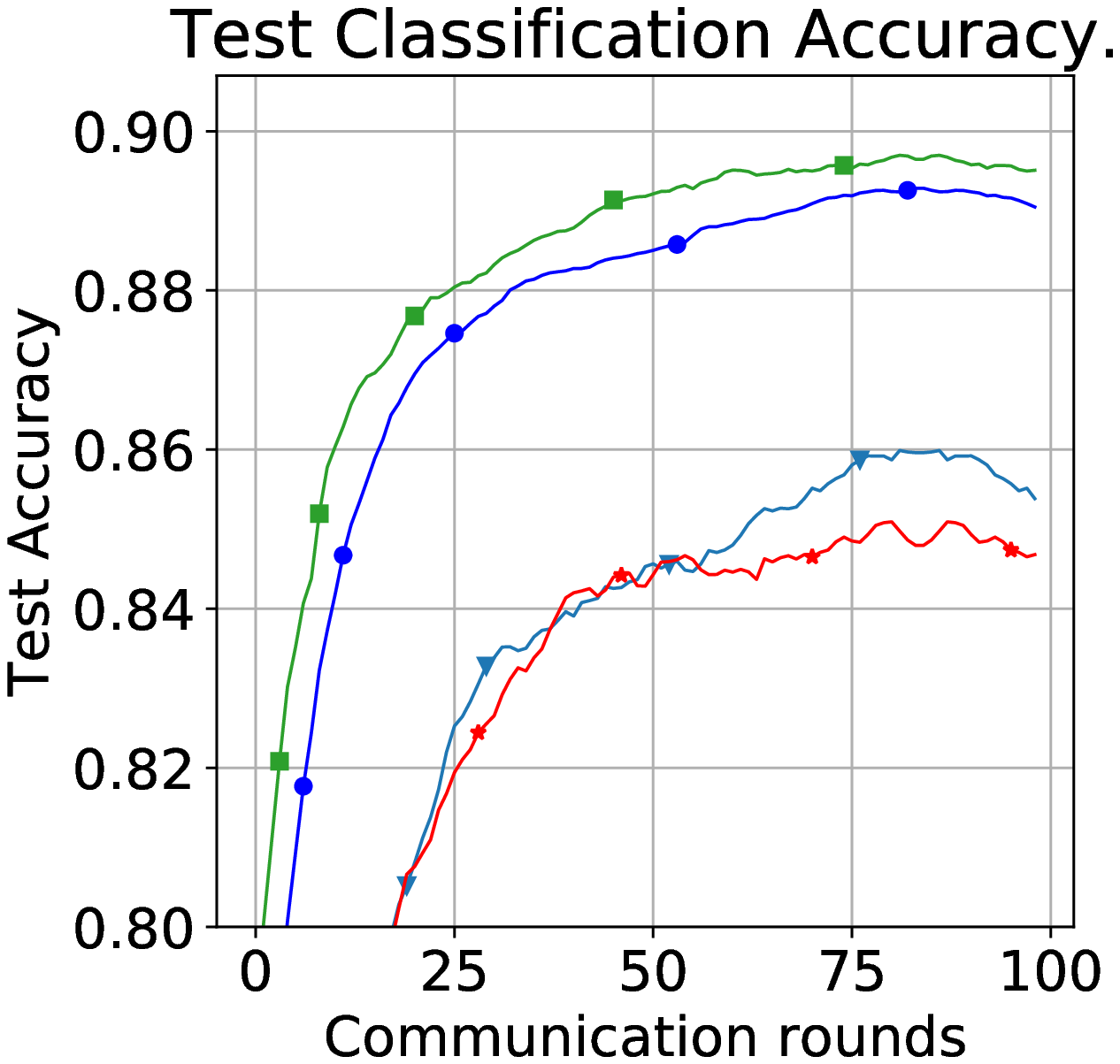}}
        \subcaption{\Cdata, $r=5/25$.}
        \end{subfigure} 
        \begin{subfigure}[b]{0.2\textwidth}
            \centerline{\includegraphics[width=\columnwidth]{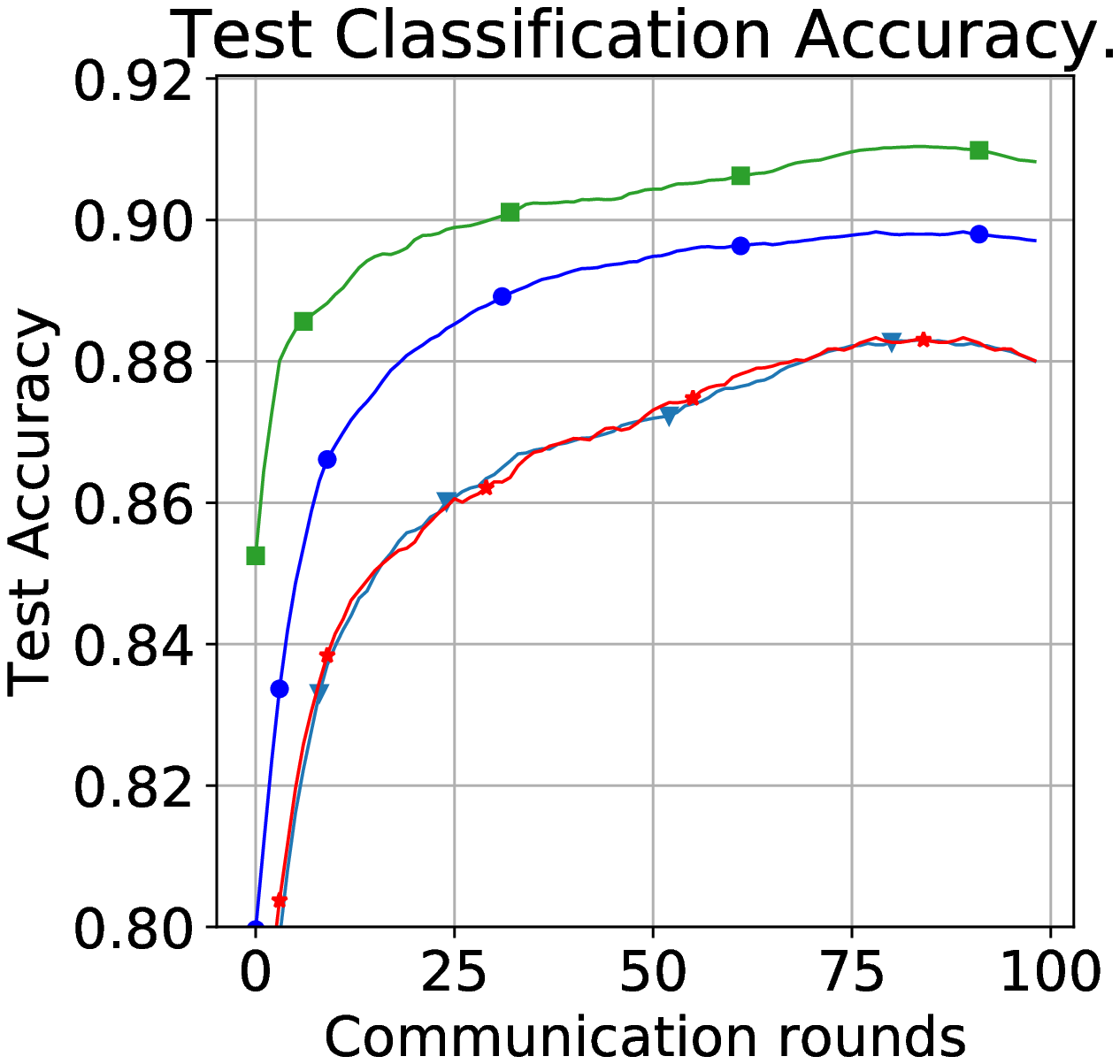}}
        \subcaption{\Cdata, $r=10/25$.}
        \end{subfigure} 
    \end{center}
    \vspace{-0.1in}
    \caption{Performance on \Cdata~ dataset using \textit{Adam} optimizer for local user.}\label{appendix-fig:celebe-adam}
\end{figure*} 
}


\comment{
\begin{table*}[htb!]
    \begin{center}
        \scalebox{0.9}{
            \begin{tabular}{llccccccc}
                \toprule 
\multicolumn{1}{l}{}  &  & \multicolumn{5}{c}{\textbf{Number of Communication Rounds to Reach Target Performance}}  \\ \hline 
Dataset 
&  
Setting 
& \multicolumn{1}{c}{\Avg} 
& \multicolumn{1}{c}{\Prox} 
%
& \multicolumn{1}{c}{\Ensemble}
%
& \multicolumn{1}{c}{\textsc{\FDFL}} 
%
& \multicolumn{1}{c}{\Fusion} 
& \multicolumn{1}{c}{\approach} \\ \hline
\multirow{4}{*}{\Mdata } 
                   & $\alpha=0.05$, T=$80\%$ & 38$\pm$6 & 37$\pm$10 & 35$\pm$8  & 64$\pm$25 & $-$ & \textbf{32$\pm$5} \\ 
                   & $\alpha=0.1$, T=$85\%$ & 42$\pm$3 & 49$\pm$7 & 39$\pm$4  & 50$\pm$10 & 78$\pm$11 & \textbf{27$\pm$2} \\ 
                   \midrule
\multirow{4}{*}{\Edata } 
                   & $\alpha=1$, T=$75\%$ & 68$\pm$11 & 92$\pm$5 & 72$\pm$6 &   63$\pm$9 & 48$\pm$12 & \textbf{41$\pm$6} \\ 
                  & $\alpha=10$, T=$75\%$ & 63$\pm$6 & 76$\pm$7 & 61$\pm$9 &  59$\pm$8 & \textbf{20$\pm$3} & 36$\pm$3 \\ 
                    \midrule
\multirow{4}{*}{\Cdata} 
                   & $r=5/10$, T=$85\%$ & 42$\pm$7 & 38$\pm$5 & 30$\pm$6 &    56$\pm$7 & 53$\pm$19 & \textbf{17$\pm$2} \\ 
                   & $r=9/10$, T=$85\%$ & 42$\pm$6 & 47$\pm$11 & 28$\pm$3  & 34$\pm$25 & 44$\pm$6 & \textbf{15$\pm$3} \\ \midrule
\bottomrule 
\end{tabular}} 
\caption{Training efficiency comparison, in which ``$-$'' indicates that the approach did not reach predefined accuracy before training stopped. $T$ denotes the target accuracy to be reached, $\alpha$ is for data heterogeneity. $r$ in \Cdata~is the ratio of active users, and $U$ denotes the total number of users.\label{appendix-table:communication-efficiency}}
\end{center}
\end{table*}
}

\comment{
\begin{figure}[hbt!]  
    \begin{center}
        \hspace{-0.1in}
        \begin{subfigure}[b]{0.2\textwidth}
            \centerline{\includegraphics[width=\columnwidth]{figs/toys/decision_boundary_global.png}}
        \subcaption{global model, with accuracy = $92.11\%$, obtained by averaging user parameters.} \label{appendix-fig:global}
        \end{subfigure} 
        \hspace{0.1in}
        \begin{subfigure}[b]{0.2\textwidth}
            \centerline{\includegraphics[width=\columnwidth]{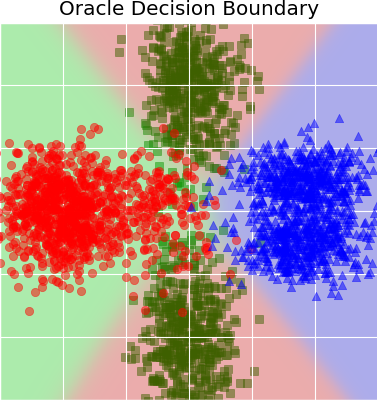}}
        \subcaption{Oracle model, with accuracy = $97.44\%$, obtained by training on all user data.} \label{appendix-fig:oracle}
        \end{subfigure} 
    \end{center}
    \vspace{-0.1in}
    \caption{Biased local models can negatively impact the global model obtained by \Avg~\Figref{appendix-fig:global}, compared with an ideal model trained using all global dataset~\Figref{appendix-fig:oracle}.}\label{appendix-fig:global-model}
\end{figure}
}

\bibliographystyle{icml2021}
\end{document}